\documentclass{dukecei}

\usepackage{times}  % DO NOT CHANGE THIS
\usepackage{helvet}  % DO NOT CHANGE THIS
\usepackage{courier}  % DO NOT CHANGE THIS
\usepackage{graphicx} % DO NOT CHANGE THIS
\usepackage{natbib}  % DO NOT CHANGE THIS AND DO NOT ADD ANY OPTIONS TO IT
\usepackage{caption} % DO NOT CHANGE THIS AND DO NOT ADD ANY OPTIONS TO IT
\usepackage{amsmath,amsfonts,amssymb,amsthm}
\usepackage{amssymb}

\usepackage{bm}
\usepackage{graphicx}
\usepackage{multirow}
\usepackage{mathtools}
\usepackage{subcaption}
\usepackage{textcomp}
\usepackage[inline]{enumitem}

\usepackage{cleveref} 
\usepackage{booktabs}
\usepackage{colortbl}
\usepackage{caption}
\usepackage[utf8]{inputenc}
\usepackage{booktabs} % For better table formatting
\usepackage{longtable} % For long tables
\usepackage{tocloft} % Allows customization of the Table of Contents
\usepackage{makecell, multirow} % Add this preamble
\usepackage{algorithm}
\usepackage{algorithmic}

\usepackage{float}
\usepackage{graphicx}
\usepackage[utf8]{inputenc}
\usepackage{placeins}

\usepackage{amssymb}   % already allowed by AAAI
\usepackage{soul}

\title{FlashSVD: Memory-Efficient Inference with Streaming for Low-Rank Models}

\author{
    Zishan Shao$^{1,2}$\footnote{Correspondence E-mail: zishan.shao@duke.edu},
    Yixiao Wang$^{2}$,
    Qinsi Wang$^{2}$,
    Ting Jiang$^{2}$,    
    Zhixu Du$^{2}$,
    Hancheng Ye$^{2}$,
    Danyang Zhuo$^{3}$,    
    Yiran Chen$^{2}$,
    Hai Li$^{2}$\\[1em]
    \normalsize $^{1}$Department of Statistical Science \\
    \normalsize $^{2}$Department of Electrical \& Computer Engineering \\
    \normalsize $^{3}$Department of Computer Science
}

\begin{document}

\maketitle
\thispagestyle{firstpagestyle} % Draws the header on the first page

\begin{abstract}
Singular Value Decomposition (SVD) has recently seen a surge of interest as a simple yet powerful tool for large language models  (LLMs) compression, with a growing number of works demonstrating 20-80\% parameter reductions at minimal accuracy loss. Previous SVD-based approaches have focused primarily on reducing the memory footprint of model weights, largely overlooking the additional activation memory overhead incurred during inference when applying truncated factors via standard dense CUDA kernels. Our experiments demonstrate that this activation overhead, scaling with sequence length and hidden dimension, prevents current SVD compression techniques from achieving any reduction in peak inference memory, thereby limiting their viability for real-world, on-device deployments. 

We introduce \textbf{FlashSVD}, a novel, end-to-end rank-aware streaming inference framework specifically designed for SVD-compressed large language models. FlashSVD can be seamlessly integrated with any model that employs SVD-based methods for parameter reduction. By fusing low-rank projection kernels directly into both the self-attention and feed-forward network (FFN) pipelines, FlashSVD avoid materializing full‐size activation buffers. Instead, small tiles of the truncated factors are loaded into on-chip SRAM, multiplied and reduced on the fly, and immediately evicted, preserving high GPU occupancy and adding no extra latency. On standard encoder benchmarks (e.g., BERT-Base), FlashSVD cuts peak activation memory by up to 70.2\% and intermediate transient memory by 75\%, all while incur no accuracy loss with upstreaming compression methods, offering a practical path toward memory-constrained deployment of low-rank LLMs.  
\end{abstract}

\section{Introduction}

Recent advances in Transformer architectures have yielded state‐of‐the‐art performance across diverse tasks with the cost of ever‐increasing model size and computational complexity.
Rapid scaling of model size poses non-negligible challenges on deploying large pre‐trained models on edge devices due to stringent \emph{peak memory} constraints: activations and full‐precision weights together often exceed available high-bandwidth memory (HBM), leading to out‐of‐memory (OOM) failures even for moderately sized context. While parameter compression techniques such as truncated singular value decomposition (SVD) can reduce model size by over 50\% with minimal accuracy loss \cite{hsu2022fwsvd, yuan2023asvd}, they do not address the \emph{activation} memory spikes during inference. In practice, mainstream inference backends (e.g., FlashAttention \cite{dao2022flashattention,dao2023flashattention2}, xFormers \cite{zhangxformer}, Megatron \cite{shoeybi2019megatron}) offer no native support for low-rank layers, so peak memory remains dominated by large intermediate buffers, almost as high as in fully dense implementations when using these inference backends.

%Concurrent with model compression techniques, there exists an exhaustive research for activation awared approach in dense model optimizations. Key–value (KV) caching reuses past computations in autoregressive decoders to avoid recomputing and storing full‐length attention states \cite{pope2023kvcache}. The causal relationship enables decoder to compute token-by-token, ensure the FFN complexity never dominates. FlashAttention fuses softmax and matmul into a single GPU kernel, avoiding formulation of $\mathcal O(B M^2)$ score matrices and significantly enhance clock-time performance of transformer-based models in both training and inference. \cite{dao2022flashattention, dao2023flashattention2}. Mixed precison approaches and model pruning also provide a promising, model agnositc approach in enhancing the computational and storage efficiency in Quantization and structured pruning further reduce model footprints Quantization and structured pruning further reduce model footprints \cite{dettmers2023qlora,han2016deep,jacob2018quantization,hubara2016quantized,wu2020integer,he2017channel,li2017filter,luo2017entropy,luo2017thinet,lin2019ssr,gordon2017morphnet}. Despite these advances, all prior work focuses on \emph{dense} layer infrastructural efficiency or \emph{block‐sparse} layers efficiency from algorithmic side; the case of \emph{low‐rank} (SVD‐based) inference efficiency from infrastructural side remains underexplored. % , and multi‐head latent attention projects into low‐dimensional subspaces to shrink both parameters and activations \cite{liu2024deepseek}

Concurrent with advances in model compression, extensive research has targeted activation-aware optimizations for dense transformers. Key–value caching in autoregressive decoders reuses past attention states to eliminate redundant $\mathcal{O}(B M^2)$, where B is the batch size and M is the sequence length, computations and storage overhead \cite{pope2023kvcache}. %\st{, while causal decoding guarantees that the feed-forward network never becomes a runtime bottleneck}. 
FlashAttention further accelerates both training and inference by fusing the softmax and matrix‐multiply operations into a single GPU kernel, thereby obviating the need to materialize full attention score matrices \cite{dao2022flashattention,dao2023flashattention2}. Complementary methods such as mixed‐precision arithmetic, quantization, and structured pruning have also proven effective at reducing compute and memory demands without modifying the core transformer architecture \cite{dettmers2023qlora,han2016deep,jacob2018quantization,wu2020integer,he2017channel,li2017filter,luo2017entropy,luo2017thinet,lin2019ssr,gordon2017morphnet}. %\st{Despite these infrastructural and algorithmic innovations for dense and block-sparse layers, the specific challenge of efficient, low-rank (represented by SVD) inference remains largely unexplored.} 
Despite these infrastructural and algorithmic innovations for dense and block-sparse layers, prior SVD-based compression techniques focus on dense model weights factorization but still execute inference with standard dense matrix kernels-thereby overlooking both activation memory overhead and the opportunity to tailor computations to low-rank structure.

In this work we introduce \textbf{FlashSVD}, an end‐to‐end inference framework that enables rank-awared activation for SVD-based low-rank models. Our key contributions are:

\begin{itemize}

%\item identifies the issue of peak activation domination, which the prior focus was primarily on weights compression and limited work in activation side. Potential direction: finetuning enables even lower-rank yet preserves the performance, turely enables the
%\item Highlights that activation is also an devil along with fixed parameter cost in inference, whereas prior work has largely focused on weight compression with limited attention to activation‐side optimizations. A promising direction is rank‐aware fine‐tuning, which can drive ranks even lower while preserving accuracy and further shrink activation costs.
\item  We identify that activation memory-alongside fixed parameter cost-dominates inference overhead, whereas prior work has largely targeted weight compression. We propose rank-aware fine-tuning as a promising direction to drive ranks lower without sacrificing accuracy, yielding further activation-cost reductions.

% peak activation memory-rather than model parameters-now dominates inference overhead, 

\item We proposed two series of rank-awared streaming low-rank kernels that consumes only low-rank activations in a single pass with no memory-expensive dense intermediates while ensure computational efficiency in moderate-large context window (Theorem~\ref{thm:flashsvd-mem}, \ref{thm:multihead-bound}, \ref{thm:flashsvdffn-v1-memory}, \ref{thm:flashsvdffn-v2-memory}). %, achieving \(\mathcal O\bigl(r_{\Theta^*}(B\,M+\mathcal D_{\mathcal{F}}^\Theta)\bigr)\) off‐chip cost (Theorem~\ref{thm:flashsvd-mem}, \ref{thm:multihead-bound}, \ref{thm:flashsvdffn-v1-memory}, \ref{thm:flashsvdffn-v2-memory}).

   %FFN that provides promising HBM memory efficiency while ensure computational efficiency in moderate-large context window, achieving
  %\(\mathcal O\bigl(r_{\Theta^*}(B\,M+\mathcal D_{\mathrm{ff}})\bigr)\) off‐chip cost

% \item More efficiency in compression: Proposition~\ref{prop:grouped-tradeoff} shows that grouping \(H\) heads into \(G\) blocks of size \(\tfrac{H}{G}\) yields $ \mathcal O\!\bigl(G\,r_{\Theta^*}(B\,M + \tfrac{\mathcal D_\Theta}{G})\bigr),$ and Proposition~\ref{prop:multihead-sensitivity} establishes that each unit drop in rank saves exactly \(\Delta\mathcal M = G\,(B\,M + \tfrac{\mathcal D_\Theta}{G})\) off‐chip elements that are independent of \(r_{\Theta^*}\).  Setting \(G=H\) recovers the head-wise threshold \(r_{\Theta^*}<\tfrac{\mathcal D_{\Theta}}{H+1}\), while any \(G\) gives a predictable, tight trade-off backed by matching \(\Omega\) and \(\mathcal O\) bounds.

%\item When compressing the attention projection matrices by innate multi-head structure, the required rank‐loss ratio, given by $ 1 - \frac{r_{\text{retained}}}{r_{\max}}$, is significantly reduced compared to prior single head based compression, enabling less aggressive rank reduction to achieve similar compression performance.

\item By exploiting the multi‑head structure when compressing the attention projection matrices, the necessary rank‑loss ratio $1 - \frac{r_{\text{retained}}}{r_{\max}}$ is substantially lower than that required by single‑head compression methods in prior works, thereby permitting more moderate rank reductions while still achieving comparable overall compression performance.

  %\item On a standard Transformer (e.g.\ BERT‐Base) with rank \(r_{\Theta^*}=128\), FlashSVD reduces peak activation memory by 74\% versus naive SVD inference and by 87\% versus the dense baseline, with no loss in accuracy.
\end{itemize}

\begin{figure*}[t!]
  \centering
  \includegraphics[width=\textwidth]{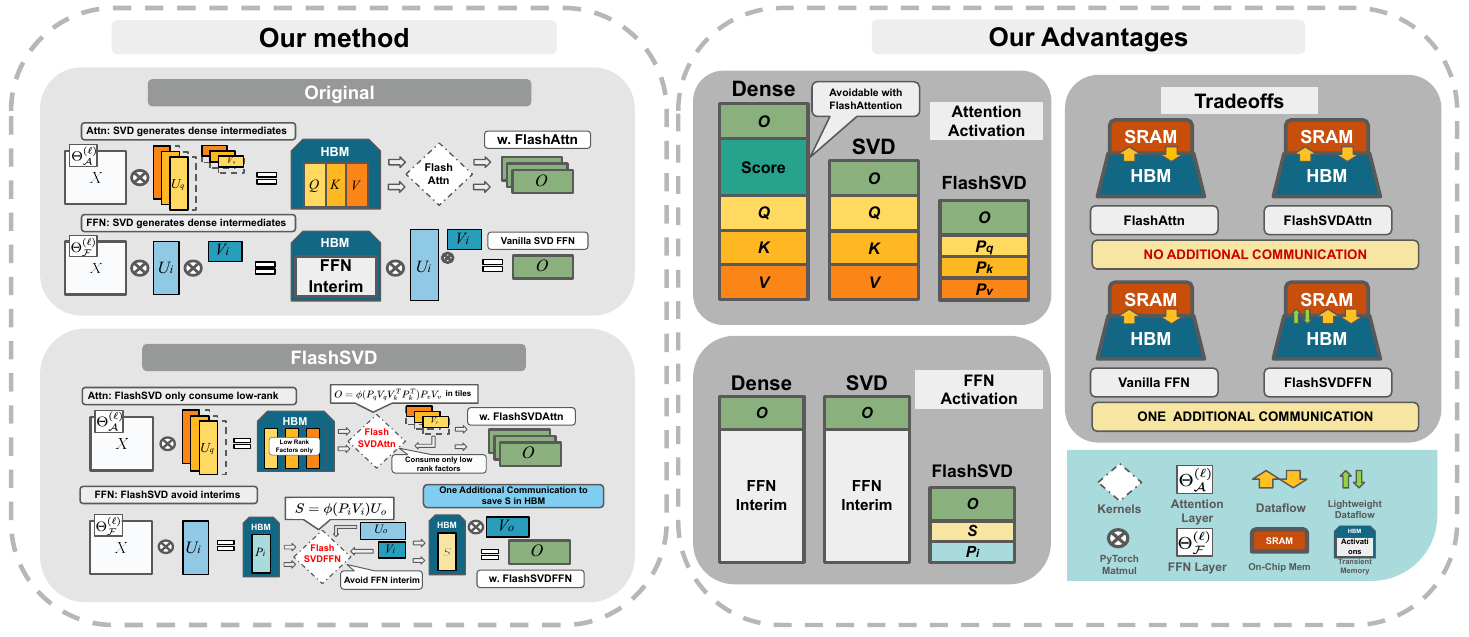}
  \caption{FlashSVD fuses low‑rank projections into both the self‑attention and FFN stages by streaming small tiles of the truncated U and V factors through on‑chip SRAM instead of materializing dense activation matrices in HBM. \textbf{FlashSVDAttn} only consumes low-rank factors, reducing complexity from $O(BM\mathcal{D}_\mathcal{A}^\Theta)$ to $O(BMr_{\Theta^*})$. In the FFN, inputs are first projected into the low‑rank factor space via a cuBLAS GEMM (for standard matrix multiplication) and then passed through \textbf{FlashSVDFFN} in one pass, avoiding any $O(BM\!\mathcal{D}_\mathcal{F}^\Theta)$ intermediate activations and ensuring the rank-awared HBM performance.}
  \label{fig:pipeline}
\end{figure*}

\section{Related Work}
% \paragraph{Low-Rank Weight Compression}
% 1. svd tehcnique + SVD related model compression (fwsvd, asvd, svd-llm, dobi-svd etc)
% 2. other compression technique, nystrom compression, CR decomposition, tensor-train, HOSVD.

% \paragraph{streaming technique}
% 1. block LLM that execute the entire model with streaming
% 2. flashattention, flash ffn, xFormer (FlashAttention should work with the lora, layernormalization etc)

% \paragraph{Finetuning \& Incorporation of other compression techniques}
% 1. quantization: quantization has been used along with teh 
% 2. LoRA, used in svd-llm, that is incorporable in normal application, asvd, that is memory away 

% \paragraph{Low‐Rank Weight Compression}  
% SVD‐based factorization has long been used for transformer compression.  Early work on Activation‐aware SVD (ASVD) targets full $H\times H$ projections, but only yields savings when $r<H/2$, often degrading accuracy \cite{yuan2023asvd}.  MoDeGPT (ModeGPT) further structures this via modular decomposition across layers \cite{lin2024modegpt}.  Dobi‐SVD makes truncation differentiable, learning optimal cuts via backpropagation \cite{wang2025dobi}.  Beyond SVD, Nystrom‐style methods approximate attention kernels in $O(n)$ time \cite{xiong2021nystromformer}, and tensor‐train (TT) decompositions compress both embedding and feed‐forward blocks \cite{novikov2015tensorizing}.  Higher‐order SVD (HOSVD) and hybrid tensor schemes have also been explored for balanced compression–expressivity trade‐offs \cite{tucker1966some,oseledets2011tensor}. 
\paragraph{Low‐Rank Weight Compression}  
SVD‐based factorization is a classical approach for transformer compression, where small singular values are truncated under the Eckart–Young Mirsky theorem to reduce parameter counts at some cost of accuracy \cite{yuan2023asvd,wang2024svd}.  
% Activation‐aware SVD (ASVD) applies a rank-awared approach that, along with up to 30\% compression in model weights, compressed the KV-Cache memory by storing components in low-rank format, achieving 10–30\% weight and 50\% KV-cache reductions without retraining \cite{yuan2023asvd}.  SVD-LLM applies a similar approach in storing KV-Cache while applies a calbration strategy that utilize a lightweighted calibration set to ensure smart truncation, ensure performance retaining. 
Activation‐aware SVD (ASVD) introduces rank awareness into both weight and KV‐cache compression for decoders, factorizing cache states into low‐rank components to achieve up to 30 \% weight reduction and 50 \% KV‐cache memory savings with no retraining \cite{yuan2023asvd}. Building on this idea, SVD-LLM (v1-v2) \cite{wang2024svd, wang2025svd} also stores KV‐cache in low‐rank form but further employs a lightweight calibration set to determine optimal truncation thresholds, thereby preserving downstream performance despite more aggressive compression.
MoDeGPT partitions each Transformer block into modular subcomponents and applies Nystrom, CR decomposition, and SVD at the module level - enabling structured, fine‐tuning-free compression that saves up to 30\% compute and increases inference throughput by 46\% on 13B models \cite{lin2024modegpt}.  
Dobi‐SVD makes the truncation process differentiable by learning per‐layer cut positions via gradient‐based optimization and deriving optimal low‐rank updates, thus maintaining performance under 40\% parameter compression \cite{wang2025dobi}. Palu \cite{chang2025palu} introduced a rank‐aware KV-Cache compression scheme that dynamically adapts the low‐rank approximation of cached keys and values per attention head that significantly reducing memory footprint without degrading retrieval accuracy for the SVD decomposed models. 

% Nystromformer adapts the classical Nystrom method to approximate the self-attention kernel with $O(n)$ complexity via landmark points, matching standard attention quality on GLUE and the LRA benchmark while scaling to thousands of tokens \cite{xiong2021nystromformer}.  
% Tensor‐Train (TT) decompositions convert dense fully‐connected weight matrices into a chain of low-rank tensor cores, reducing parameters by up to 200,000$\times$ with minimal expressivity loss \cite{novikov2015tensorizing}.  
% Higher-order SVD (Tucker/HOSVD) generalizes matrix SVD to multiway tensors by decomposing into a small core tensor and factor matrices, enabling multilinear compression across embeddings, attention, and FFN tensors for balanced compression–accuracy trade-offs \cite{tucker1966some,oseledets2011tensor}.

% \paragraph{Streaming \& Block‐Parallel Inference}  
% FlashAttention tiles the attention matmul to minimize the communication between HBM-SRAM IO, yielding up to $3$\times$$ speedups on GPT-2 and exact correctness \cite{dao2022flashattention}.  FlashAttention-2 refines work‐partitioning to reach 50–73\% of peak FLOPs on A100 GPUs \cite{dao2023flashattention2}.  BlockLLM shares KV‐cache blocks across tenants to reduce memory and latency in multi‐model serving \cite{shi2024blockllm}, while Star Attention shards long‐sequence tokens across hosts in a two‐phase local or global protocol \cite{liu2024starattention}.  

\paragraph{Streaming Inference}  
FlashAttention reorders the standard attention %$QK^\top V$  
computation into tiled matrix‐multiply and softmax steps that fit in on‐chip SRAM, thereby minimizing HBM-SRAM transfers and achieving up to $3\times$ end‐to‐end speedups on GPT-2 with exact numerical equivalence to dense attention \cite{dao2022flashattention}.  
% FlashAttention v2 and v3 extends this approach by optimizing work partitioning and kernel fusion to utilize 50–73\% of peak FLOPs on NVIDIA A100 GPUs and n. We demonstrate that our method, FlashAttention-3, achieves speedup on
% H100 GPUs by 1.5-2.0$\times$ with FP16 reaching up to 740 TFLOPs/s (75% utilization), and with FP8 reaching
% close to 1.2 PFLOPs/s. We validate that FP8 FlashAttention-3 achieves 2.6$\times$ lower numerical error than a
% baseline FP8 attention., further boosting throughput for both single‐ and multi‐head attention patterns \cite{dao2023flashattention2, shah2024flashattention}.  
FlashAttention-2 and -3 optimize work partitioning and kernel fusion to sustain 50–73 \% of peak FLOPs on NVIDIA A100s, and FlashAttention-3 on H100s delivers 1.5–2.0$\times$ speedups in FP16 (up to 740 TFLOPs/s, 75 \% utilization) and near 1.2 PFLOPs/s in FP8—with 2.6$\times$ lower numerical error than baseline FP8 attention \cite{dao2023flashattention2,shah2024flashattention}.
BlockLLM introduces a block‐parallel inference engine that shards and caches KV blocks across multiple concurrent workflows, reducing redundant memory loads and lowering per‐token latency by as much as 33\% in multi‐tenant serving scenarios \cite{shi2024blockllm}.  
Star Attention partitions long‐sequence inputs across hosts in two phases—a block‐local intra‐node attention followed by a global inter‐node aggregation—enabling exact, distributed attention on sequences exceeding 64K tokens with sublinear communication overhead \cite{liu2024starattention}.

% \paragraph{Fine‐Tuning \& Hybrid Compression}  
% LoRA injects low‐rank adapters into frozen weights, reducing trainable parameters by orders of magnitude with no extra inference cost \cite{hu2021lora}.  QLoRA combines 4-bit quantization with LoRA for memory‐efficient fine‐tuning of 65B models on a single GPU \cite{dettmers2023qlora}.  
\paragraph{Fine‐Tuning \& Hybrid Compression}  
LoRA inserts trainable, low-rank adaptation matrices into the frozen weights of each Transformer layer, reducing the number of fine-tuned parameters by over four orders of magnitude while matching full-fine-tuning performance on tasks such as GLUE and SQuAD \cite{hu2021lora}.  
QLoRA builds on this by first quantizing all model weights to 4-bit integers using NormalFloat and GPTQ techniques, then applying LoRA to the quantized weights; this combination enables fine-tuning of 65B-parameter LLMs on a single 48 GB GPU with negligible quality loss \cite{dettmers2023qlora}.

% Although low-rank SVD compressors and IO-aware streaming kernels have each demonstrated substantial benefits, they have evolved in isolation: SVD-based methods perform offline factorization without regard to GPU memory‐I/O or block-parallel execution, while streaming kernels assume dense weight matrices.  To the best of our knowledge, no prior work co-designs weight compression and inference scheduling.  In this paper, we bridge that gap and proposing a generalized observation for both encoder and decoder models that, by integrating head-wise SVD, we could reduce the per-head rank threshold from $\frac{\mathcal{D}_\Theta}{2}$ to $\frac{\mathcal{D}_\Theta}{H+1}$ along with directly into a streaming, block-parallel attention and FFN engine. This joint design delivers both immediate parameter savings and end-to-end throughput improvements without introducing extra parameters or requiring retraining of fusion layers.  

%  To our knowledge, no prior work jointly optimizes weight compression and inference scheduling. 

Although low‐rank SVD compressors and I/O‐aware streaming kernels have each shown dramatic gains, they have developed in isolation: SVD methods focus on offline factorization with no regard for GPU memory‐I/O or block‐parallel execution, while streaming kernels target dense weights. There is limited work in developing a rank-awared streaming kernels for SVD-based low-rank models inference. We bridge this gap with two families of rank‐aware GPU kernels that eliminate memory‐intensive dense activation materialization. Our method is agnostic to the upstream SVD compression -- applicable to any SVD‐compressed model without degrading accuracy.

\section{Methodology}
\paragraph{Optimal Low‐Rank Factorization via SVD}  
Let \(W\in\mathbb{R}^{m\times n}\) be a dense matrix (e.g.\ a full self‐attention projection matrix).  Its full singular value decomposition  
\begin{equation}
\begin{aligned}
W &= U\,\Sigma\,V^\top, %\\[4pt]
\Sigma = \mathrm{diag}\bigl(\sigma_1,\dots,\sigma_{\min\{m,n\}}\bigr), \\[4pt]
\sigma_i &\ge \sigma_j,\quad \forall\,i > j.
\end{aligned}
\end{equation}
yields the best rank‐\(r\) approximation in Frobenius norm by the Eckart–Young–Mirsky theorem \cite{eckart1936approximation}. Concretely, truncating to the top-\(r\) singular values gives  
\[
W^{[r]}
= U_{[:,1:r]}\,\Sigma_{[1:r,1:r]}\,V_{[:,1:r]}^\top,
\]
which provably minimizes \(\|W - W^{[r]}\|_F\) over all rank-\(r\) matrices. For simplicity, we evenly distribute the singular values, rewriting as 
\(W^{[r]} = \tilde{U}\,\tilde{V}^\top\) where \(\tilde{U} = U_{[:,1:r]}\Sigma^{1/2}\) and \(\tilde{V} = V_{[:,1:r]}\Sigma^{1/2}\).

\paragraph{Model Decomposition}  
For a general large‐language model $\Theta \;=\;\bigl\{\Theta_{\mathcal{A}}^{(\ell)}\cup\Theta_{\mathcal{F}}^{(\ell)}\bigr\}_{\ell=1}^L$, where $\Theta_{\mathcal{A}}^{(\ell)}$ and $\Theta_{\mathcal{F}}^{(\ell)}$ denote the attention and FFN parameters in block~$\ell\,$.  Equivalently, grouping across layers gives
$
\Theta = \{\Theta_{\mathcal{A}}^{(\ell)}\}_{\ell=1}^{L}\cup \{\Theta_{\mathcal{F}}^{(\ell)}\}_{\ell=1}^{L}
$, so truncated SVD can be independently applied layer-wise on each $\Theta_{\mathcal{A}}^{(\ell)}$ and $\Theta_{\mathcal{F}}^{(\ell)}$ subfactors, enabling fine-grained, per-layer rank selection and compression. We apply rank-$r_{\Theta^*}$ truncated SVD (Definition~\ref{def:attn-activation-approx}, \ref{def:ffn-svd}) to each dense projection matrix in a Transformer block, where the rank for the self-attention module follows Theorem~\ref{thm:singlehead-lower}, while the FFN rank is selected according to Theorem~\ref{thm:multihead-lower}. The resulting compressed model is denoted $\Theta^*$, so that $r_\Theta^* = r_{\Theta^*}$. These rank choices directly influence the memory–compute trade-offs described in Theorem~\ref{thm:flop-reduction}. The self-attention projections are square matrices of size $\mathcal{D}_\mathcal{A}^\Theta$, while the FFN projection matrices $W_i^\top$ and $W_o$ are rectangular of shape $\mathcal{D}_\mathcal{F}^\Theta \times \mathcal{D}_\mathcal{A}^\Theta$.

% \paragraph{Streaming Low-Rank Activations}  

% Rather than materializing the dense truncation
% $W^{[r]}\in\mathbb{R}^{m\times n}$ and touching all $mn$ entries (where $m,n$ might be $\mathcal D_{\mathcal F}$, $\mathcal D$, or $d_{\Theta_{\mathcal A},h}$, etc.), FlashSVDFFN (Lemma~\ref{lem:ffn-flashsvd-upper}) and FlashSVDAttn (Theorem~\ref{thm:flashsvd-mem}) enables autoloading with tiling the low-rank factors stored in HBM, such as \(U\in\mathbb{R}^{\mathcal{ D}^{\Theta}\times r_{\Theta^*}}\) and \(V\in\mathbb{R}^{r_{\Theta^*}\times \mathcal D^\Theta}\), then load only small tiles of \(U\) and \(V\) into on-chip SRAM, perform the local matmuls, and immediately evict them.  By never materializing a dense \((B\times M)\times \mathcal D^\Theta_{\mathcal{A}}\) or \((B\times M)\times\mathcal{D}^\Theta_{\mathcal{F}}\) buffer off-chip, the HBM footprint per layer is reduced from \(\mathcal O(BM\,\mathcal D^\Theta_{\mathcal{A}})\) or \(\mathcal O(BM\,\mathcal D_{\mathcal{F}}^\Theta)\) to  
% $
% \mathcal O\bigl(r_{\Theta^*}(\,B\,M + \mathcal D^\Theta_{\mathcal{A}})\bigr)
% \quad\text{and}\quad
% \mathcal O\bigl(r_{\Theta^*}(B\,M + \mathcal D_{\mathcal{F}}^\Theta)\bigr),
% $
% respectively, directly yielding the memory–compute trade-offs of Theorems~\ref{thm:flop-reduction} and~\ref{thm:io-bandwidth}.  

\paragraph{Streaming Activations}
In standard multi-head self‐attention calculation, one forms the full score matrix $\textrm{Softmax}(\frac{QK^\top}{\sqrt{d_{\Theta_{\mathcal A},h}}}) \in \mathbb{R}^{B \times M\times M}$, where $
d_{\Theta_{\mathcal A},h} = \mathcal D^{\Theta}_{\mathcal A}\;/\;H$,  $\forall \; h \in \{1,\dots,H\}
$, \(B\) is batch size, and \(M\) is sequence length, incurring $\mathcal O(BM^2)$ off‐chip memory and bandwidth per head.  FlashAttention \cite{dao2022flashattention, dao2023flashattention2} instead tiles and streams the queries, keys, and values, small blocks of $Q$, $K$ and $V$, sized to fit the GPU's shared memory per streaming multiprocessor, are loaded on-chip, dot-products are then computed and softmaxed with streaming, and results are immediately accumulated in full FP32 precision without ever materializing the full $B \times H \times M^2$ matrix on HBM. Consequently, HBM memory complexity for attention computation drops from $\mathcal O\bigl(BHM^2\bigr)$
to $\mathcal O(BH\,M\,d_{\Theta_{\mathcal A},h})$ for $H$ heads attention, enabling state-of-the-art attention computation performance in both training and inference.

\begin{algorithm}[t]
  \small
  \caption{FlashSVDAttention Forward Pass}\label{alg:combined}
  \begin{algorithmic}[1]
    \STATE \textbf{\textsc{Part I. Subroutine \textsc{load\_tiles}}:}
    \STATE \textbf{Input:} Factor $P_a\in\mathbb{R}^{B\times B_M \times r_{\Theta^*}}$, projection $V_a\in\mathbb{R}^{r_{\Theta^*}\times d_{\Theta_{\mathcal{A}},h}}$, bias $b_a\in\mathbb{R}^{d_{\Theta_{\mathcal{A}},h}}$, tile size $B_R$, rank $r_{\Theta^*}$, padded rank $R = \lceil \frac{r_{\Theta^*}}{B_R} \rceil \cdot B_R$
    \STATE \textbf{Output:} Reconstructed tiles $A_{\mathrm{tile}}$ for each $r$ with padding mask
    \FOR{$r = 0$ \TO $R$ \textbf{step} $B_R$}
      \STATE $A_{\mathrm{tile}} \gets P_a[:,\,:,\,r:r+B_R]\;V_a[r:r+B_{R},:]\;+\;b_a$
    \ENDFOR
    %\STATE{\rule{\linewidth}{0.4pt}}
    \STATE \textbf{\textsc{Part II. FlashSVDAttention Forward Pass}:}
    \STATE \textbf{Input:} Low‐rank factors $P_q,P_k,P_v\in\mathbb{R}^{B\times M\times r_{\Theta^*}}$, projections $V_q,V_k,V_v\in\mathbb{R}^{r_{\Theta^*}\times d_{\Theta_{\mathcal{A}},h}}$, biases $b_q,b_k,b_v\in\mathbb{R}^{d_{\Theta_{\mathcal{A}},h}}$, tile size $B_M$ and $B_R$, sequence length $M$
    \STATE \textbf{Output:} Output $O\in\mathbb{R}^{B\times M\times d_{\Theta_{\mathcal{A}},h}}$
    \STATE Initialize $O \gets 0$
    \FOR{$\ell = 0$ \TO $M-1$ \textbf{step} $B_M$}
      \STATE \texttt{Reconstruct Q block on‐chip}
      %\STATE $Q_{\mathrm{tile}} \gets P_q[:,\,\ell:\ell+B_M,\,:]\;V_q\;+\;b_q$
      \STATE $Q_{\mathrm{tile}} \gets \rm LOAD\_TILES(P_q[:,\,\ell:\ell+B_M,\,:],V_q[:,\,:,\,\ell:\ell+B_M],b_q,B_R)$
      \FOR{$m = 0$ \TO $M-1$ \textbf{step} $B_M$}
        \STATE \texttt{Reconstruct K and V blocks on‐chip}
        %\STATE $K_{\mathrm{tile}} \gets P_k[:,\,m:m+B_M,\,:]\;V_k\;+\;b_k$
        \STATE $K_{\mathrm{tile}} \gets \rm LOAD\_TILES(P_k[:,\,\ell:\ell+B_M,\,:],V_k[:,\,\,:,\ell:\ell+B_M],b_k,B_R)$
        %\STATE $V_{\mathrm{tile}} \gets P_v[:,\,m:m+B_M,\,:]\;V_v\;+\;b_v$
        \STATE $V_{\mathrm{tile}} \gets \rm LOAD\_TILES(P_v[:,\,\ell:\ell+B_M,\,:],V_v[:,\,:,\,\ell:\ell+B_M],b_v,B_R)$
        \STATE \texttt{Compute attention scores and softmax}
        \STATE $\text{scores} \gets Q_{\mathrm{tile}}\,(K_{\mathrm{tile}})^{\!\top} \,/\,\sqrt{d_{\Theta_{\mathcal{A}},h}}$
        \STATE $\alpha \gets \mathrm{\textsc{stream-softmax}}(\text{scores},\,\mathrm{axis}=-1)$ \\
        \cite{dao2022flashattention}
        \STATE $O[:,\,\ell:\ell+B_M,\,:] \;+\!=\; \alpha\,V_{\mathrm{tile}}$
      \ENDFOR
    \ENDFOR
  \end{algorithmic}
\end{algorithm}

\paragraph{FlashSVDAttention}  
FlashSVDAttention (Theorem~\ref{thm:flashsvd-mem}) enable rank-awared streaming by consuming only the low-rank factors of each head's query, key, and value projections:
\[
P_q,\,P_k,\,P_v \;\in\;\mathbb{R}^{B\times M\times r_{\Theta^*}},
\quad
V_q,\,V_k,\,V_v \;\in\;\mathbb{R}^{r_{\Theta^*}\times d_{\Theta_{\mathcal{A}},h}}.
\]
where the $P_a$ factors are formulated from $XU_a$, and $V_a$ $\forall \; a \in \{q,k,v\}$. 
Internally the kernel tiles the \((M,d_{\Theta_{\mathcal{A}},h})\) plane into small \((B_M,d_{\Theta_{\mathcal{A}},h})\) blocks that stream through on-chip SRAM.  For each tile it (1) forms the local low-rank query block \(Q_{\rm tile}=P_{q, {\rm{tile}}}\,V_{q, {\rm{tile}}} + b_{q}\) with iterative accumulation over $r_{\Theta^*}$, (2) multiplies and accumulates score with corresponding reconstructed key blocks \(K_{\rm tile}=P_{k, {\rm tile}}\,V_{k, {\rm tile}}+b_{k}\), (3) applies a row-wise softmax, and (4) weights and reduces the value factors \(P_{v, {\rm tile}}\,V_{v, {\rm tile}}\) to produce the output slice.  By fusing the two small GEMMs and the softmax into one pass, and never materializing the full \(B\times H\times M\times d_{\Theta_{\mathcal{A}},h}\) score or output buffers off-chip, FlashSVDAttn attains the rank-aware \( \mathcal{O}\bigl(Hr_{\Theta^*}(B\,M+d_{\Theta_{\mathcal{A}},h})\bigr)\) HBM bound proven in Theorem~\ref{thm:multihead-bound}. %while keeping high GPU occupancy via per-tile parallelism and Triton’s shared-memory tiling.  

\paragraph{FlashSVDFFN}
Once the FlashAttention alleviated the peak‐memory burden of self‐attention \cite{dao2022flashattention,dao2023flashattention2}, the feed‐forward network (FFN) becomes the new dominant factor (theorem~\ref{thm:ffn-upper}). Even when $W_i$ and $W_o$ admit low‐rank decompositions  \(W_i^{[r_{\Theta^*}]}= U_i\,V_i\) and \(W_o^{[r_{\Theta^*}]}= U_o\,V_o\) with ranks $r_{\Theta^*}\ll D_{\mathcal{F}}^{\Theta}$, the intermediate
$\phi(XU_iV_i + b_i)\in\mathbb{R}^{(B\times M)\times D_{\mathcal{F}}^{\Theta}}$
must still be materialized in full, dominating HBM despite in low-rank case (theorem~\ref{thm:ffn-upper}). We propose two rank‐aware streaming FFN variants to eliminate this large interim activation that trade off compute versus I/O: a mild version that maximizes reuse of vendor‐tuned GEMM kernels at the cost of extra off‐chip traffic, and an extreme version that enable zero I/O to HBM. 

% To eliminate this large intermediate activation and enable true rank-awareness, we propose two rank‐aware streaming FFN variants that trade off compute versus I/O: a mild version that maximizes reuse of vendor‐tuned GEMM kernels at the cost of extra off‐chip traffic, and an extreme version that enable zero I/O to HBM. 

\paragraph{FlashSVDFFN V1 (Rank‐Aware Fusion)}
 First, the input is projected into the factor space with $P = X\,U_i$ via the cuBlas GEMM, storing  $(B\times M)\times r_i$ sized intermediate in HBM. Second, a lightweight streamed kernel applies the nonlinearity and the next projection in one fused pass by computing $S = \phi(P\,V_i + b_i)\,U_o$ according to algorithm~\ref{alg:flashsvdffn-v1} without ever materializing the full $(B\times M)\times D_{\mathcal{F}}^{\Theta}$ activation. Finally, the output is reconstructed with $O = S\,V_o + b_o$ via a second cuBlas GEMM. Because both $P$ and $S$ remain bounded by the truncated rank $r_{\Theta^*}$, the overall computation and intermediate memory scale with $r_{\Theta^*}$ with the tradeoff for two extra I/O with HBM, directly yielding the tight bounds on activation memory and FLOP reduction stated in Theorem~\ref{thm:flashsvdffn-v1-memory} and Theorem~\ref{thm:flop-reduction}.
 
 %Given that the $P$ and $S$ are rank-aware, the computation now rank-bounded, so with lower-rank, we expected to see the speedup in our method and the intermediate memory now remain activation aware (corollary \ref{cor:ffn-compression-tight-bound}, theorem~\ref{thm:flashsvdffn-v1-memory}, theorem~\ref{thm:flop-reduction}.

%In the mild variant, the streamed kernel is kept minimal by performing only the element-wise activation and a small $r_i\times r_o$ GEMM, ensuring highly optimized cuBLAS routines two large GEMMs ($XU_i$ and $SV_o$). This design incurs two extra reads and one write per activation block to HBM but preserves peak FLOP throughput. In the extreme variant, we fuse more of the activation and projection logic into a single kernel that requires unrolling loops over the rank dimensions, which drive nearly all intermediate data movement on‐chip, at the expense of additional compute overhead. Empirically \textcolor{blue}{some refs here}. \textcolor{blue}{we need some experiment result (in ablation) here}. % the version 1 reduces peak memory by up to 40\% with under a 5\% runtime increase, whereas the extreme version achieves up to 60\% peak‐memory reduction at the cost of a 10–15\% compute overhead. This flexibility makes FlashSVDFFN a practical primitive for memory‐constrained transformer inference.

% on each activation block \(X\in\mathbb{R}^{B\times M\times \mathcal D^\Theta_{\mathcal{A}}}\) and

\paragraph{FlashSVDFFN V2 (Extreme-Case)} employs a single fused GEMM–Activation–GEMM GPU kernel to stream the low-rank FFN forward pass by directly loading $X$, projection factors $U_i,V_i,U_o,V_o$, and biases $b_i,b_o$ on-the-fly, incuring no additional arithmetic cost or asymptotic complexity compared to a dense FFN (identical $\mathcal O\bigl(B\,M\,r_{\Theta^*}\,D_{\mathcal{F}}^\Theta + B\,M\,r_{\Theta^*}\,D_{\mathcal{A}}^\Theta\bigr)$ FLOPs). This unified pipeline supports arbitrary activations $\phi(\cdot)$, and theoretically pushes the lower bound for variable part of rank-aware activation compression to zero (see Theorem~\ref{thm:flashsvdffn-v2-memory}).  By launching one kernel block per tile and accumulating directly into the output buffer, we avoid any materialized \((B\times M)\times\mathcal D_{\mathcal{F}}^\Theta\) mid-buffers, achieving the \(O(r_{\Theta^*}(B\,M+\mathcal D_{\mathcal{F}}^\Theta))\) HBM bound of Theorem~\ref{thm:flashsvdffn-v1-memory}.

\begin{figure*}[t!]
  \centering
  \includegraphics[width=0.65\linewidth]{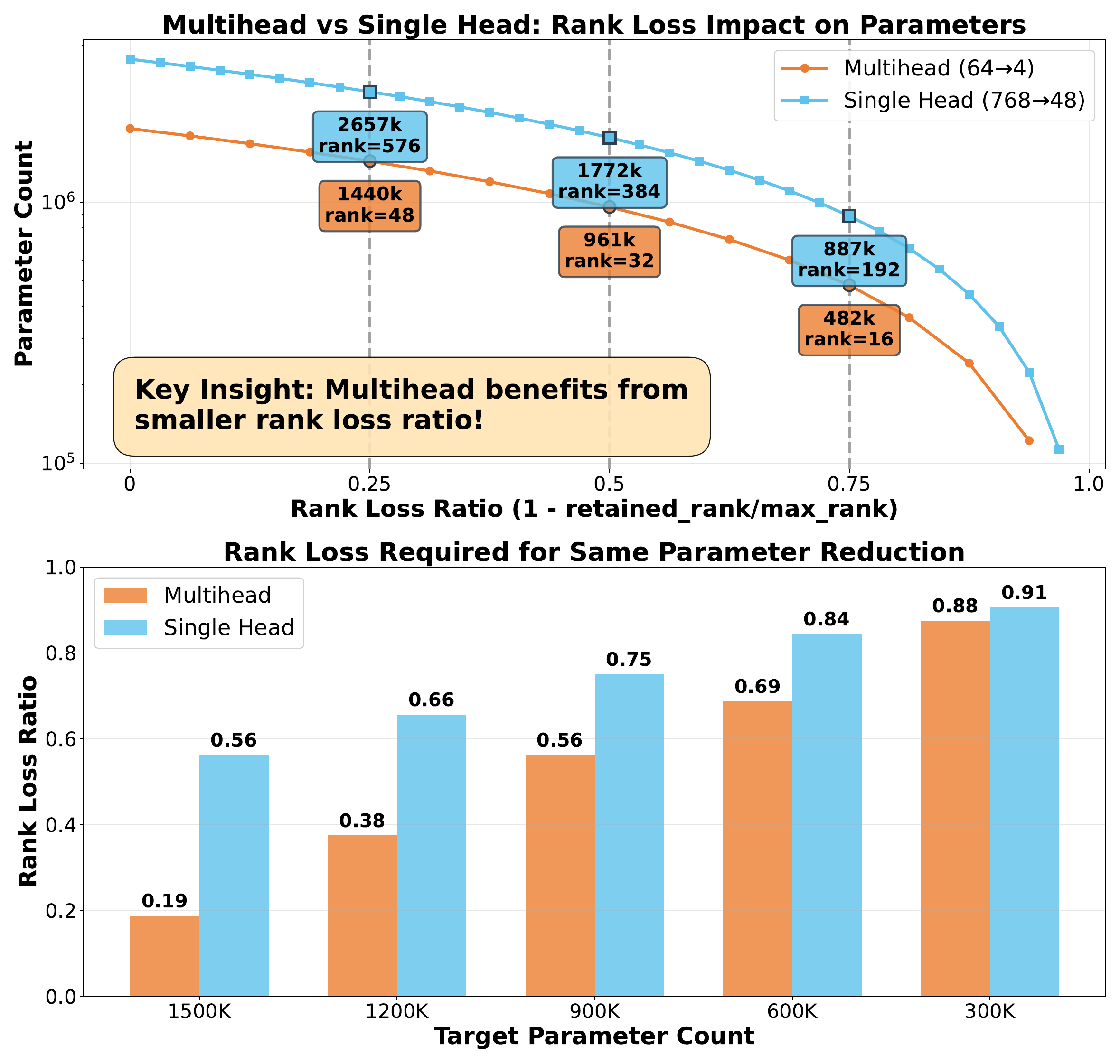}
  \caption{Multi‑head SVD compresses model attention projection matrices more gently than single‑head SVD.}
  \label{fig:rank-loss-comparison}
\end{figure*}

% \paragraph{Multi‐Head Attention Enables Efficient SVD} For current, some of the popular compression method applies a single-head SVD compression strategy \cite{yuan2023asvd,wang2025dobi} to a full \(\!\mathcal D^\Theta_{\mathcal{A}}\times \mathcal D^\Theta_{\mathcal{A}}\) projection only reduces parameters when  
% $r_{\Theta^*} < \tfrac{\mathcal D^\Theta_{\mathcal{A}}}{2},$  
% and in practice often degrades model quality \cite{yuan2023asvd,lin2024modegpt}.  In contrast, when viewing the attention projection matrices in multi-head, each head can be factorized independently.  By Theorem~\ref{thm:multihead-bound}, the off‐chip memory bound for \(\mathcal D^\Theta_{\mathcal{A}}\) separate heads becomes  
% $
% \mathcal O\bigl(\mathcal D^\Theta_{\mathcal{A}}\,r_{\Theta^*}\,(B\,M + d_{\Theta_{\mathcal{A}},h})\bigr)
% \;=\;
% \mathcal O\!\Bigl(r_{\Theta^*}\,( \mathcal D^\Theta_{\mathcal{A}} + \mathcal D^\Theta_{\mathcal{A}}\,B\,M )\Bigr),
% $  
% and the compression threshold for nontrivial savings drops to $r_{\Theta^*} < \frac{\mathcal{D}_\Theta}{\,H+1\,}.$  Figure~\ref{fig:rank-loss-comparison}, the single head consistently requires a more drastic cut in rank loss ratio (computed with $1 - \frac{\textrm{retained rank}}{\rm max\ rank}$), for instance, to compress BERT to 1.5 million head projection matrices parameter, it requires 56\% rank reduction while multihead case only requires 19\% cut in rank. 

\paragraph{Multi‐Head Attention Amplifies SVD Savings}
Many recent compression schemes (e.g., ASVD \cite{yuan2023asvd}, Dobi-SVD \cite{wang2025dobi}) apply SVD to the entire $\mathcal D_{\mathcal A}^\Theta \times \mathcal D_{\mathcal A}^\Theta$ attention projection as a single matrix.  This single head approach only yields parameter reduction when $r_{\Theta^*} < \tfrac12\,\mathcal D_{\mathcal A}^\Theta,$ and in practice forces drastic rank cuts that frequently degrade accuracy and limits the overall compression capacity of SVD-based method.  In contrast, the native multi‐head structure splits $\mathcal D_{\mathcal A}^\Theta$ into $H$ heads of size $d_{\Theta_{\mathcal{A}},h}=\mathcal D_{\mathcal A}^\Theta/H$.  By applying truncated SVD per head, the threshold for parameter compression drops to $r_{\Theta^*} < \tfrac{\mathcal D_{\mathcal A}^\Theta}{\,H+1\,}\,.$ As Figure \ref{fig:rank-loss-comparison} illustrates, compressing BERT’s head projections to 1.5 M parameters requires a 56\% rank reduction under single‐head SVD, risking severe quality loss, whereas multi‐head SVD needs only a 19\% cut, achieving the same compression with far gentler rank pruning.

\begin{algorithm}[t]
  \small
  \caption{FlashSVDFFN Forward Pass V1}\label{alg:flashsvdffn-v1}
  \begin{algorithmic}[1]
    \REQUIRE Input $X\in\mathbb{R}^{B\times M\times \mathcal D_{\mathcal A}^\Theta}$;  
      SVD factors $U_i\in\mathbb{R}^{\mathcal D_{\mathcal A}^\Theta\times r_{\Theta^*}}$,  
      $V_i\in\mathbb{R}^{r_{\Theta^*}\times\mathcal D_{\mathcal F}^\Theta}$, $b_i\in\mathbb{R}^{\mathcal D_{\mathcal F}^\Theta}$;\\
      $U_o\in\mathbb{R}^{\mathcal D_{\mathcal F}^\Theta\times r_{\Theta^*}}$,  
      $V_o\in\mathbb{R}^{r_{\Theta^*}\times \mathcal D_{\mathcal A}^\Theta}$, $b_o\in\mathbb{R}^{\mathcal D_{\mathcal A}^\Theta}$;  
      number of blocks $G$; nonlinearity $\phi(\cdot)$.
    \ENSURE $O\in\mathbb{R}^{B\times M\times \mathcal D_{\mathcal A}^\Theta}$
    \STATE $B_M \leftarrow \lceil M/G\rceil,\quad B_{D_{\mathcal{F}}^\Theta}\leftarrow \lceil \mathcal D_{\mathcal F}^\Theta/G\rceil$
    \STATE \texttt{GEMM 1st: project into factor‐space (cuBlas)}
    \STATE $P \leftarrow X\,U_i$\COMMENT{$(BM)\times r_{\Theta^*}$}
    \STATE \texttt{On-Chip: rank-aware activation streaming}
    \FOR{$\ell=0$ \TO $M-1$ \textbf{ step } $B_M$}
      \STATE initialize $Z_{\mathrm{tile}}\leftarrow 0_{B\times B_M\times r_{\Theta^*}}$
      %\COMMENT{accumulates in rank space}
      \FOR{$d=0$ \TO $\mathcal D_{\mathcal F}^\Theta-1$ \textbf{ step } $B_{D_{\mathcal{F}}^\Theta}$}
        \STATE $V_i^{(d)}\leftarrow V_i[:,\,d:d+B_{D_{\mathcal{F}}^\Theta}]$\COMMENT{$(r_{\Theta^*}\times B_{D_{\mathcal{F}}^\Theta})$}
        \STATE $U_o^{(d)}\leftarrow U_o[d:d+B_{D_{\mathcal{F}}^\Theta},\,:]$\COMMENT{$(B_{D_{\mathcal{F}}^\Theta}\times r_{\Theta^*})$}
        \STATE $b_i^{(d)}\leftarrow b_i[d:d+B_{D_{\mathcal{F}}^\Theta}]$\COMMENT{$(B_{D_{\mathcal{F}}^\Theta})$}
        \STATE $Y \leftarrow \phi\bigl(P\,V_i^{(d)} + b_i^{(d)}\bigr)$ \COMMENT{$B\times B_M\times B_{D_{\mathcal{F}}^\Theta}$}
        \STATE $Z_{\mathrm{tile}} \mathrel{+}= Y\,U_o^{(d)}$  
        \COMMENT{$(B\times B_M\times B_{D_{\mathcal{F}}^\Theta})\times(B_{D_{\mathcal{F}}^\Theta}\times r_{\Theta^{*}})$}
      \ENDFOR
      \STATE $Z \leftarrow Z_{\mathrm{tile}} $ \COMMENT{$ (BM)\times r_{\Theta^{*}}$}
    \ENDFOR
  \STATE \texttt{GEMM 2nd: project to model‐space (cuBlas)}
      \STATE $O \leftarrow Z\,V_o + b_o$\quad\COMMENT{$(BM)\times \mathcal D_{\mathcal A}^\Theta$}
  \end{algorithmic}
\end{algorithm}

\section{Analysis}
%%%%%%%%%%%%%%%%%%%%%%%%%%%%%%%%%%%%%%%%%%%%%%%%%%%%%%%%%%%%%%%%%%%%%%%%%%%%
% Section: Analysis
%%%%%%%%%%%%%%%%%%%%%%%%%%%%%%%%%%%%%%%%%%%%%%%%%%%%%%%%%%%%%%%%%%%%%%%%%%%%
% \section{Analysis}
% \label{sec:analysis}

%This section distills the key theoretical insights behind FlashSVD’s efficiency, with full derivations deferred to Appendix. It begins with peak memory comparisons across attention and FFN modules under both dense and compressed configurations. The analysis then contrasts the compute-memory trade-offs of the two FFN variants. Next, it quantifies the memory savings achieved by reducing the low-rank dimension $r_{\Theta^*}$ in each sub-module. Lastly, it characterizes theoretical speedups in terms of arithmetic operations, memory bandwidth usage, and end-to-end latency under the roofline model.

%FlashSVD provides a promising framework with theoretically lower activation bounds for major overheads of activations during inference. In this section, we provide the summary of most important theorems to assist understanding of tradeoffs in FlashSVD

FlashSVD establishes provable memory efficient upper bounds on dominant activations during inference by exploiting low-rank structure. Below, we summarize the key theorems that characterize these activation bounds and clarify the fundamental time–memory trade-offs in FlashSVD.

%------------------------------------------------------------
\paragraph{Peak Memory Footprint Analysis}
\label{ssec:analysis-memory}

Table~\ref{tab:analysis-memory} demonstrates that combining \emph{FlashSVDAttention} with \emph{FlashSVDFFN} reduces the activation memory footprint from \(\mathcal O\bigl(BM( \mathcal D_\mathcal{F}^{\Theta}+\mathcal D_\mathcal{A}^{\Theta}+M)\bigr)\) to \(\mathcal O\bigl(r_{\Theta^*}(BM+\mathcal D_\mathcal{F}^{\Theta}+\mathcal D_\mathcal{A}^{\Theta})\bigr)\). The FFN memory can be further reduced using \emph{FlashSVDFFN‑V2}, and ultimately brought variable part to zero. In the typical regime where \(r_{\Theta^*} \ll \mathcal{D}_\mathcal{A}^\Theta,\, \mathcal{D}_\mathcal{F}^\Theta,\, BM\), these optimizations yield a ranked-bounded reduction in HBM activation.

%\textcolor{blue}{$10\times$-$50\times$} reduction in off‑chip activation traffic.

%------------------------------------------------------------
% \paragraph{Speed-Memory Trade‑off of FFN Variants} \label{ssec:analysis-ffn}

% \emph{FlashSVDFFN‑V1} leverages built-in GEMM primitives to achieve high parallelism of GPU while ensure the rank-awareness through tradeoff two \(P\) and \(Z\) and additional I/O to the HBM, in off‑chip memory (Theorem~\ref{thm:flashsvdffn-v1-memory}). In contrast, \emph{FlashSVDFFN‑V2} eliminates all intermediate memory by replacing GEMM operations with a custom tiled kernel that fuses low-rank projection, activation, and output projection into a single pass (Theorem~\ref{thm:flashsvdffn-v2-memory}). 

%------------------------------------------------------------
\paragraph{Rank Sensitivity}
\label{ssec:analysis-rank}
The memory footprint of FlashSVD depends critically on the selected SVD rank $r_{\Theta^*}$.
For the attention module, lowering the rank by one in a single head saves $\Delta\mathcal{M}_{\mathrm{flash\text{-}svd\text{-}attn\text{-}single}} = \mathcal{O}(BM + \mathcal D^\Theta_{\mathcal{A}})$ words, as proved in Proposition~\ref{prop:rank-sensitivity}.  Applied across all $H$ heads, the aggregate saving is $\Delta\mathcal{M}_{\mathrm{flash\text{-}svd\text{-}attn\text{-}multi}}
=\mathcal{O}(H\!\left(BM + d_{\Theta_\mathcal{A},h}\right))$ (Proposition~\ref{prop:multihead-sensitivity}).

An analogous result holds for the feed‑forward network.  Proposition~\ref{prop:ffn-rank-sensitivity}, \ref{prop:ffn-rank-sensitivity-v2} shows that a unit decrease in rank reduces off‑chip memory by
$$
\Delta\mathcal{M}_{\mathrm{flash\text{-}svd\text{-}ffn}} \leq \mathcal{O}\!\left(BM + D_{\mathcal{F}}^{\Theta}\right),
$$
demonstrating that FFN layers benefit from rank reduction along both the batch-sequence and feature dimensions.

% In addition to per-head compression, grouping heads into shared-rank SVD blocks provides an effective interpolation between fine-grained and monolithic compression. Compressing \( H \) heads in \( G \) groups yields a memory complexity of
% \[
% \mathcal{M} = \mathcal{O}\big(r_{\Theta}^*(BM + D_{\mathcal{A}}/G)\big),
% \]
% as established in Proposition~\ref{prop:grouped-tradeoff}. This trade-off governs the smooth Pareto frontier between full-rank expressiveness and memory sharing, which we visualize in \textcolor{blue}{Fig.\,TODO.}

% \input{AnonymousSubmission/LaTeX/tables/exp_bert_full_results}

%------------------------------------------------------------

\begin{table}[h]
\centering
\scriptsize
\setlength{\tabcolsep}{3.5pt}  % tighten columns
%\vspace{-0.5em}
\caption{Peak HBM $\mathcal{M}$  Complexity for each sub‑module. See Theorems~\ref{thm:flashsvd-mem}, \ref{thm:multihead-bound}, \ref{thm:flashsvdffn-v1-memory}, \ref{thm:flashsvdffn-v2-memory} in Appendix, Supplementary.}
\label{tab:analysis-memory}
\begin{tabular}{@{}lcc@{}}
\toprule
\textbf{Module} & $\mathcal{M}_{\text{dense}}$ & $\mathcal{M}_{\text{FlashSVD}}$ \\ \midrule
Single‑Head Attn. & $3BM\mathcal D_\mathcal{A}+BM^2$ & $\mathcal O(r_{\Theta^*}(BM+\mathcal D_\mathcal{A}))$ \\
Multi‑Head Attn.  & $BM\mathcal D_\mathcal{A}$       & $\mathcal O(r_{\Theta^*}(BM+\mathcal D_\mathcal{A}))$ \\
FFN‑V1            & $BM\mathcal D_\mathcal{F}$       & $\mathcal O(r_{\Theta^*}(BM+\mathcal D_\mathcal{F}))$ \\
FFN‑V2            & $BM\mathcal D_\mathcal{F}$       & $\mathcal O(r_{\Theta^*}\mathcal D_\mathcal{F})$ \\ 
\bottomrule
\end{tabular}
\end{table}

\paragraph{Speedup and Latency Analysis}
\label{ssec:speedup-latency}
We analyze the performance benefits of FlashSVD along three dimensions: FLOP count, memory I/O volume, and end-to-end latency. 

Theoretical FLOP savings arise from replacing dense matrix multiplications with low-rank projections: the full-rank computation costs 
$\mathcal{O}\big((D_{\mathcal{A}}^{\Theta})^2MB + D_{\mathcal{A}}^{\Theta}D_{\mathcal{F}}^{\Theta}MB\big)$, while the low-rank version requires only 
$\mathcal{O}\big(r_{\Theta^*}^2 MB + r_{\Theta^*}D_{\mathcal{A}}^{\Theta}MB + r_{\Theta^*}D_{\mathcal{F}}^{\Theta}MB\big)$. 
The resulting asymptotic speedup factor is given in Theorem~\ref{thm:flop-reduction}:
\[
\Theta\left(\frac{(D_{\mathcal{A}}^{\Theta})^2 + D_{\mathcal{A}}^{\Theta} D_{\mathcal{F}}^{\Theta}}{r_{\Theta^*}^2 + r_{\Theta^*} D_{\mathcal{A}}^{\Theta} + r_{\Theta^*} D_{\mathcal{F}}^{\Theta}}\right).
\]

In terms of data movement, the full-rank model must load all query/key/value activations and FFN outputs, with total volume \(\mathcal{O}(BM(D_{\mathcal{A}}^{\Theta} + D_{\mathcal{F}}^{\Theta}))\) 
as shown in Theorem~\ref{thm:io-bandwidth}. In contrast, FlashSVD only needs to read and write low-rank components, yielding reduced volume 
\[
\mathcal{O}(MB r_{\Theta^*} + r_{\Theta^*} D_{\mathcal{A}}^{\Theta} + r_{\Theta^*} D_{\mathcal{F}}^{\Theta}),
\] 
also see Theorem~\ref{thm:io-bandwidth}.

Finally, following the roofline model, the end-to-end latency per layer is bounded below by the slower of the compute- and bandwidth-limited paths:
\begin{equation}
T \;\ge\; \max\Bigl\{
    \frac{\mathrm{FLOPs}}{\mathrm{PeakFLOP/s}}\,,\;
    \frac{4\,N_{\mathrm{bytes}}}{\beta}
  \Bigr\}\nonumber,
\end{equation}
where $N_{\text{bytes,in}}$ is the number of input bytes and $\beta$ is the memory bandwidth in bytes/s, as formalized in Theorem~\ref{thm:io-bandwidth}.

\section{Experiment}

\subsection{Experimental Setup}

\newcommand{\fmidrule}{\specialrule{0.1pt}{0.2ex}{0.2ex}}

\begin{table*}[t]
\centering
\tiny
\caption{Main Results: Base Mem (Base), Total Mem (Peak), Transient Mem (Tran), Latency (Lat), Accuracy (Acc), Compression Ratio (Ratio), batch = 64, ratio = 25\%, 50\%, \textbf{All Training-Free}}
%\caption{Main results: for batch size 64 at 25\% and 50\% parameter ratios. Base: persistent memory (MiB); Peak GPU memory usage (Peak, MiB); activation memory (Trans, MiB); inference latency (Lat, ms); Acc: accuracy (\%); parameter count ratio (Ratio). All methods are training‑free.}
\label{tab:main_result}
\begin{tabular}{%
  c            % Methodology (centered)
  c            % Base
  *{3}{ccccc}  % three datasets × [Peak, Tran, Lat, Acc, Ratio]
}
\toprule
\multicolumn{17}{c}{\textbf{BERT Experiments}} \\
\midrule
\textbf{Methodology}
  & \textbf{Base}
  & \multicolumn{5}{c}{\textbf{sst2}}
  & \multicolumn{5}{c}{\textbf{qqp}}
  & \multicolumn{5}{c}{\textbf{mnli}} \\
\cmidrule(lr){3-7} \cmidrule(lr){8-12} \cmidrule(lr){13-17}
  & 
  & Peak & Tran & Lat & Acc & Ratio
  & Peak & Tran & Lat & Acc & Ratio
  & Peak & Tran & Lat & Acc & Ratio \\
\midrule 
Dense
  & 418.7
  & 695.0 & 277.3 &  79.7 & 92.32 & 1.00
  & 971.1 & 553.5 &  147.3 & 90.93 & 1.00
  & 1547.4 & 1129.7 & 310.2 & 84.06 & 1.00 \\
\fmidrule
\multicolumn{17}{c}{\textbf{\textit{Parameter Ratio: 25\%}}} \\
\fmidrule 
Vanilla SVD
  & 332.4
  & 741.6 & 409.1 &  161.1 & 85.16 & -
  & 1149.7 & 817.3 &  200.2 & 72.80 & -
  & 1990.0 & 1657.5 &  359.7 & 66.73 & - \\
%\fmidrule
FlashSVD v1
  & 332.4
  & 576.7 & 244.2 & 188.8 & 85.16 & \textbf{0.60}
  & 785.8 & 453.4 &  264.8 & 72.82 & \textbf{0.55}
  & 1230.2 & 897.7 &  490.5 & 66.71 & \textbf{0.54} \\
%\fmidrule
FlashSVD v2
  & 332.4
  & 576.7 & 244.2 &  253.5 & 85.16 & \textbf{0.60}
  & 785.8 & 453.4 &  400.9 & 72.81 & \textbf{0.55}
  & 1230.2 & 897.7 &  753.0 & 66.71 & \textbf{0.54} \\
%\fmidrule
FWSVD 
  & 332.4
  & 913.6 & 581.1 &  83.8 & 89.91  & -
  & 1692.7 & 1360.3 &  194.0 & 84.80 & -
  & 3853.0 & 3520.5 &  463.2 & 78.01 & - \\
%\fmidrule
FlashFWSVD 
  & 332.4
  & 541.6 & 209.2 &  139.6 & 89.91 & \textbf{0.36}
  & 756.8 & 424.4 & 269.3 & 84.78 & \textbf{0.31}
  & 1213.2 & 880.8 & 479.9 & 77.90 & \textbf{0.25} \\
\fmidrule
\multicolumn{17}{c}{\textbf{\textit{Parameter Ratio: 50\%}}} \\
\fmidrule 
Vanilla SVD
  & 249.7
  & 653.5 & 403.8 &  150.8 & 66.88 & -
  & 1057.1 & 807.4 &  180.4 & 64.05 & -
  & 1888.4 & 1638.7 &  321.1 & 37.92 & - \\
%\fmidrule
FlashSVD v1
  & 249.7
  & 461.6 & 211.8 & \textbf{141.6} & 66.88 & \textbf{0.52}
  & 673.3 & 423.5 &  193.5 & 64.05 & \textbf{0.52}
  & 1120.7 & 870.9 &  350.6 & 37.81 & \textbf{0.53} \\
%\fmidrule
FlashSVD v2
  & 249.7
  & 461.6 & 211.8 &  192.9 & 66.88 & \textbf{0.52}
  & 673.3 & 423.5 &  313.2 & 64.05 & \textbf{0.52}
  & 1120.7 & 870.9 &  584.7 & 37.82 & \textbf{0.53} \\
%\fmidrule
FWSVD 
  & 249.7
  & 821.1 & 571.3 & 66.5 & 79.44 & -
  & 1592.2 & 1342.5 &  165.5 & 63.75 & -
  & 3734.3 & 3484.5 &  427.8 & 51.06 & - \\
%\fmidrule
FlashFWSVD 
  & 249.7
  & 453.4 & 203.7 & 101.5 & 79.44 & \textbf{0.36}
  & 665.3 & 415.6 &  189.9 & 63.75 & \textbf{0.31}
  & 1112.5 & 862.8 & \textbf{341.4} & 50.92 & \textbf{0.25} \\
\midrule
\multicolumn{17}{c}{\textbf{RoBERTa Experiments}} \\
\midrule
Dense
  & 475.5
  & 753.0 & 277.5 &  78.1 & 94.17 & 1.00
  & 1029.1 & 553.6 &  110.1 & 91.3 & 1.00
  & 1605.4 & 1129.9 & 217.5 & 87.59 & 1.00 \\
%%%%%%%%%%%%%%%%%%%%%%%%%%%%%
\fmidrule
\multicolumn{17}{c}{\textbf{\textit{Parameter Ratio: 25\%}}} \\
\fmidrule 
Vanilla SVD
  & 390.3
  & 883.4 & 493.1 &  119.2 & 49.20 & - 
  & 1375.5 & 985.3 &  148.6 & 60.63 & - 
  & 2383.8 & 1993.5 &  225.0 & 33.37 & - \\
%\fmidrule
FlashSVD v1
  & 390.3
  %& 1288.0 & 897.7 & 394.7 & 60.60 & - 
  & 634.5 & 244.2 &  184.0 & 49.20 & \textbf{0.49} 
  & 843.7 & 453.4 & 263.9  & 60.59 & \textbf{0.46} 
  & 1288.0 & 897.7 & 393.1 & 33.40 & \textbf{0.45} \\
%\fmidrule
FlashSVD v2
  & 390.3
  & 634.5 & 244.2 &  238.1 & 49.20 & \textbf{0.49}
  & 843.7 & 453.4 & 396.8 & 60.63 & \textbf{0.46}
  & 1288.0 & 897.7 & 744.4 & 33.42 & \textbf{0.45} \\
%\fmidrule
FWSVD 
  & 390.3
  & 875.4 & 485.1 &  110.5 & 89.29 & -
  & 1366.5 & 976.3 &  113.7 & 82.34 & -
  & 2374.8 & 1984.5 &  220.8 & 77.21 & - \\
%\fmidrule
FlashFWSVD 
  & 390.3
  & 599.5 & 209.2 &  127.4 & 89.00 & \textbf{0.43}
  & 814.7 & 424.4 &  217.1 & 82.11 & \textbf{0.43}
  & 1271.1 & 880.8 &  401.3 & 77.46 & \textbf{0.44} \\
\fmidrule
\multicolumn{17}{c}{\textbf{\textit{Parameter Ratio: 50\%}}} \\
\fmidrule 
Vanilla SVD
  & 307.6
  & 790.8 & 483.3 & 110.3 & 50.94 & -
  & 1274.0 & 966.4 &  100.7 & 62.07 & -
  & 2264.2 & 1956.7 &  247.7 & 31.16 & - \\
%\fmidrule
FlashSVD v1
  & 307.6
  & 519.4 & 211.8 & 148.8 & 51.34 & \textbf{0.44}
  & 731.1 & 423.5 & 170.1 & 62.08 & \textbf{0.44}
  & 1178.5 & 870.9 & 287.1 & 31.10 & \textbf{0.45} \\
%\fmidrule
FlashSVD v2
  & 307.6
  & 519.4 & 211.8 &  195.1 & 51.63 & \textbf{0.44}
  & 731.1 & 423.5 &  313.7 & 62.06 & \textbf{0.44}
  & 1178.5 & 870.9 &  570.8 & 31.07 & \textbf{0.45} \\
%\fmidrule
FWSVD 
  & 307.6
  & 782.9 & 475.3 &  103.0 & 71.09 & -
  & 1266.0 & 958.5 &  93.5 & 63.36 & -
  & 2256.1 & 1948.5 & 345.6 & 37.82 & - \\
%\fmidrule
FlashFWSVD 
  & 307.6
  & 511.3 & 203.7 &  105.5 & 69.49 & \textbf{0.43}
  & 723.2 & 415.6 &  156.7 & 63.31 & \textbf{0.43}
  & 1170.4 & 862.8 & \textbf{279.2} & 38.40 & \textbf{0.44} \\
\bottomrule
\end{tabular}
\end{table*}

\paragraph{Model Setup}
While the prior activation-aware inference methods (e.g., ASVD \cite{yuan2023asvd}, Palu \cite{chang2025palu}) have focused primarily on decoder-side KV-Cache compression for causal LMs, the encoder-side activation overhead remains largely unexplored.  FlashSVD naturally extends to decoder models, but since KV-Cache overheads are already well-studied, our evaluation concentrates on encoder architectures. Concretely, we target the most widely adopted compressed encoder pipelines, including naive SVD‐BERT, fine-tuned low-rank BERT (FWSVD-BERT), and our fully fused FlashSVD-BERT-applied to BERT and RoBERTa. To isolate the impact of our fusion kernels, we perform an ablation study-toggling FlashSVDAttn and FlashSVDFFN (with autotuned streaming block size) independently across diverse settings of inputs. Decoder-side experiments should follow a similar protocol on causal language models, but our primary focus remains the encoder.

\begin{table}[t]
  \centering
  \small
  % tighten up inter-column spacing
  \setlength\tabcolsep{4pt}
  \caption{Rank Awareness: latency and transient‐memory trade‐offs on SST‐2  as MHA/FFN rank varies, Sequence Length = 128 for SST-2 dataset and 256 for STS-B dataset}
  \label{tab:rank-memory-tradeoff}
  % first (headered) table
  \begin{tabular}{@{}l c c c c@{}}
    \toprule
    \shortstack{Dataset\\ \& Method} 
      & \shortstack{MHA\\Rank} 
      & \shortstack{Latency\\(ms)} 
      & \shortstack{Trans.\ Mem.\\(MB)} 
      & \shortstack{Peak.\ Mem.\\(MB)} \\ 
    \fmidrule
    \multirow{5}{*}{\shortstack{SST‐2\\ Vanilla SVD\\ (b = 32)}} 
      & 64 & 112.9 & 214.2 & 719.6 \\ % acc=0.9252
      & 56 & 114.0 & 214.2 & 708.6 \\ % acc=0.9196
      & 48 & 112.2 & 214.2 & 697.7 \\ % acc=0.9085
      & 32 & 110.1 & 214.2 & 675.7 \\ % acc=0.8783
      %& 16 & — & 214.2 & — \\
    \midrule 
    \multirow{5}{*}{\shortstack{SST‐2\\ FlashSVD\\ (b = 32)}} 
      & 64 & 166.0 & \textbf{136.5} & \textbf{641.9} \\ % acc=0.9252
      & 56 & 165.5 & \textbf{125.2} & \textbf{619.7} \\ % acc=0.9196
      & 48 & 164.9 & \textbf{118.2} & \textbf{601.7} \\ % acc=0.9085
      & 32 & 158.9 & \textbf{118.2} & \textbf{579.7} \\ % acc=0.8783
      %& 16 & — & — & — \\
    \midrule
    % Dataset 
    %   & \shortstack{FFN\\Rank} 
    %   & \shortstack{Latency\\(ms)} 
    %   & \shortstack{Trans.\ Mem.\\(MB)} 
    %   & \shortstack{Peak.\ Mem.\\(MB)} \\ 
    % \midrule
    \multirow{5}{*}{\shortstack{STS‐B\\ Vanilla SVD\\ (b = 32)}} 
      & 768 & 158.1 & 424.3 & 929.7 \\ % acc=0.9252
      & 384 & 132.7 & 413.1 & 756.5  \\ % acc=0.8638
      & 192 & 138.3 & 406.8 & 669.2  \\ % acc=0.6708
      & 96 & 126.7 & 404.1 & 626.0  \\ % acc=0.6708
    \fmidrule
    \multirow{5}{*}{\shortstack{STS‐B\\ FlashSVD\\ (b = 32)}} 
      & 768 & 285.8 & \textbf{250.6} & \textbf{756.0} \\ % acc=0.9252
      & 384 & 152.0 & \textbf{251.4} & \textbf{594.8}  \\ % acc=0.8638
      & 192 & \textbf{119.4} & \textbf{251.1} & \textbf{513.5}  \\ % acc=0.6730
      & 96 & \textbf{109.6} & \textbf{251.4} & \textbf{473.3}  \\ % acc=0.6708
    \bottomrule
  \end{tabular}  
  % \\
  % {\footnotesize%
  %   In the lower table, the three (unnamed) columns correspond, respectively, to
  %   the variable FFN rank, the fixed‐attention rank, and the same
  %   latency/transient‐memory metrics as above.
  % }
\end{table}

%  to quantify the effects of low-rank compression across both short and long contexts, , so that reductions directly reflect eased pressure on high-bandwidth memory and improved deployability on resource-constrained hardware
% Finally, to illustrate the trade-off between efficiency and compactness, w

% We assess model quality via task-specific classification accuracy on GLUE datasets, evaluating performance at input lengths of 128, 256, and 512 tokens. Memory efficiency is quantified by measuring peak activation usage during end-to-end inference, and reported in mebibytes (MiB) - using PyTorch’s profiler.  Compute efficiency is captured through wall-clock latency per batch (milliseconds) on an NVIDIA L40S GPU. We evaluate a family of models spanning a range of compression ratios - defined as the compressed parameter count divided by that of the dense baseline - and report how latency and memory savings scale as the parameter budget is varied.

\paragraph{Evaluation Metrics} For encoders, we assess model quality via task-specific classification accuracy on widely-recognized GLUE datasets \cite{wang2018glue}, evaluating performance at input lengths of 128, 256, and 512 tokens. Memory efficiency is quantified by measuring peak activation usage during end-to-end inference, and reported in mebibytes (MiB) - using PyTorch’s profiler. Compute efficiency is captured through wall-clock latency per batch (milliseconds) on an NVIDIA L40S GPU. We evaluate a family of models spanning a range of compression ratios (\(
\frac{\# \ \text{Param}\  {\text{Compressed}}}
     {\# \ \text{Param} \ {\text{Dense}}}
\))  divided by that of the dense baseline) and report how latency and memory savings scale as the parameter budget is varied. %by averaging over 10 runs with warmup.

\subsection{Main Results}

Table~\ref{tab:main_result} demonstrates that FlashSVD strikes an outstanding balance between model compression, memory footprint, and inference throughput on both BERT-Base and RoBERTa-Base across multiple GLUE benchmarks. Regardless of the compression ratio, Flashsvd incurs no measurable degradation relative to the vanilla SVD and FWSVD method (maximum difference of only 0.58 \% for RoBERTa at 50 \% parameters on MNLI). At the same time, it slashes transient activation memory by 69 \% on SST-2 and QQP and by 75 \% on MNLI, reducing peak inference usage to 1,213 MiB, well below the 1,547 MiB of the original BERT and the 1,990 MiB of a naive SVD, compressed variant. These results reveal that the naive SVD-based compression, despite lowering parameter counts, can actually inflate runtime memory demands and undermine deployment in resource-constrained environments. %By embedding rank-awareness directly into the inference pipeline, FlashSVD avoids large dense intermediates and delivers a method-agnostic compression scheme that empowering the SVD method for edge-focused applications. 

Moreover, FlashSVD maintains competitive inference latency while entirely eliminating large dense intermediate tensors, despite the additional I/O tradeoff. On MNLI, with 50 \% of the original parameters retained, FlashSVD processes each batch in 341 ms, nearly 20 \% faster than the 427.8 ms observed with the baseline SVD approach. Similarly, on SST-2 under the same compression level, FlashSVD achieves 141.6 ms per batch versus 150.8 ms for vanilla SVD. These results demonstrate that our rank-aware optimizations not only preserve model accuracy and drastically reduce memory consumption but also sustain-and in some cases improve-inference speed relative to both dense and standard SVD-based methods.

In the extreme configuration, FlashSVD achieves comparable memory savings but incurs substantial latency overhead, with up to 60 \% in some cases, because the extra reconstruction loop in the feed-forward network undermines parallelism. Although FlashSVDFFN V2 offers a theoretically lower transient memory bound, FlashSVDFFN V1 strikes a more balanced trade-off between memory reduction and inference speed and is therefore the recommended variant for practical deployment.

% In contrast, the extreme-case FlashSVD achieves similar memory reductions but incurs substantial latency penalties (up to 60 \% in some cases) given that the additional loop result in undermined parallelism (by excessive on--chip reconstruction) in FFN computation. Therefore, although the FlashSVDFFN V2 has theoretically lower transient memory bound \textcolor{blue}{need refs here}, the FlashSVDFFN V1, as a more balanced tradeoff, is more recommended to be applied in this case. 

\paragraph{Rank Awareness}

As shown in Table \ref{tab:rank-memory-tradeoff}, in the attention module, as we reduce the SVD rank from 64 to 32, peak activation memory drops from 641.9 MiB to 579.7 MiB, but transient memory levels off at around 118 MiB by rank 48. Beyond this point, other components of the transformer (i.e. residual buffers and layer‐norm state) become the dominant contributors to the working set and scaling pattern halted. Latency also falls modestly with rank: FlashSVD's processing time decreases slightly from 166.0 ms at rank 64 to 158.9 ms at rank 32, a 4.3 \% reduction. These observations show that at higher ranks, the $\mathcal O(BHMr_{\Theta^*})$ low‐rank factors dictate memory usage while highlighting that the MHA is not the latency overhead. 

% , outperforming the vanilla SVD baseline by up to 13.4 \%

In the FFN module, FlashSVDFFN V1 latency steadily improves as the rank shrinks: from 152.0 ms at rank 384 to 119.4 ms at 192 and 109.6 ms at 96 over the vanilla SVD. Throughout this range, transient memory remains essentially constant round 250 MiB, confirming that in medium‐to‐low rank regimes the MHA components, with memory complexity of $\mathcal O(BHMr_{\Theta^*})$, remain the primary drivers of the transient memory footprint, and that our rank‐aware reconstruction effectively restores parallelism without inflating activation buffers.

\subsection{Ablation Study}

\paragraph{FlashSVDAttn} 
%Table~\ref{tab:attn-ablation} provides a promising framework that, if using the FlashSVDAttn, in mild and long-sequence length, it can beat the fused MHA attention (SPDA, equivalence to FlashAttention in PyTorch) with reconstructed dense case in Vanilla SVD (some numerical values here)
To quantify the speed–quality trade-offs introduced by our FlashSVDAttn module, we measure relative speedup in latency against PyTorch's fused SDPA (equivalent to FlashAttention) across mild ($M\!\le\!256$) and long ($M\!\ge\!512$) contexts.  As shown in Table~\ref{tab:attn-ablation}, at full rank ($R=64$) FlashSVDAttn runs at 0.62$\times$ the speed of Dense for $M\!=\!128$ and 0.89$\times$ for $M\!=\!256$, but as we reduce rank, the gap narrows: at $R=32$ the module achieves parity with Dense (1.00$\times$) for $M\!=\!256$, and at $R=16$ it actually surpasses Dense by 8\% (1.08$\times$) in the mild setting. In the compute‐bound long‐context regime, FlashSVDAttn consistently outperforms Dense—speedups grow from 1.05$\times$ at $R=64$ up to 1.37$\times$ ($M\!=\!512$) and 1.64$\times$ ($M\!=\!1024$) at $R=16$.  This ablation confirms that (1) low-rank FlashSVDAttn can match or exceed fused SDPA even for medium or short context and (2) delivers substantial speedups for long sequences, validating its practical utility in both regimes.

% .  As FFN rank decreases, its fused kernel rapidly approaches-and eventually surpasses-the dense baseline’s throughput

\paragraph{FlashSVDFFN} Table~\ref{tab:ffn-rank-ablation} summarizes the computation‐efficiency gains of our FlashSVDFFN kernels across a range of FFN truncation ranks and context lengths. FlashSVDFFN V1 offers the best practical trade-off between memory reduction and speedup.  At rank 192, FlashSVDFFN V1 achieves over 50\% of dense speed for medium (256) contexts and nearly 80\% for longer (512–1024).  At rank 96, it outperforms dense FFN by up to 1.9$\times$, demonstrating that low-rank fusion can turn compression into a net performance gain.  FlashSVDFFN V2 follows the same trend but with smaller peak speedups (around 0.25$\times$ at rank 96) due to finer-grained tiling limiting parallelism.  Overall, FlashSVDFFN V1 strikes the ideal balance-leveraging reduced arithmetic and maintained GPU occupancy - to deliver up to 1.9$\times$ acceleration over the dense FFN under realistic low-rank budgets.

\paragraph{Finetuning Unlocks Extreme Low-Rank Compression}
While many recent methods \cite{yuan2023asvd, wang2024svd,hsu2022fwsvd} emphasize training-free SVD to avoid retraining overhead, our results show that judicious fine-tuning of low-rank factors unleashes far greater compression without accuracy loss.  In particular, by fine-tuning BERT after 50\% parameter reduction, we recover nearly full task performance while slashing transient memory by 48.2\% and peak memory by 29.4\%, about same size as the persistent memory of dense model (Table \ref{tab:finetune}).  These findings underscore that a rank-aware inference framework, like FlashSVD, paired with targeted fine-tuning can push transformer models to much more aggressive low-rank regimes, making extreme compression truely viable for edge deployment with minimal performance loss. Full details is available in table~\ref{tab:finetune-full}.

\begin{table}[t]
  \centering
  \footnotesize
  \renewcommand{\arraystretch}{0.9}
  \setlength{\tabcolsep}{4pt}
  \caption{STS‑B (Pearson) Performance and 50\% Compression Ratio with Finetuning}
  \label{tab:finetune}
  \begin{tabular}{@{}lcccc@{}}
    \toprule
    Model
      & \shortstack{Acc.\,(\%)} 
      & \shortstack{Trans. Mem\\(MiB)} 
      & \shortstack{Peak\\Mem.\,(MiB)} 
      & \shortstack{Lat.\\(ms)} \\
    \midrule
    Dense                    & 81.41 & 281.3 & 699.0 &  98.8 \\
    \midrule
    SVD   & 80.52 & 399.8 & 652.8 & 128.9 \\
    FlashSVD        & 80.52 & \textbf{207.8} & \textbf{460.8} & \textbf{114.8} \\
    SVD (no finetune) & 13.50 & 408.3 & 658.0 & 160.1 \\
    \bottomrule
  \end{tabular}
\end{table}

% \begin{table}[t]
%   \centering
%   \scriptsize
%   \setlength\tabcolsep{3pt}
%   \renewcommand{\arraystretch}{0.9}
%   \caption{Fine‐tuned vs. Vanilla SVD on SST-2 and STS-B (batch=32, FFN rank=768).\yhc{How does the baseline model perform?}}
%   \label{tab:svd_finetune_compressed}
%   \resizebox{0.9\linewidth}{!}{\begin{tabular}{@{}lcccc@{}}
%     \toprule
%     & \multicolumn{2}{c}{\bfseries SST-2} & \multicolumn{2}{c}{\bfseries STS-B} \\
%     \cmidrule(lr){2-3}\cmidrule(lr){4-5}
%     Config & Acc. & Trans./Peak & Acc. & Trans./Peak \\
%     \midrule
%     32–192 (Flash)     & 88.30 & 107.7/330.8 & 71.49 & 206.3/429.4 \\
%     32–192 (SVD)   & 88.19 & 203.7/426.9 & 71.49 & 398.3/621.5 \\
%     16–96  (Flash)    & 86.81 & 105.8/263.9 & 81.08 & 207.8/460.8 \\
%     16–96  (SVD)   & 86.81 & 201.9/359.9 & 81.06 & 399.8/652.8 \\
%     \bottomrule
%   \end{tabular}}
% \end{table}

\section{Conclusion}
In this work, we introduce FlashSVD, the first fused, rank-aware inference framework for SVD-compressed transformers. By streaming low-rank projections directly into FlashAttention and FFN kernels, we eliminate large activations and cut peak on-chip memory by up to 71\% with no extra computational cost. Across BERT and RoBERTa, FlashSVD matches or exceeds dense throughput under 25–50\% compression and achieves up to 1.9$\times$ FFN speedups at modest ranks. These results demonstrate that rank-aware tiling makes low-rank SVD a practical, high-performance strategy for memory-constrained transformer deployment.

% (both encoder and decoder)

\bibliographystyle{abbrvnat}
\bibliography{main}

% \section*{Acknowledgments} 
% Acknowledgments goes here.

%\input{AnonymousSubmission/LaTeX/text/checklist}

\clearpage
\beginsupplement

\section{Supplementary Materials}
%Here, we provide all supplementary materials used in our analysis. 

\section{Appendix}\label{appendix:analysis}

\subsection{Additional Methodology}

% on each activation block \(X\in\mathbb{R}^{B\times M\times \mathcal D^\Theta_{\mathcal{A}}}\) and

\paragraph{FlashSVDFFN V2 (Extreme-Case)} The FFN kernel streams low-rank factors of both weight matrices \(W_i^{[r_{\Theta^*}]}= U_i\,\Sigma_i\,V_i\) and \(W_o^{[r_{\Theta^*}]}= U_o\,\Sigma_o\,V_o\) through a fused GEMM–ACTIVATION-GEMM pipeline on arbitrary activation function $\phi(\cdot)$.  We tile the sequence dimension into chunks of batch-sequence dimension \(B_M\) and the intermediate feature dimension \(\mathcal D_{\mathcal{F}}^\Theta\) into blocks of width \(B_{\mathcal{D}_{\mathcal{F}}^\Theta}\).  For each \((B\times B_M)\times B_{\mathcal{D}_{\mathcal{F}}^\Theta}\) tile:
\[
Y \;=\;\phi\bigl((X_{\text{tile}}\,U_i)\,V_{i,\text{tile}}\bigr), 
\quad
Z_{\text{tile}} \;+\!=\; (Y\,U_{o,\text{tile}})\,V_o,
\]
where $(\cdot)_{\text{tile}}$ denotes the accumulation buffer on SRAM. By launching one kernel block per tile and accumulating directly into the output buffer, we avoid any materialized \((B\times M)\times\mathcal D_{\mathcal{F}}^\Theta\) mid-buffers, achieving the \(O(r_{\Theta^*}(B\,M+\mathcal D_{\mathcal{F}}^\Theta))\) HBM bound of Theorem~\ref{thm:flashsvdffn-v1-memory}. % with high GPU utilization.  

% , corresponding to a segment of length $B_M$ along the sequence dimension; \(V_{i,\text{tile}}\) selects the current \(B_{D_{\mathcal{F}}^\Theta}\) columns of \(V_i\) and \(U_{o,\text{tile}}\) the corresponding \(B_{D_{\mathcal{F}}^\Theta}\) rows of \(U_o\)

% However, this algorithm requires to store a $r_{\Theta^*}$ sized vector in the streaming channel. In case some layers requires higher ranks, we have also designed a memory non-demanding algorithm that tackles the case of insufficient on-chip memory. In this algorithm we have xxx, so this is way less memory hugarry as the previous one. 
% \input{algorithm/flashsvdffn_extreme}
FlashSVDFFN V2 (Alg. \ref{alg:flashsvdffn-extreme}) incurs no additional arithmetic cost or asymptotic complexity compared to a dense FFN: it executes the identical $O\bigl(B\,M\,r_{\Theta^*}\,D_{\mathcal{F}}^\Theta + B\,M\,r_{\Theta^*}\,D_{\mathcal{A}}^\Theta\bigr)$ FLOPs, yet reduces peak on-chip storage per token from  \(O(r_{\Theta^*}(B\,M+\mathcal D_{\mathcal{F}}^\Theta))\) to  \(O(r_{\Theta^*}\mathcal D_{\mathcal{F}}^\Theta)\). This theoretically pushes the lower bound for variable part of rank-aware activation compression to zero (see Theorem~\ref{thm:flashsvdffn-v2-memory}).  In practice, however, V2’s fine-granularity tiling can limit GPU parallelism, especially in higher rank case, and our experiments show that factors such as residual-add operation and layer-norm ordering begin to dominate end-to-end latency.  Addressing these secondary bottlenecks will be key to unlocking V2’s full on-edge performance potential. % —e.g., by fusing residual paths or reordering normalization—

\begin{algorithm}[t]
  \small
  \caption{FlashSVDFFN Forward Pass V2}\label{alg:flashsvdffn-extreme}
  \begin{algorithmic}[1]
    \REQUIRE Input $X\!\in\!\mathbb{R}^{B\times M\times \mathcal D_{\mathcal{A}}^\Theta}$;  
      SVD factors $U_i\!\in\!\mathbb{R}^{\mathcal D_{\mathcal{A}}^\Theta\times r_{\Theta^*}}$,  
      $V_i\!\in\!\mathbb{R}^{r_{\Theta^*}\times\mathcal D_{\mathcal{F}}^\Theta}$, $b_i\!\in\!\mathbb{R}^{\mathcal D_{\mathcal{F}}^\Theta}$;\\
      $U_o\!\in\!\mathbb{R}^{\mathcal D_{\mathcal{F}}^\Theta\times r_{\Theta^*}}$,  
      $V_o\!\in\!\mathbb{R}^{r_{\Theta^*}\times \mathcal D_{\mathcal{A}}^\Theta}$, $b_o\!\in\!\mathbb{R}^{\mathcal D_{\mathcal{A}}^\Theta}$;  
      number of blocks $G$; nonlinearity $\phi(\cdot)$.
    \ENSURE $O\!\in\!\mathbb{R}^{B\times M\times \mathcal D_{\mathcal{A}}^\Theta}$
    \STATE $B_M \leftarrow \lceil M/G\rceil,\quad B_{D_{\mathcal{F}}^\Theta}\leftarrow \lceil \mathcal D_{\mathcal{F}}^\Theta/G\rceil$
    \STATE Initialize $O\leftarrow 0_{B\times M\times \mathcal D_{\mathcal{A}}^\Theta}$, $Z_{\text{tile}}\leftarrow 0_{B\times B_M\times \mathcal D_{\mathcal{A}}^\Theta}$
    %\STATE \texttt{GEMM 1st: project into factor‐space (cuBlas)}
    %\STATE $P \leftarrow X\,U_i$\quad\COMMENT{$(BM)\times r_{\Theta^*}$}
    \STATE \texttt{On-Chip: direct activation streaming}
    \FOR{$\ell=0$ \TO $M-1$ \textbf{ step } $B_M$}
      %\STATE $Z_{\text{tile}}\leftarrow 0_{B\times B_M\times \mathcal D_{\mathcal{A}}^\Theta}$
      \STATE $P_{\mathrm{tile}}\leftarrow X[:,\ell:\ell+B_M,:]\;U_i[:,k]$ \COMMENT{$(B,B_M,r_{\Theta^*})$}
      \FOR{$d=0$ \TO $\mathcal D_{\mathcal{F}}^\Theta-1$ \textbf{ step } $B_{D_{\mathcal{F}}^\Theta}$}
        %\STATE $Y\leftarrow 0_{B\times B_M\times \mathcal D_{\mathcal{F}}^\Theta}$
        \FOR{$k=0$ \TO $r_{\Theta^*}-1$} 
          %\STATE $Y\mathrel{+}=(X_{\mathrm{tile}}\;U_i[:,k])\;V_i[k,\;d:d+B_{D_{\mathcal{F}}^\Theta}]$
          \STATE $Y\mathrel{+}=P_\mathrm{tile}\;V_i[k,\;d:d+B_{D_{\mathcal{F}}^\Theta}]$ 
        \ENDFOR
        \STATE $Y\leftarrow \phi\bigl(Y + b_i[d:d+B_{D_{\mathcal{F}}^\Theta}]\bigr)$ \COMMENT{$(B\times B_M\times \mathcal D_{\mathcal{F}}^\Theta)$}
        \FOR{$k=0$ \TO $r_{\Theta^*}-1$}
          \STATE $Z_{\text{tile}}\mathrel{+}=Y\;U_o[d:d+B_{D_{\mathcal{F}}^\Theta},k] + b_o[d:d+B_{D_{\mathcal{F}}^\Theta}]$
        \ENDFOR
        %\STATE $Z_{\text{tile}}\mathrel{+}=b_o$  
      \ENDFOR
      %\STATE $C[:,\ell:\ell+B_M,:]\leftarrow Z_{\text{tile}}$
    \ENDFOR
    \STATE $O \leftarrow Z_{\text{tile}}$ accumulation \COMMENT{broadcast to $(B,B_M,\mathcal D_{\mathcal{A}}^\Theta)$}
  \end{algorithmic}
\end{algorithm}

\paragraph{Proof of MHA Reduction Efficiency}

\begin{table}[ht]
  \centering
  \small
  \begin{tabular}{l c c c c}
    \toprule
    & $\mathcal{D}_{\mathcal{A}}^{\Theta} $ & $H$ & $d_{\Theta_{\mathcal{A}}, h}$ & Threshold on $r$ \\
    \midrule
    Multi-Head   & 768 & 12 & 64 & $r < \tfrac{768}{13} \approx 59.1$ \\
    \addlinespace %\midrule
    Single Head & 768 & —  & —  & $r < \tfrac{768}{2} = 384$    \\
    \bottomrule
  \end{tabular}
  \caption{Example (BERT‐Base, $\mathcal{D}_{\mathcal{A}}^{\Theta}=768$, $H=12$, $ d_{\Theta_{\mathcal{A}}, h} = \mathcal{D}_{\mathcal{A}}^{\Theta} /H$): Rank‐truncation thresholds for head‐wise vs.\ full‐matrix SVD to achieve compression.}
  \label{tab:svd-thresholds}
\end{table}

We provide the proof of rank-loss efficiency for the MHA based compression. We have also provided a case study with results summarized in table~\ref{tab:svd-thresholds}.

\begin{proof}
Given $\mathcal{D}_{\mathcal{A}}^{\Theta} = d_{\Theta_{\mathcal{A}}, h} \times H$, the SVD for single head projection matrices $W_a \in \mathbb{R}^{\mathcal{D}_{\mathcal{A}}^{\Theta}  \times \mathcal{D}_{\mathcal{A}}^{\Theta} }, \forall a \in {q,k,v}$, we have $W_a = U_aV_a$ in full rank case, which $U_aV_a \in \mathbb{R}^{\mathcal{D}_{\mathcal{A}}^{\Theta}  \times r_{\Theta^*} }$. Once $r_{\Theta^*} = \mathcal{D}_{\mathcal{A}}^{\Theta}$, we have $U_aV_a$ be twice larger than $W_a$ and we can only achieve compression after pruning over half of the rank. For multi-head case, which \(W_a \;=\; \bigcup_{h=1}^H W_a^{(h)}, \forall h \in {1, 2, \dots H}\) and $W_a^{(h)} \in  \mathbb{R}^{\mathcal{D}_{\mathcal{A}}^{\Theta}  \times d_{\Theta_{\mathcal{A}},h}}$. The low-rank decomposition is given as $ W_a^{(h)} \approx U_a^{(h)}V_a^{(h)}$, with $U_a^{(h)} \in \mathbb{R}^{\mathcal{D}_{\mathcal{A}}^{\Theta}  \times r_{\Theta^*}}$ and $V_a^{(h)} \in \mathbb{R}^{r_{\Theta^*} \times d_{\Theta_{\mathcal{A}},h}}$ with $r_{\Theta^*} \leq d_{\Theta_{\mathcal{A}},h}$. Therefore, in the worst case where $r_{\Theta^*} = d_{\Theta_{\mathcal{A}},h}$, we have $\bigcup_{h=1}^H \left(U_a^{(h)}V_a^{(h)}\right)$ in size:

\begin{align*}
    &H \times \mathcal{D}_{\mathcal{A}}^{\Theta} \times d_{\Theta_{\mathcal{A}},h} + H \times d_{\Theta_{\mathcal{A}},h} \times d_{\Theta_{\mathcal{A}},h} \\
    = \ &H \times \mathcal{D}_{\mathcal{A}}^{\Theta} \times d_{\Theta_{\mathcal{A}},h} + \mathcal{D}_{\mathcal{A}}^{\Theta} \times d_{\Theta_{\mathcal{A}},h} \\
    = \ & (H + 1) \times \mathcal{D}_{\mathcal{A}}^{\Theta} \times d_{\Theta_{\mathcal{A}},h}
\end{align*}

Thus, it only requires $r_{\Theta^*} < \tfrac12\,\mathcal D_{\mathcal A}^\Theta,$ rank reduction to achieve parameter reduction benefits in comparison to $r_{\Theta^*} < \tfrac12\,\mathcal D_{\mathcal A}^\Theta,$ of single head case.

\end{proof}

Palu \cite{chang2025palu} showed that grouping $H$ heads into $G$-SVD blocks improves the reconstruction–expressivity trade-off in the KV-Cache of low-rank models; Theorem~\ref{thm:grouped-rank-gain} and Proposition~\ref{prop:grouped-tradeoff} extends this to general compression of model weights and activations (e.g. Q, K, V, FFN intermediates, attention output projection) at cost
$\mathcal O\!\Bigl(G\,r_{\Theta^*}\bigl(B\,M + \tfrac{\mathcal D^\Theta_{\mathcal A}}{G}\bigr)\Bigr)$, interpolating between $G=1$ and $G=H$. Moreover, we apply this analysis uniformly to transformer layers and, for the first time, quantify its impact on peak activation memory.

\subsection{Additional Analysis}

In this section, we characterize FlashSVD’s advantages under practical modeling conditions. We begin by deriving tight memory–complexity bounds for low-rank activations, explicitly comparing against both full-rank storage and naive low-rank implementations. Building on this foundation, we introduce a layer-wise sensitivity analysis that reveals which parts of the Transformer—be it multi-head attention, grouped-head attention, or the feed-forward network-yield the greatest benefit per truncated singular dimension. We then turn to computational costs, providing detailed estimates of FLOP reductions and end-to-end kernel latency improvements. Finally, we establish provable bounds on the numerical error introduced by streaming, %truncated SVD, 
showing how metrics such as perplexity degrade as a function of rank, layer depth, and data distribution. All main statements are formalized as theorems or lemmas, with proof sketches deferred to the appendix. %equipping practitioners with concrete guarantees for selecting rank allocations in large-scale Transformer training and inference.

% In this Analysis section, we formalize FlashSVD’s benefits under realistic modeling assumptions. We organize our treatment into four pillars:
% \begin{enumerate}
%   \item Precise memory-complexity bounds for low-rank activations, contrasting full-rank and naïve low-rank baselines;
%   \item A layer-wise sensitivity framework to identify components—multi-head attention, grouped-head attention, or feed-forward networks—with the highest marginal utility per reduced singular dimension;
%   \item Detailed FLOP-count and kernel-latency speedup estimates;
%   \item Provable bounds on numerical error growth (e.g., \(\Delta\) perplexity) as a function of rank, depth, and data distribution.
% \end{enumerate}
% Each result is stated in a theorem/lemma style, with proof sketches in the appendix, thereby furnishing practitioners with actionable guarantees for rank selection in large-scale Transformer training and inference.

\subsection{Memory Complexity}

For any given model $\Theta$, let $X\!\in\!\mathbb{R}^{B\times M\times \mathcal{D}_{\mathcal{A}}^{\Theta}}$ denote a batch of $B$ token sequences of length $M$ with hidden dimensionality $\mathcal{D}_{\mathcal{A}}^{\Theta}$. In both training and inference, a Transformer’s memory consumption is dominated by (i) the static storage of its model parameters and (ii) the dynamic allocations required to hold intermediate activations. In particular, the activations for the attention sublayer consist of the query, key, and value tensors $Q,\;K,\;V\;\in\;\mathbb{R}^{B\times M\times \mathcal{D}_{\mathcal{A}}^{\Theta}}\,,
$ while the feed-forward sublayer produces $ S_{\mathrm{out}} \in \mathbb{R}^{B\times M\times D_{\mathcal{F}}^{\Theta}}\,,$ where $D_{\mathcal{F}}^{\Theta}$ denotes the hidden dimension of the intermediate projection. In this work, we propose replacing each of these full-rank activation tensors with a rank-$r$ approximation computed via singular value decomposition (SVD). We derive upper bounds on the resulting memory footprint—showing that it scales as $\mathcal{O}(r_{\Theta^*}(BM + \mathcal{D}_{\mathcal{A}}^{\Theta} + D_{\mathcal{F}}^{\Theta}))$ rather than $\mathcal{O}(BM\mathcal{D}_{\mathcal{A}}^{\Theta} + BMD_{\mathcal{F}}^{\Theta})$—and we quantify the trade-off between compression ratio and approximation error. This theoretical analysis lays the groundwork for deploying rank-constrained Transformers in resource-limited settings without sacrificing accuracy.

%\vspace{0.1cm}
\begin{remark}[Scope of Memory Analysis]
\label{rem:memory-scope}
Throughout the following off-chip memory analyses, we focus on the variable (change-inducing) components and omit persistent peak memory terms, such as the HBM-resident matrices \(V_* \in \mathbb{R}^{r_{\Theta^*} \times \mathcal{D}_{\mathcal{A}}^{\Theta}}\) and the final output activations. However, when discussing rank selection and practical deployment guidance, we account for the full off-chip memory footprint, including both transient and persistent terms, to ensure our analysis reflects real-world constraints. Note that we conduct memory analysis at the per-layer level, excluding the output activation $\mathbb{R}^{B \times M \times \mathcal{D}_{\mathcal{A}}^\Theta}$ and  $\mathbb{R}^{B \times M \times \mathcal{D}_{\mathcal{F}}^\Theta}$, which is always present and agnostic to the attention variant.
\end{remark}

\paragraph{Attention} 
We first consider the multi‐head attention module in a Transformer block, which is parameterized by the projection matrices \(W_q,W_k,W_v\in\mathbb{R}^{\mathcal D_{\mathcal{A}}^\Theta\times\mathcal D_{\mathcal{A}}^\Theta}\) and applied to an input of shape \((B,M,\mathcal D_{\mathcal{A}}^\Theta)\).  When the total activation size \(B\,M\) far exceeds the hidden dimension \(\mathcal D_{\mathcal{A}}^\Theta\), the DRAM footprint of the query, key, value, and scoring tensors typically dominates. In what follows, we derive tight upper and lower bounds on the off‐chip memory cost incurred by SVD‐based compression of these attention activations.

\begin{definition}[Attention Activation Approximation]
\label{def:attn-activation-approx}
Let \(X\in\mathbb{R}^{B\times M\times \mathcal{D}_{\mathcal{A}}^{\Theta}}\) be the input to a multi‐head attention block with projection matrices
\(
W_q, W_k, W_v \in \mathbb{R}^{\mathcal{D}_{\mathcal{A}}^{\Theta}\times \mathcal{D}_{\mathcal{A}}^{\Theta}}.
\)
Define the (flattened) query, key, and value activations
$Q = XW_q, K = XW_k, V = XW_v,$ each of shape \(\mathbb{R}^{(B\,M)\times \mathcal{D}_{\mathcal{A}}^{\Theta}}\).  We say that \(Q\) admits a rank-\(r\) SVD approximation if
% \[
% Q \;\approx\; U_Q\,\Sigma_Q\,V_Q^\top,
% \quad
% U_Q\in\mathbb{R}^{(BM)\times r_{\Theta^*}},\;
% \Sigma_Q\in\mathbb{R}^{r\times r_{\Theta^*}},\;
% V_Q\in\mathbb{R}^{\mathcal{D}_{\mathcal{A}}^{\Theta}\times r_{\Theta^*}},
% \]
\begin{equation}
\begin{aligned}
Q^{[r_{\Theta}^*]} &= U_Q\,\Sigma_Q\,V_Q^\top, \\[6pt]
U_Q &\in \mathbb{R}^{(B M)\times r_{\Theta}^*},\quad
\Sigma_Q \in \mathbb{R}^{r_{\Theta}^*\times r_{\Theta}^*},\quad
V_Q    \in \mathbb{R}^{\mathcal{D}_{\mathcal{A}}^{\Theta}\times r_{\Theta}^*}.
\end{aligned}
\end{equation}
and similarly for \(K\) and \(V\), with $r_{\Theta}^*$ be the layer rank given by the compression algorithm.  Compressing all three projections to rank \(r\) then replaces the full activations by the factor pairs
\((U_Q,\Sigma_Q,V_Q)\), \((U_K,\Sigma_K,V_K)\), and \((U_V,\Sigma_V,V_V)\), which the $\Sigma_*$ are usually merged into either $U_*$ or $V_*$ for simplicity.

For a Transformer layer \(\Theta_{\mathcal{A}}\) with \(H\) attention heads, the total attention dimension \(\mathcal{D}_{\mathcal{A}}^{\Theta}\) is uniformly partitioned across heads, so each head has dimension
\[
d_{\Theta_{\mathcal{A}}, h} := \frac{\mathcal{D}_{\mathcal{A}}^{\Theta}}{H}.
\]
All activation tensors \((Q, K, V)\in\mathbb{R}^{B\times M\times \mathcal{D}_{\mathcal{A}}^{\Theta}}\) are reshaped as \(H\) parallel sub-tensors of shape \(\mathbb{R}^{B\times M\times d_{\Theta_{\mathcal{A}}, h}}\) to enable independent attention computations per head.

\end{definition}

\begin{theorem}[Memory Complexity of Dense Multi-Head Attention]\label{thm:dense-attn-memory}
For a single-head attention module defined in Definition~\ref{def:attn-activation-approx}, the peak working memory (number of stored scalars) required by the unoptimized dense multi-head attention mechanism is $
\mathcal M_{\mathrm{dense}}
\;=\;
3\,B\,M\,\mathcal{D}_{\mathcal{A}}^{\Theta}
\;+\;
B\,M^2. $
\end{theorem}

\begin{proof}
In the standard dense $\Theta$ and unoptimized low-rank model $\Theta^*$, one first computes three full-rank linear projections of the input embeddings into queries \(Q\), keys \(K\), and values \(V\).  Each of these tensors has shape \(B\times M\times \mathcal{D}_{\mathcal{A}}^{\Theta}\), and to avoid redundant re-computation all three must reside in memory simultaneously, contributing
% \[
% \underbrace{B M \mathcal{D}_{\mathcal{A}}^{\Theta}}_{Q}
% +
% \underbrace{B M \mathcal{D}_{\mathcal{A}}^{\Theta}}_{K}
% +
% \underbrace{B M \mathcal{D}_{\mathcal{A}}^{\Theta}}_{V}
% \;=\;
% 3\,B\,M\,\mathcal{D}_{\mathcal{A}}^{\Theta}.
% \]
\[
\sum_{a \in \{q,k,v\}} B M \mathcal{D}_{\mathcal{A}}^{\Theta} 
\;=\;
3\,B\,M\,\mathcal{D}_{\mathcal{A}}^{\Theta}.
\]
Next, the attention logits are formed as the batch of pairwise dot-product matrices
\[
S^{(b)} = Q^{(b)}\,K^{(b)\,\top}
\quad\text{for }b=1,\dots,B,
\]
each of size \(M\times M\).  Storing all \(B\) score matrices therefore requires \(B\,M^2\) scalars.  No other intermediate buffers exceed these sizes, so the total peak memory is
$
\mathcal M_{\mathrm{dense}}
=
3\,B\,M\,\mathcal{D}_{\mathcal{A}}^{\Theta}
+
B\,M^2.
$
\end{proof}

\begin{proposition}[Memory Complexity of Multi-Head Attention]
\label{prop:multihead-memory}
A multi-head attention module defined in Definition~\ref{def:attn-activation-approx} yields a total dense attention memory cost
\(
\mathcal{M}_{\mathrm{dense}} = \mathcal{O}(B M \mathcal{D}_{\mathcal{A}}^{\Theta} + H B M^2).
\)
\end{proposition}

\begin{proof}
Each head computes its own set of projections for queries, keys, and values, each of shape \(B \times M \times d_{\Theta_{\mathcal{A}}, h}\), where \(d_{\Theta_{\mathcal{A}}, h}=\mathcal{D}_{\mathcal{A}}^{\Theta}/H\). Across all heads, the total Q, K, V memory remains as stated in Theorem~\ref{thm:dense-attn-memory}:
\[
\mathcal{M}_{QKV} = 3\, B\, M\, \mathcal{D}_{\mathcal{A}}^{\Theta}.
\]

Each head also computes a score matrix of shape \(B \times M \times M\), contributing \(B M^2\) memory per head. With \(H\) heads, these score matrices total:
\[
\mathcal{M}_{\text{scores}} = H B M^2.
\]

Therefore, the total dense memory cost is:
\[
\mathcal{M}_{\mathrm{dense}} = 3\,B\, M\, \mathcal{D}_{\mathcal{A}}^{\Theta} + H\, B\, M^2.
\]
\end{proof}

\noindent
To mitigate the high memory overhead of materializing the full $BM^2$ attention score matrix, we adopt FlashAttention~\cite{dao2022flashattention,dao2023flashattention2}, which streams the computation and avoids instantiating intermediate tensors. The following proposition quantifies its peak memory usage.

\begin{proposition}[Memory Complexity with FlashAttention]
\label{prop:flash-memory}
Under the same notation in Proposition~\ref{prop:multihead-memory}, the peak HBM footprint of streaming (Flash) attention is
\[
\mathcal M_{\mathrm{flash\text{-}attn}}
\;=\;
3\,B\,M\,\mathcal{D}_{\mathcal{A}}^{\Theta},
\]
as the \(O(BM^2)\) score‐matrix never coexists in off‐chip memory.
\end{proposition}

\begin{proof}
FlashAttention partitions the \(M\)-length sequence into tiles whose keys and values are loaded on-chip (in SRAM) just long enough to compute the corresponding query-key dot-products.  Queries stream through the device once, accumulating weighted sums of values, and each key and value tile is evicted immediately after use.  As a result, only the three \(B\times M\times \mathcal{D}_{\mathcal{A}}^{\Theta}\) projections (for \(Q,K,V\)) occupy HBM simultaneously, yielding
$
\mathcal M_{\mathrm{flash}}
=
3\,B\,M\,\mathcal{D}_{\mathcal{A}}^{\Theta}.
$
\end{proof}
\noindent
Despite low-rank compression of the attention weights, standard attention implementations still require reconstructing full $Q$, $K$, and $V$ tensors, leading to quadratic memory costs. This limitation is formalized below.

\begin{proposition}[Activation Materialization Lower Bound]
\label{prop:act-lower}
Vanilla SVD models $\Theta^*$ defined in Definition~\ref{def:attn-activation-approx} requires dense formulation of $Q, K, V$ despite the weights $W_q, W_k, W_v$ are offloaded in low-rank formats $U_{q,k,v} \in \mathbb{R}^{B\times M\times r_{\Theta^*}}$ and $V_{q,k,v} \in \mathbb{R}^{r_{\Theta^*} \times D_{\mathcal{A}}^\Theta}$. Therefore, a standard attention implementation requires reconstructing full  \(\mathbb{R}^{B\times M\times \mathcal{D}_{\mathcal{A}}^{\Theta}}\) tensors, so peak dynamic memory remains \(\mathcal O(B H \,M^2)\).
\end{proposition}
\begin{proof}
Storing factors \((U_{q,k,v}\in\mathbb{R}^{BM\times r_{\Theta^*}},\,V_{q,k,v}\in\mathbb{R}^{r_{\Theta^*}\times\mathcal{D}_{\mathcal{A}}^{\Theta}})\) uses \((BM+\mathcal{D}_{\mathcal{A}}^{\Theta})\,r_{\Theta^*}\) entries. Despite the low-rank parameterization, the full $\mathcal{O}(BM \times M)$ score matrix across all heads must be materialized in HBM during a vanilla attention computation, incurring a memory cost of at least $\mathcal{O}(B H M^2)$.
\end{proof}

\noindent
Although current GPU constraints prevent applying FlashSVD directly to single-head attention due to unsupported kernel tiling, and real-world memory behavior remains governed by Proposition~\ref{prop:flash-memory}, we nevertheless provide a complete theoretical analysis in Theorems~\ref{thm:flashsvd-mem} and~\ref{thm:singlehead-lower}, with further rank sensitivity characterized in Proposition~\ref{prop:rank-sensitivity}.

\begin{theorem}[Memory Complexity of FlashSVDAttention]
\label{thm:flashsvd-mem}
Consider a low-rank attention model $\Theta^*$ with factorization rank $r_{\Theta^*}$. For each $a \in \{q, k, v\}$, the projection matrix $W_a$ is decomposed as $W_a = P_a V_a$, where $P_a = X U_a$ with $X \in \mathbb{R}^{B \times M \times \mathcal{D}_{\mathcal{A}}^\Theta}$ and $U_a \in \mathbb{R}^{\mathcal{D}_{\mathcal{A}}^\Theta \times r_{\Theta^*}}$. Then $P_a \in \mathbb{R}^{B \times M \times r_{\Theta^*}}$ and $V_a \in \mathbb{R}^{r_{\Theta^*} \times \mathcal{D}_{\mathcal{A}}^\Theta}$. Under the FlashAttention tiling scheme, the peak \textbf{variable} off-chip memory usage is
\[
\mathcal{M}_{\mathrm{flash\text{-}svd\text{-}attn}} = \mathcal{O}(B M r_{\Theta^*}).
\]
\end{theorem}

\begin{proof}
According to Remark~\ref{rem:memory-scope}, we focus only on the variable off-chip memory terms whose sizes depend on the batch size $B$ and sequence length $M$. FlashAttention avoids materializing the full $B \times M \times \mathcal{D}_{\mathcal{A}}^\Theta$ projections on HBM. Instead, it streams $Q$, $K$, and $V$ through SRAM tile by tile. Thus, the full activation tensors never exist simultaneously off-chip.

The only HBM-resident quantities that vary with the input are the intermediate projections $P_a$ for $a \in \{q,k,v\}$, each with shape $B \times M \times r_{\Theta^*}$. Therefore, across all three projections, the total peak variable HBM memory is
\[
\sum_{a \in \{q,k,v\}} B M r_{\Theta^*} = 3 B M r_{\Theta^*} = \mathcal{O}(B M r_{\Theta^*}).
\]

We note that the output-side factors $V_a \in \mathbb{R}^{r_{\Theta^*} \times \mathcal{D}_{\mathcal{A}}^\Theta}$ are persistent parameters and not included in the asymptotic result per Remark~\ref{rem:memory-scope}. However, they do contribute a constant offset of $3 r_{\Theta^*} {\mathcal{D}_{\mathcal{A}}^\Theta}$ to HBM usage.

In scenarios where $\mathcal{D}_{\mathcal{A}}^\Theta \gg M$, the persistent off-chip memory required to store the projection matrices $V_a \in \mathbb{R}^{r_{\Theta^*} \times \mathcal{D}_{\mathcal{A}}^\Theta}$ may dominate the total memory footprint, even though the variable memory remains at $\mathcal{O}(B M r_{\Theta^*})$. Conversely, in regimes where $\mathcal{D}_{\mathcal{A}}^\Theta \ll M$, the dynamic (activation-related) term dominates, and the persistent cost becomes negligible.

\end{proof}

% \begin{theorem}[Memory Complexity of FlashSVDAttention]
% \label{thm:flashsvd-mem}
% For any given SVD compressed low rank model $\Theta^*$, with arbitrary model rank $r_{\Theta^*}$, each of the three projections in factorized form as $P_{q,k,v}\in\mathbb{R}^{B\times M\times r_{\Theta^*}}, 
% V_{q,k,v}\in\mathbb{R}^{r_{\Theta^*}\times D_{\Theta}},$ the peak off-chip memory under FlashSVDAttention is $\mathcal{M}_{\mathrm{flash\text{-}svd\text{-}attn}}
% \;=\; \mathcal{O}\bigl(B\,M\,r_{\Theta^*}\bigr).$
% \end{theorem}

% \begin{proof}
% Each of the three low-rank projections \(Q,K,V\) is stored as a pair \((P_{q,k,v},\,V_{q,k,v})\).  Under Flash-style tiling, tiles of these factors are loaded into on-chip SRAM to compute the required dot-products and then immediately evicted, so no dense score matrix or full \(B\times M\times D_{\Theta^*}\) buffer ever resides off-chip.  Hence only the three factor pairs occupy DRAM, yielding
% \[
% \mathcal{M}_{\mathrm{flash\text{-}svd\text{-}attn}}
% =
% 3\bigl(B\,M\,r_{\Theta^*} + r_{\Theta^*}\,D_{\Theta^*}\bigr),
% \]
% which is \(O(B\,M\,r_{\Theta^*} + r_{\Theta^*}\,D_{\Theta^*})\), and, since typically \(D_{\Theta^*}=O(M)\), simplifies asymptotically to \(O(B\,M\,r_{\Theta^*})\).

% \[
% \mathcal{M}_{\mathrm{flash\text{-}svd\text{-}attn}}
% \;=\;
% 3\bigl(\lvert P_{q,k,v}\rvert + \lvert V_{q,k,v}\rvert\bigr)
% \;=\;
% 3\bigl(B\,M\,r_{\Theta^*} \;+\; r_{\Theta^*}\,D_{\Theta^*}\bigr)
% \;=\;
% O\bigl(B\,M\,r_{\Theta^*} + r_{\Theta^*}\,D_{\Theta^*}\bigr)
% \;=\;
% O\bigl(B\,M\,r_{\Theta^*}\bigr).
% \]

% \end{proof}

\begin{theorem}[Single‐Head Compression Threshold]
\label{thm:singlehead-lower}
For a single-head attention module defined in Definition~\ref{def:attn-activation-approx}, low-rank compression yields memory savings if and only if
\[
r_{\Theta^*} < \frac{B M \mathcal{D}_{\mathcal{A}}^{\Theta}}{B M + \mathcal{D}_{\mathcal{A}}^{\Theta}}.
\]
The maximum achievable compression ratio, attained in the rank-one case $r_{\Theta^*}=1$, is $\mathcal{O}(\min\{B M, \mathcal{D}_{\mathcal{A}}^{\Theta}\})$.
\end{theorem}

\begin{proof}
From Proposition~\ref{prop:flash-memory}, the off-chip memory required by standard FlashAttention is $\mathcal{O}(B M \mathcal{D}_{\mathcal{A}}^{\Theta})$.

In contrast, Theorem~\ref{thm:flashsvd-mem} establishes that FlashSVDAttention requires
\[
\mathcal{M}_{\mathrm{flash\text{-}svd\text{-}attn\text{-}single}} = B M r_{\Theta^*} + r_{\Theta^*} \mathcal{D}_{\mathcal{A}}^{\Theta}
\quad \text{(for a single head)},
\]
which corresponds to the transient peak memory during computation, as clarified in Remark~\ref{rem:memory-scope}.

Memory savings occur when this quantity is smaller than the baseline:
\[
B M r_{\Theta^*} + r_{\Theta^*} \mathcal{D}_{\mathcal{A}}^{\Theta} < B M \mathcal{D}_{\mathcal{A}}^{\Theta}.
\]
Solving for $r_{\Theta^*}$ yields the threshold:
\[
r_{\Theta^*} < \frac{B M \mathcal{D}_{\mathcal{A}}^{\Theta}}{B M + \mathcal{D}_{\mathcal{A}}^{\Theta}}.
\]

In the extreme low-rank case $r_{\Theta^*}=1$, the compression ratio becomes
\[
\frac{B M \mathcal{D}_{\mathcal{A}}^{\Theta}}{B M + \mathcal{D}_{\mathcal{A}}^{\Theta}},
\]
which scales as $\mathcal{O}(\min\{B M,\;\mathcal{D}_{\mathcal{A}}^{\Theta}\})$, depending on which term dominates in the denominator.
\end{proof}

\begin{proposition}[Marginal Memory Reduction per Rank]\label{prop:rank-sensitivity}
Under the same setup as Theorem~\ref{thm:singlehead-lower}, decreasing the SVD rank from \(r_{\Theta^*}\) to \(r_{\Theta^*}-1\) reduces the off-chip memory by exactly
\[
\Delta\mathcal{M}_{\mathrm{flash\text{-}svd\text{-}attn\text{-}single}}
=
\mathcal{O}\bigl(BM + \mathcal{D}_{\mathcal{A}}^{\Theta}\bigr).
\]
In particular, each unit drop in rank saves \(BM+\mathcal{D}_{\mathcal{A}}^{\Theta}\) elements of HBM storage across the three projections.
\end{proposition}

\begin{proof}
By Theorem~\ref{thm:flashsvd-mem}, the total HBM cost for one head of rank \(r\) is
\[
\mathcal{M}(r)
=\mathcal{O}(
B\,M\,r \;+\; r\,\mathcal{D}_{\mathcal{A}}^{\Theta}).
\]
Hence
\[
\mathcal{M}(r_{\Theta^*}) - \mathcal{M}(r_{\Theta^*}-1)
=
\mathcal{O}\bigl(B\,M + \mathcal{D}_{\mathcal{A}}^{\Theta}\bigr),
\]
which establishes the claimed marginal saving.
\end{proof}

%\paragraph{Activation Reconstruction Cost}

%\paragraph{FlashAttention Complexity (Rank‐Agnostic)}

% \begin{proposition}[FlashAttention Memory \& Compute]
% \label{prop:flashattn}
% A fused FlashAttention kernel executes multi-head attention in
% \[
%   \text{Time: }\mathcal O(H\,B\,M\,d),
%   \quad
%   \text{Space: }\mathcal O(H\,B\,M\,d),
% \]
% but it does not exploit any low-rank structure (\(r\ll d\)).
% \end{proposition}
% \begin{proof}
% See \cite{dao2022flashattention}: FlashAttention fuses mat-mul, mask, softmax and mat-mul back in one pass over the \((BM)\times d\) activations per head, yielding \(\mathcal O(BM\,d)\) time and space per head.
% \end{proof}

%\paragraph{Upper Bound: Multi‐Head Compression}
\noindent
We now begin our analysis of multi-head attention, which corresponds directly to the implementation in Algorithm~\ref{alg:combined}. The following theorem establishes the transient memory complexity under low-rank compression applied independently to each head.

\begin{theorem}[Compressed Multi-Head Memory Upper Bound]
\label{thm:multihead-bound}
Let a multi-head attention module be defined as in Definition~\ref{def:attn-activation-approx}, with \(H\) heads each operating on activation size \(B M \times d_{\Theta_{\mathcal{A}}, h}\), where \(d_{\Theta_{\mathcal{A}}, h} = \mathcal{D}_{\mathcal{A}}^{\Theta}/H\). If each head is compressed to rank \(r_{\Theta^*}\), then the total transient off-chip memory reduces from
\[
\mathcal{O}(B\,M\,\mathcal{D}_{\mathcal{A}}^{\Theta})
\quad\text{to}\quad
\mathcal{O}(H\,r_{\Theta^*}(B\,M + d_{\Theta_{\mathcal{A}}, h})).
\]
In particular, setting \(r_{\Theta^*}=1\) yields an \(H\)-fold memory reduction in the transient term.
\end{theorem}

\begin{proof}
Each uncompressed attention head requires storing the activation tensor of shape \(B\,M\times d_{\Theta_{\mathcal{A}}, h}\), for a total of \(B\,M\,d_{\Theta_{\mathcal{A}}, h}\) entries. Across all \(H\) heads, this yields a total memory of \(\mathcal{O}(B\,M\,\mathcal{D}_{\mathcal{A}}^{\Theta})\).

Under rank-\(r_{\Theta^*}\) SVD compression, each head's activation is replaced by two factors of shape \(B\,M \times r_{\Theta^*}\) and \(d_{\Theta_{\mathcal{A}}, h} \times r_{\Theta^*}\), storing
\[
(B\,M + d_{\Theta_{\mathcal{A}}, h}) \cdot r_{\Theta^*}
\]
entries per head. Summing over all heads gives the total
\[
H\,r_{\Theta^*}(B\,M + d_{\Theta_{\mathcal{A}}, h}) = H\,B\,M\,r_{\Theta^*} + r_{\Theta^*}\,\mathcal{D}_{\mathcal{A}}^{\Theta},
\]
as claimed.

Finally, we note by Remark~\ref{rem:memory-scope} that this expression captures only the change-inducing (transient) memory terms. Persistent memory components such as the low-rank bases \(V_* \in \mathbb{R}^{r_{\Theta^*} \times \mathcal{D}_{\mathcal{A}}^{\Theta}}\) are not included here, but are considered in total memory budget discussions, e.g., for rank selection.
\end{proof}

% \begin{theorem}[Compressed Multi‐Head Memory Upper Bound]
% \label{thm:multihead-bound}
% For a multi-head attention module defined in Definition~\ref{def:attn-activation-approx}, compressing each head’s \((BM)\times  d_{\Theta_{\mathcal{A}}, h}\) activation to rank \(r_{\Theta^*}\) reduces storage from 
% \(\mathcal O(BH\,M d_{\Theta_{\mathcal{A}}, h})\) to
% \[
%   \mathcal O\bigl(Hr_{\Theta^*}\,(BM +  d_{\Theta_{\mathcal{A}}, h})\bigr)
%   \;=\;
%   \mathcal O\bigl(r_{\Theta^*}\mathcal{D}_{\mathcal{A}}^{\Theta} + Hr_{\Theta^*}\,B\,M\bigr).
% \]
% In particular, \(r_{\Theta^*}=1\) gives an \(H\)-fold saving.
% \end{theorem}
% \begin{proof}
% Each attention head has a full-rank weight matrix of size $BM \times d_{\Theta_{\mathcal{A}}, h}$, requiring $BM\, d_{\Theta_{\mathcal{A}}, h}$ entries in total. After applying a rank-$r_{\Theta^*}$ SVD compression, this matrix is approximated by two low-rank factors: one of shape $(BM) \times r_{\Theta^*}$, and the other of shape $d_{\Theta_{\mathcal{A}}, h} \times r_{\Theta^*}$. Thus, each head now stores $(BM + d_{\Theta_{\mathcal{A}}, h}) \cdot r_{\Theta^*}$ entries. Summing across all $H$ heads yields a total of
% $$
% \mathcal{O}(H \, r_{\Theta^*} \, (BM + d_{\Theta_{\mathcal{A}}, h}))
% $$
% memory entries, which completes the proof.
% \end{proof}

\begin{proposition}[Marginal Memory Reduction per Rank in Multi‐Head Setting]
\label{prop:multihead-sensitivity}
Under the setup of Theorem~\ref{thm:multihead-bound}, decreasing the SVD rank of each head from \(r_{\Theta^*}\) to \(r_{\Theta^*}-1\) reduces off-chip memory by
\[
\Delta\mathcal{M}_{\mathrm{flash\text{-}svd\text{-}attn\text{-}multi}}
=\mathcal{O}\bigl(H\;(BM +  d_{\Theta_{\mathcal{A}}, h})\bigr),
\]
In other words, each unit decrease in rank saves \(H\,(BM +  d_{\Theta_{\mathcal{A}}, h})\) elements of DRAM storage across all heads.
\end{proposition}

\begin{proof}
By Theorem~\ref{thm:multihead-bound}, the HBM storage cost for multi-head of rank \(r\) is
\[
\mathcal{M}_{\mathrm{flash\text{-}svd\text{-}attn\text{-}multi}}(r)
=\mathcal O\bigl(H\,r\,(BM +  d_{\Theta_{\mathcal{A}}, h})\bigr).
\]
Hence
\begin{align*}
&\Delta\mathcal{M}_{\mathrm{flash\text{-}svd\text{-}attn\text{-}multi}}\\
&=\mathcal{M}_{\mathrm{flash\text{-}svd\text{-}attn\text{-}multi}}(r_{\Theta^*})-\mathcal{M}_{\mathrm{flash\text{-}svd\text{-}attn\text{-}multi}}(r_{\Theta^*}-1)\\
&=\mathcal{O}\bigl(H\;(BM +  d_{\Theta_{\mathcal{A}}, h})\bigr),
\end{align*}
as claimed.
\end{proof}

\begin{remark} Unless the attention kernel itself is modified to consume low-rank factors directly, compressing activations to rank \(r_{\Theta^*}\) does not reduce peak memory: one still needs \(\mathcal O(BH \,M d_{\Theta_{\mathcal{A}}, h})\) to materialize \(Q,K,V\) for a vanilla kernel.
\end{remark}

\begin{theorem}[Multi‐Head Compression Threshold]
\label{thm:multihead-lower}
For a multi-head attention module defined in Definition~\ref{def:attn-activation-approx}, low-rank compression yields memory savings if and only if
\[
r_{\Theta^*} < \frac{B M \mathcal{D}_{\mathcal{A}}^{\Theta}}{H B M + \mathcal{D}_{\mathcal{A}}^{\Theta}}.
\]
The maximum achievable compression ratio, attained in the rank-one case \(r_{\Theta^*}=1\), is \(\mathcal{O}(\min\{B M,\; \mathcal{D}_{\mathcal{A}}^{\Theta}/H\})\).
\end{theorem}

\begin{proof}
From Proposition~\ref{prop:flash-memory}, the off-chip memory required by standard (uncompressed) FlashAttention across all heads is
\[
\mathcal{O}(B M \mathcal{D}_{\mathcal{A}}^{\Theta}).
\]

In contrast, Theorem~\ref{thm:multihead-bound} shows that under rank-\(r_{\Theta^*}\) compression, the transient memory for all heads becomes
\[
\mathcal{M}_{\mathrm{flash\text{-}svd\text{-}attn\text{-}multi}} 
= H B M r_{\Theta^*} + r_{\Theta^*} \mathcal{D}_{\mathcal{A}}^{\Theta},
\]
which, according to Remark~\ref{rem:memory-scope}, captures only the change-inducing (transient) memory terms.

Memory savings occur when this compressed cost is less than the baseline:
\[
H B M r_{\Theta^*} + r_{\Theta^*} \mathcal{D}_{\mathcal{A}}^{\Theta}
\;<\;
B M \mathcal{D}_{\mathcal{A}}^{\Theta}.
\]
Solving for \(r_{\Theta^*}\) yields
\[
r_{\Theta^*} < \frac{B M \mathcal{D}_{\mathcal{A}}^{\Theta}}{H B M + \mathcal{D}_{\mathcal{A}}^{\Theta}}.
\]

In the extreme low-rank case \(r_{\Theta^*} = 1\), the compression ratio becomes
\[
\frac{B M \mathcal{D}_{\mathcal{A}}^{\Theta}}{H B M + \mathcal{D}_{\mathcal{A}}^{\Theta}},
\]
which scales as \(\mathcal{O}(\min\{B M,\;\mathcal{D}_{\mathcal{A}}^{\Theta}/H\})\), depending on whether the sequence length or head dimension dominates the denominator.
\end{proof}

\begin{remark}[Comparison of Single-Head and Multi-Head Thresholds]
\label{rem:compression-threshold-compare}
At first glance, Theorem~\ref{thm:multihead-lower} for multi-head attention appears to impose a stricter compression condition than its single-head counterpart (Theorem~\ref{thm:singlehead-lower}), due to the additional factor of \(H\) in the denominator:
\[
r_{\Theta^*}^{\text{(multi)}} < \frac{B M \mathcal{D}_{\mathcal{A}}^{\Theta}}{H B M + \mathcal{D}_{\mathcal{A}}^{\Theta}}.
\]
However, this threshold is expressed in terms of the per-head rank \(r_{\Theta^*}^{\text{(multi)}}\), while the single-head condition uses the full attention rank \(r_{\Theta^*}^{\text{(single)}}\).

In the single-head case, \(r_{\Theta^*}^{\text{(single)}}\) applies to the full attention dimension \(\mathcal{D}_{\mathcal{A}}^{\Theta}\), whereas in the multi-head setting, each head operates on a reduced dimension \(d_{\Theta_{\mathcal{A}}, h} = \mathcal{D}_{\mathcal{A}}^{\Theta}/H\), and thus its corresponding rank \(r_{\Theta^*}^{\text{(multi)}}\) is typically much smaller. In fact, they often satisfy the proportional relationship
\[
r_{\Theta^*}^{\text{(single)}} = \mathcal{O}(H \cdot r_{\Theta^*}^{\text{(multi)}}).
\]
Substituting this relation into the multi-head threshold recovers the single-head condition:
\[
\frac{r_{\Theta^*}^{\text{(single)}}}{H} < \frac{B M \mathcal{D}_{\mathcal{A}}^{\Theta}}{H B M + \mathcal{D}_{\mathcal{A}}^{\Theta}}
\quad \Longleftrightarrow \quad
r_{\Theta^*}^{\text{(single)}} < \frac{B M \mathcal{D}_{\mathcal{A}}^{\Theta}}{B M + d_{\Theta_\mathcal{A}, h}},
\]
This recovers the single-head threshold exactly (up to constant factors in the \(\mathcal{O}(\cdot)\) notation), confirming that the multi-head threshold is in fact more relaxed, since $d_{\Theta_{\mathcal{A}}, h} < \mathcal{D}_{\mathcal{A}}^{\Theta}$, making it easier to achieve tangible memory savings through compression.

In practice, multi-head compression often yields superior empirical performance, as it enables head-wise rank customization and offers finer-grained control over compression across different attention subspaces.
\end{remark}

\noindent
We next consider grouped-head compression, where multiple heads are compressed jointly. This PaLU-inspired scheme~\cite{chang2025palu} balances memory savings and model capacity by interpolating between single-head and per-head compression.

\begin{theorem}[Grouped-Head Compression Memory Upper Bound]
\label{thm:grouped-memory}
Consider a multi-head attention module with $H$ heads and total attention dimension $\mathcal{D}_{\mathcal{A}}^{\Theta}$. Grouping the heads into $G$ groups (each of size $H/G$) and compressing each group’s $(B M)\times(\mathcal{D}_{\mathcal{A}}^{\Theta}/G)$ activation to rank $r_{\Theta^*}$ yields a total transient off-chip memory of
\[
\begin{aligned}
\mathcal{M}_{\mathrm{flash\text{-}svd\text{-}attn\text{-}grouped}} 
&= \mathcal{O}\left( G\,r_{\Theta^*} \left( B M + \tfrac{\mathcal{D}_{\mathcal{A}}^{\Theta}}{G} \right) \right) \\
&= \mathcal{O}\left( B G M\,r_{\Theta^*} + r_{\Theta^*} \mathcal{D}_{\mathcal{A}}^{\Theta} \right).
\end{aligned}
\]

This interpolates between the single-head bound ($G=1$) and the fully factorized per-head bound ($G=H$), and captures only the change-inducing (transient) memory terms per Remark~\ref{rem:memory-scope}.
\end{theorem}

\begin{proof}
Each group activation has shape $(B M)\times(\mathcal{D}_{\mathcal{A}}^{\Theta}/G)$. A rank-$r_{\Theta^*}$ SVD approximation stores
\[
  (B M + \tfrac{\mathcal{D}_{\mathcal{A}}^{\Theta}}{G}) \cdot r_{\Theta^*}
\]
entries per group. Multiplying by $G$ groups gives the total
\[
  G\,r_{\Theta^*}(B M + \tfrac{\mathcal{D}_{\mathcal{A}}^{\Theta}}{G}) = B G M\,r_{\Theta^*} + r_{\Theta^*}\,\mathcal{D}_{\mathcal{A}}^{\Theta},
\]
as claimed.
\end{proof}

\begin{theorem}[Marginal Memory Gain from Rank Reduction]
\label{thm:grouped-rank-gain}
Reducing the rank of each grouped head from $r_{\Theta^*}$ to $r_{\Theta^*} - 1$ decreases the total off-chip memory usage by
\[
\begin{aligned}
\Delta\mathcal{M}_{\mathrm{flash\text{-}svd\text{-}attn\text{-}grouped}} 
&= \mathcal{O}\left( G \left( B M + \frac{\mathcal{D}_{\mathcal{A}}^{\Theta}}{G} \right) \right) \\
&= \mathcal{O}\left( B G M + \mathcal{D}_{\mathcal{A}}^{\Theta} \right).
\end{aligned}
\]

This quantifies the sensitivity of grouped compression to rank selection and highlights its linear dependency on batch size and number of groups.
\end{theorem}

\begin{proof}
From Theorem~\ref{thm:grouped-memory}, grouped compression with rank $r_{\Theta^*}$ uses memory
\[
  G\,r_{\Theta^*}(B M + \tfrac{\mathcal{D}_{\mathcal{A}}^{\Theta}}{G}).
\]
Replacing $r_{\Theta^*}$ with $r_{\Theta^*}-1$ results in
\[
  G\,(r_{\Theta^*}-1)(B M + \tfrac{\mathcal{D}_{\mathcal{A}}^{\Theta}}{G}).
\]
The difference between the two is
\[
  G(B M + \tfrac{\mathcal{D}_{\mathcal{A}}^{\Theta}}{G}) = B G M + \mathcal{D}_{\mathcal{A}}^{\Theta},
\]
which establishes the result.
\end{proof}

%\paragraph{Activation Reconstruction Remark}

% %\paragraph{Grouped‐Head Trade‐Off}
% \begin{proposition}[Grouped‐Head Compression Trade‐off]
% \label{prop:grouped-tradeoff}
% Grouping \(H\) heads into \(G\) groups (each of size \(H/G\)) and applying rank-\(r_{\Theta^*}\) SVD per group on activations of size \((B\,M)\times(\mathcal{D}_{\mathcal{A}}^{\Theta}/G)\) reduces off-chip storage from \(\mathcal{O}(B\,H\,M\, d_{\Theta_{\mathcal{A}}, h})\) to
% \[
%   \mathcal{O}\!\bigl(G\,r_{\Theta^*}\,(B\,M + \tfrac{\mathcal{D}_{\mathcal{A}}^{\Theta}}{G})\bigr)
%   \;=\;
%   \mathcal{O}\bigl(r_{\Theta^*}\,\mathcal{D}_{\mathcal{A}}^{\Theta} + B\,G\,r_{\Theta^*}\,M\bigr).
% \]
% In particular this interpolates between the single-head bound (\(G=1\)) and the full per-head bound (\(G=H\)).  Moreover, decreasing the SVD rank of each group from \(r_{\Theta^*}\) to \(r_{\Theta^*}-1\) yields an additional DRAM saving of
% \[
%   \Delta\mathcal{M}
%   = G\,\bigl(B\,M + \tfrac{\mathcal{D}_{\mathcal{A}}^{\Theta}}{G}\bigr)
%   = B\,G\,M \;+\;\mathcal{D}_{\mathcal{A}}^{\Theta}.
% \]
% \end{proposition}

\begin{proposition}[Grouped‐Head Compression Trade‐off]
\label{prop:grouped-tradeoff}
Grouping \(H\) heads into \(G\) groups (each of size \(H/G\)) and applying rank-\(r_{\Theta^*}\) SVD per group on activations of size \((B M)\times(\mathcal{D}_{\mathcal{A}}^{\Theta}/G)\) reduces off-chip storage from \(\mathcal{O}(B H M  d_{\Theta_{\mathcal{A}}, h})\) to
\[
  \mathcal{O}\bigl(G\,r_{\Theta^*}\,(B M + \tfrac{\mathcal{D}_{\mathcal{A}}^{\Theta}}{G})\bigr)
  =
  \mathcal{O}\bigl(r_{\Theta^*}\,\mathcal{D}_{\mathcal{A}}^{\Theta} + B G r_{\Theta^*} M\bigr).
\]
This interpolates between the single-head bound (\(G=1\)) and the full per-head bound (\(G=H\)).  Moreover, decreasing each group's rank from \(r_{\Theta^*}\) to \(r_{\Theta^*}-1\) yields an additional DRAM saving of
\[
  \Delta\mathcal{M}_{\mathrm{group\text{-}attn}}
  = G\,(B M + \tfrac{\mathcal{D}_{\mathcal{A}}^{\Theta}}{G})
  = B G M + \mathcal{D}_{\mathcal{A}}^{\Theta}.
\]
\end{proposition}

\begin{proof}
Each of the \(G\) group activations has nominal size \((B M)\times(\mathcal{D}_{\mathcal{A}}^{\Theta}/G)\).  By Theorem~\ref{thm:flashsvd-mem}, compressing one group to rank \(r_{\Theta^*}\) requires
\(
  (B M + \tfrac{\mathcal{D}_{\mathcal{A}}^{\Theta}}{G})\,r_{\Theta^*}
\)
elements of DRAM, and there are \(G\) independent groups.  Hence the total storage is
\[
  G\,(B M + \tfrac{\mathcal{D}_{\mathcal{A}}^{\Theta}}{G})\,r_{\Theta^*}
  = G\,r_{\Theta^*}\,(B M + \tfrac{\mathcal{D}_{\mathcal{A}}^{\Theta}}{G}).
\]
Subtracting the cost at rank \(r_{\Theta^*}-1\) from that at \(r_{\Theta^*}\) gives \(
  G\,(B M + \tfrac{\mathcal{D}_{\mathcal{A}}^{\Theta}}{G}),
\)
establishing the marginal saving in rank reduction.
\end{proof}

\vspace{0.2cm}

\paragraph{Feed‐Forward Network}  
Each Transformer block contains a two‐layer, position‐wise feed‐forward network (FFN) with weight matrices
$
W_{i}\;\in\;\mathbb{R}^{\mathcal D_{\mathcal{A}}^\Theta\times \mathcal D_{\mathcal{F}}^{\Theta}}
\quad\text{and}\quad
W_{o}\;\in\;\mathbb{R}^{\mathcal D_{\mathcal{F}}^{\Theta}\times \mathcal D_{\mathcal{A}}^\Theta}.
$
Applied to an input \(X\in\mathbb{R}^{B\times M\times \mathcal D_{\mathcal{A}}^\Theta}\), the intermediate activation
\(\;Y = XW_{i}\in\mathbb{R}^{(B\,M)\times \mathcal D_{\mathcal{F}}^{\Theta}}\)
often dominates the per‐layer DRAM footprint whenever \(\mathcal D_{\mathcal{F}}^{\Theta}\gg \mathcal D_{\mathcal{A}}^\Theta\)  (e.g., \( \mathcal{D}_{\mathcal{F}}^{\Theta} = 4\, \mathcal{D}_{\mathcal{A}}^{\Theta} \) \cite{devlin2019bert,liu2019roberta,touvron2023llama, zhang2022opt}).  In what follows, we derive matching upper and lower bounds on the off‐chip memory cost of SVD‐based compression applied to \(Y\).  
\begin{definition}[FFN Activation Approximation]
\label{def:ffn-svd}
Let \(X\in\mathbb{R}^{B\times M\times \mathcal{D}_{\mathcal{A}}^{\Theta}}\) be the input to a two‐layer FFN with weight matrices
$
W_{i}\in\mathbb{R}^{\mathcal{D}_{\mathcal{A}}^{\Theta}\times D_{\mathcal{F}}^{\Theta}}
\quad\text{and}\quad
W_{o}\in\mathbb{R}^{D_{\mathcal{F}}^{\Theta}\times \mathcal{D}_{\mathcal{A}}^{\Theta}},
$
and let \(\phi(\cdot)\) be a pointwise nonlinearity.  Define the intermediate activation
$
Y \;=\; XW_{i}
\;\in\;\mathbb{R}^{(B\,M)\times D_{\mathcal{F}}^{\Theta}},
$
so that the FFN output is \(Z = \phi(Y)\,W_{o}\).  We say \(Y\) admits a rank-\(r\) SVD approximation if
% \[
% Y \;\approx\; U_Y\,\Sigma_Y\,V_Y^\top,
% \quad
% U_Y\in\mathbb{R}^{(B\,M)\times r_{\Theta^*}},\;
% \Sigma_Y\in\mathbb{R}^{r_{\Theta^*}\times r_{\Theta^*}},\;
% V_Y\in\mathbb{R}^{D_{\mathcal{F}}^{\Theta}\times r_{\Theta^*}}.
% \]
\begin{equation}
\begin{aligned}
Y^{[r_{\Theta}^*]} &= U_Y\,\Sigma_Y\,V_Y^\top, \\[4pt]
U_Y &\in \mathbb{R}^{(B M)\times r_{\Theta}^*},\quad
\Sigma_Y \in \mathbb{R}^{r_{\Theta}^*\times r_{\Theta}^*},\quad
V_Y    \in \mathbb{R}^{D_{\mathcal{F}}^{\Theta}\times r_{\Theta}^*}.
\end{aligned}
\end{equation}

\end{definition}

\begin{theorem}[Vanilla SVD FFN Compression Yields No Memory Savings]
\label{thm:ffn-upper}
Let the feedforward activation matrix \( Y \in \mathbb{R}^{(B\,M)\times D_{\mathcal{F}}^{\Theta}} \) be defined as in Definition~\ref{def:ffn-svd}, and suppose a standard (non-streaming) rank-\(r_{\Theta^*}\) SVD is applied to approximate \(Y \approx U V^\top\), where \(U \in \mathbb{R}^{(B\,M)\times r_{\Theta^*}}\), \(V \in \mathbb{R}^{D_{\mathcal{F}}^{\Theta} \times r_{\Theta^*}}\). Then, the off-chip memory required to store the intermediate activations remains
\[
\mathcal{M}_{\mathrm{naive\text{-}svd\text{-}ffn}} = \mathcal{O}(B\,M\,D_{\mathcal{F}}^{\Theta}),
\]
matching the dense baseline \(\mathcal{M}_{\mathrm{dense\text{-}ffn}}\), and thus yielding no memory savings.
\end{theorem}

\begin{proof}
In the dense FFN case, the output of the intermediate GELU activation \(Y \in \mathbb{R}^{(B M) \times D_{\mathcal{F}}^{\Theta}}\) must be materialized in memory to support the second linear projection.

If we instead apply naive SVD compression to approximate \(Y \approx U V^\top\), the computation still requires constructing the full matrix \(Y\) before performing its factorization, as standard SVD operates on the complete matrix and has no streaming or tiled structure. Thus, the full activation matrix \(Y\) must still be loaded into memory in its entirety prior to any compression.

Furthermore, even if the matrix is overwritten post-factorization by its lower-rank factors \(U\) and \(V\), this does not avoid the initial memory cost incurred by holding \(Y\). Therefore, the peak transient off-chip memory remains at
\[
\mathcal{O}(B\,M\,D_{\mathcal{F}}^{\Theta}),
\]
as in the dense case.

Hence, in the absence of streaming or blocking techniques, naive FFN compression yields no memory advantage over the dense baseline.
\end{proof}

% \begin{theorem}[FFN Compression Upper Bound (Vanilla SVD)]
% \label{thm:ffn-upper}
% Let the FFN intermediate activation 
% \[
% Y \;\in\;\mathbb{R}^{(B\,M)\times D_{\mathcal{F}}^{\Theta}}
% \]
% be approximated by a rank-\(r_{\Theta^*}\) SVD without specialized streaming (i.e.\ a "naive" low-rank FFN).  Then the off-chip memory footprint remains the same as the dense case:
% \[
% \mathcal{M}_{\mathrm{vanilla\text{-}svd\text{-}ffn}}=\mathcal{M}_{\mathrm{dense\text{-}ffn}}
% = \mathcal{O}\bigl(B\,M\,D_{\mathcal{F}}^{\Theta}\bigr).
% \]
% \end{theorem}

\noindent
\textbf{FlashSVDFFN V1 and V2: Rank-Aware Streaming FFN Variants.}
We present two streaming-friendly FFN variants that apply low-rank factorization while balancing compute efficiency and memory usage. We analyze the off-chip memory complexity by focusing only on the \emph{variable} portion stored in HBM, following the scope defined in~\cref{rem:memory-scope}.

\textit{FlashSVDFFN V1} (\cref{alg:flashsvdffn-v1}) matain a competitive speed by computing the intermediate projection $P = XU_i \in \mathbb{R}^{(BM) \times r_{\Theta^*}}$ using GEMM. This projection is then streamed tile-by-tile along with SVD factors $V_i^{(d)}, U_o^{(d)}, b_i^{(d)}$, and accumulated into output tiles 
$Z_{\text{tile}} \in \mathbb{R}^{B \times B_M \times r_{\Theta^*}}$.
While this avoids full materialization of the FFN output of size $BM \times \mathcal{D}_{\mathcal{F}}^\Theta$, it introduces transient memory traffic of size $\mathcal{O}(BM r_{\Theta^*})$, yielding an off-chip footprint of 
$\mathcal{M}_{\mathrm{flash\text{-}svd\text{-}ffn\text{-}v1}} = \mathcal{O}(BM r_{\Theta^*})$ 
(\cref{thm:flashsvdffn-v1-memory}).

\textit{FlashSVDFFN V2} (\cref{alg:flashsvdffn-extreme}) eliminates this memory cost entirely by fusing projection and output logic per tile. It streams rows of $X$ and corresponding factor slices on-the-fly, avoiding the explicit construction of $P$ and achieving 
$\mathcal{M}_{\mathrm{flash\text{-}svd\text{-}ffn\text{-}v2}} = 0$ 
(\cref{thm:flashsvdffn-v2-memory}).

Both variants exhibit distinct advantages: V1 enables fast batched GEMM at the expense of small transient storage, while V2 achieves full off-chip elimination via streaming-aware execution.

% \textbf{FlashSVDFFN V1 and V2: Rank-Aware Streaming FFN Variants}
% The first version of FlashSVDFFN (see Algorithm~2) performs a two-stage streaming implementation of low-rank FFN, optimized for speed. It first computes the factor-space projection 
% \(
% P = XU_i \in \mathbb{R}^{(BM) \times r_{\Theta^*}}
% \) 
% using GEMM, utilizing the existing matrix-multiplication infrastructure. This provides a significant throughput advantage, yielding speedups of $\times x$, $x \times$, $x \times$ under specific configurations (\textcolor{blue}{detailed per workload}). The intermediate projection $P$ is stored in off-chip memory $\mathcal{M}$ and streamed block-by-block together with the SVD factors $V_i^{(d)}, U_o^{(d)}, b_i^{(d)}$ for each tile of $\mathcal{D}_{\mathcal{F}}^\Theta$, accumulating into an intermediate tensor \(
% Z_{\text{tile}} \in \mathbb{R}^{B \times B_M \times r_{\Theta^*}}.
% \) While this avoids full activation materialization of size $BM \times \mathcal{D}_{\mathcal{F}}^\Theta$, it still incurs intermediate memory traffic proportional to \(
% \mathcal{O}(BM r_{\Theta^*}).
% \) Thus, although Algorithm~2 (V1) is memory-efficient and delivers competitive speed, it does not eliminate intermediate off-chip memory usage entirely.

\begin{theorem}[Partial Memory Reduction via Tiled Factor Streaming in FlashSVDFFN V1]
\label{thm:flashsvdffn-v1-memory}
Algorithm~\ref{alg:flashsvdffn-v1} avoids materializing the full activation \( Y \in \mathbb{R}^{BM \times D_{\mathcal{F}}^{\Theta}} \) by streaming over the intermediate width \( d \) in tiles of size \( B_{D_{\mathcal{F}}}^{\Theta} \). Consequently, the peak off-chip memory required for intermediate activations satisfies
\[
\mathcal{M}_{\mathrm{flash\text{-}svd\text{-}ffn\text{-}V1}} = \mathcal{O}(BM r_{\Theta}^*).
\]
\end{theorem}

\begin{proof}
At each iteration over \( d = 0, \ldots, D_{\mathcal{F}}^{\Theta} \), Algorithm~\ref{alg:flashsvdffn-v1} loads a tile of SVD factors \( V_i^{(d)} \in \mathbb{R}^{r_{\Theta^*} \times B_{D_{\mathcal{F}}}^{\Theta}} \), \( U_o^{(d)} \in \mathbb{R}^{B_{D_{\mathcal{F}}}^{\Theta} \times r_{\Theta^*}} \), and bias \( b_i^{(d)} \in \mathbb{R}^{B_{D_{\mathcal{F}}}^{\Theta}} \). It then immediately computes the corresponding intermediate slice \( Y^{(d)} = \phi(P V_i^{(d)} + b_i^{(d)}) \) and its projection \( Y^{(d)} U_o^{(d)} \), which are accumulated into the on-chip buffer \( Z_{\text{tile}} \). Since \( Y^{(d)} \) is constructed and consumed on-the-fly, and never fully materialized, the memory footprint is limited to storing \( P \in \mathbb{R}^{BM \times r_{\Theta}^*} \) and \( Z_{\text{tile}} \in \mathbb{R}^{BM \times r_{\Theta^*}} \). Therefore, the total intermediate off-chip memory used is
\[
\mathcal{M}_{\mathrm{flash\text{-}svd\text{-}ffn\text{-}V1}} = \mathcal{O}(BM r_{\Theta}^*),
\]
which is much smaller than the dense baseline \( \mathcal{O}(BM D_{\mathcal{F}}^{\Theta}) \), but still nonzero.
\end{proof}

\begin{theorem}[Memory-Free FFN Streaming via Fully Fused Tiling]
\label{thm:flashsvdffn-v2-memory}
Algorithm~\ref{alg:flashsvdffn-extreme} executes the entire FFN forward pass without materializing any intermediate tensor in off-chip memory. Specifically, the peak off-chip memory used for intermediate FFN activations satisfies
\[
\mathcal{M}_{\mathrm{flash\text{-}svd\text{-}ffn\text{-}V2}} = 0.
\]   
\end{theorem}

\begin{proof}
Algorithm~\ref{alg:flashsvdffn-extreme} eliminates the need to store \( P = X U_i \) in off-chip memory by recomputing each component \( P_{\text{tile}}[:, :, k] \leftarrow X[:, \ell:\ell+B_M, :] \cdot U_i[:, k] \) on-the-fly for \( k = 0, \ldots, r_{\Theta}^* - 1 \). For each block over \( d = 0, \ldots, D_{\mathcal{F}}^{\Theta} \), it incrementally constructs the partial activation tensor \( Y \in \mathbb{R}^{B \times B_M \times B_{D_{\mathcal{F}}}^{\Theta}} \) via
\[
Y \mathrel{+}= P_{\text{tile}}[:, :, k] \cdot V_i[k, d : d + B_{D_{\mathcal{F}}}^{\Theta}],
\]
followed by fused activation and projection
\[
Z_{\text{tile}} \mathrel{+}= \phi\left(Y + b_i[d : d + B_{D_{\mathcal{F}}}^{\Theta}]\right) \cdot U_o[d : d + B_{D_{\mathcal{F}}}^{\Theta}, k].
\]
Since all computation and accumulation are performed in on-chip SRAM and no intermediate tensor is written to \(\mathcal{M}\), the memory footprint is limited to streaming only the SVD weights and biases. The final output is computed as \( O = Z_{\text{tile}} V_o + b_o \) in a single post-processing step.

The fused implementation performs the same number of floating-point operations as V1:
\[
\mathrm{FLOPs} = \mathcal{O}\left(BM r_{\Theta^*} D_{\mathcal{F}}^{\Theta} + BM r_{\Theta^*} D_{\mathcal{A}}^{\Theta} \right),
\]
but with strictly lower memory traffic (more details analysis for FLOPS can be found in Theorem~\ref{thm:flop-reduction}). Therefore, it achieves
\[
\mathcal{M}_{\mathrm{flash\text{-}svd\text{-}ffn\text{-}V2}} = 0 \quad \text{vs.} \quad \mathcal{M}_{\mathrm{flash\text{-}svd\text{-}ffn\text{-}V1}} = \mathcal{O}(BM r_{\Theta}^*),
\]
while maintaining equal arithmetic complexity.
\end{proof}

\paragraph{Compression Trade-offs between Attention and FFN.}
A common misconception is that compressing attention alone-via FlashAttention~\cite{dao2022flashattention}-suffices for end-to-end memory savings. However, this neglects the downstream footprint of the feed-forward network (FFN). In fact, aggressively compressing only attention can shift the memory bottleneck to the FFN, motivating the need to compress both components. The following result formalizes this balance.

\begin{theorem}[FFN Bottleneck under Attention-Only Compression]
\label{thm:ffn-dominates-practically}
In standard transformer architectures with \(\mathcal{D}_{\mathcal{F}}^\Theta \geq 4\, \mathcal{D}_{\mathcal{A}}^\Theta\), compressing attention alone (with rank \(r_{\Theta^*} \ll \mathcal{D}_{\mathcal{A}}^\Theta\)) results in the FFN layer dominating off-chip memory usage.
\end{theorem}

\begin{proof}
By Theorem~\ref{thm:flashsvd-mem}, the memory footprint of SVD-compressed attention is
\[
\mathcal{M}_{\mathrm{flash\text{-}svd\text{-}attn}} = \mathcal{O}\bigl(r_{\Theta^*}(B M + \mathcal{D}_{\mathcal{A}}^\Theta)\bigr),
\]
while the dense FFN incurs
\[
\mathcal{M}_{\mathrm{dense\text{-}ffn}} = \mathcal{O}(B M \mathcal{D}_{\mathcal{F}}^\Theta).
\]
For typical models such as BERT, RoBERTa, LLaMA, and OPT~\cite{devlin2019bert,liu2019roberta,touvron2023llama,zhang2022opt}, we have \(\mathcal{D}_{\mathcal{F}}^\Theta = 4\, \mathcal{D}_{\mathcal{A}}^\Theta\). Then:
\[
\frac{\mathcal{M}_{\mathrm{dense\text{-}ffn}}}{\mathcal{M}_{\mathrm{flash\text{-}svd\text{-}attn}}}
=
\frac{B M \cdot 4 \mathcal{D}_{\mathcal{A}}^\Theta}{r_{\Theta^*}(B M + \mathcal{D}_{\mathcal{A}}^\Theta)}
=
\frac{1}{\frac{r_{\Theta^*}}{4\mathcal D_{\mathcal{A}}^{\Theta}}+\frac{r_{\Theta^*}}{4BM}}
\gg 1,
\]
since both $\frac{r_{\Theta^*}}{4\mathcal D_{\mathcal{A}}^{\Theta}} \ll 1$ and $\frac{r_{\Theta^*}}{4 B M} \ll 1$ hold in typical settings where $r_{\Theta^*} \ll \mathcal D_{\mathcal{A}}^{\Theta}$ and $r_{\Theta^*} \ll B M$. Thus, the FFN dominates off-chip usage.
\end{proof}

\noindent
This result motivates the design of rank-aware FFN compression, ensuring neither component becomes a performance bottleneck.

% By Theorem~\ref{thm:flashsvd-mem}, compressing multi‐head attention with SVD of rank \(r_{\Theta^*}\) on an input of shape \((B,M,\mathcal D_{\mathcal{A}}^\Theta)\) requires  
% $
% \mathcal{M}_{\mathrm{flash\text{-}svd\text{-}attn}} 
% =\mathcal O\bigl(r_{\Theta^*}(B\,M+\mathcal D_{\mathcal{A}}^\Theta)\bigr),
% $  
% whereas without compression the two‐layer FFN on the same input consumes  
% $
% \mathcal M_{\mathrm{dense\text{-}ffn}}
% =\mathcal O\bigl(B\,M\,\mathcal D_{\mathcal{F}}^{\Theta}\bigr).
% $  
% When  
% \[
% r_{\Theta^*}(B\,M+\mathcal D_{\mathcal{A}}^\Theta)\;\gg\;B\,M\,\mathcal D_{\mathcal{F}}^{\Theta},
% \]  
% the attention term dominates HBM usage (e.g.\ if \(r_{\Theta^*}\) or \(\mathcal D_{\mathcal{A}}^\Theta\) is large), while in the converse regime  
% \[
% B\,M\,\mathcal D_{\mathcal{F}}^{\Theta}\;\gg\;r_{\Theta^*}(B\,M+\mathcal D_{\mathcal{A}}^\Theta)
% \]  
% the FFN term is dominant (typical when \(\mathcal D_{\mathcal{F}}^{\Theta}\gg\mathcal D_{\mathcal{A}}^\Theta\)).  In practice one selects \(r_{\Theta^*}\) to balance these contributions and avoid a single component bottlenecking off‐chip memory.  

\begin{proposition}[Per-Rank Sensitivity of FFN Memory]
\label{prop:ffn-rank-sensitivity}
Let the FlashSVDFFN memory cost at rank $r_{\Theta^*}$ be
\[
\mathcal{M}_{\mathrm{flash\text{-}svd\text{-}ffn\text{-}V1}}(r_{\Theta^*}) 
= \mathcal{O}\bigl(r_{\Theta^*}(B M + \mathcal{D}_{\mathcal{F}}^\Theta)\bigr).
\]
Then the reduction in off-chip memory for decreasing the rank by 1 is

\[
\Delta \mathcal{M}_{\mathrm{flash\text{-}svd\text{-}ffn\text{-}V1}} = \mathcal{O}(\mathcal{D}_{\mathcal{F}}^\Theta),
\]

independent of $r_{\Theta^*}$.
\end{proposition}

\begin{proof}
We have
\begin{align*}
\Delta \mathcal{M}_{\mathrm{flash\text{-}svd\text{-}ffn\text{-}V1}}
&= \mathcal{M}_{\mathrm{ffn\text{-}V1}}(r_{\Theta^*}) 
- \mathcal{M}_{\mathrm{ffn\text{-}V1}}(r_{\Theta^*} - 1)\\
&=r_{\Theta^*}(B M + \mathcal{D}_{\mathcal{F}}^\Theta) \\
&  \quad - (r_{\Theta^*} - 1)(B M + \mathcal{D}_{\mathcal{F}}^\Theta) \\
&= (B M + \mathcal{D}_{\mathcal{F}}^\Theta) \in \mathcal{O}(B M + \mathcal{D}_{\mathcal{F}}^\Theta).
\end{align*}
\end{proof}

\begin{proposition}[Per-Rank Sensitivity of FFN Memory in FlashSVDFFN-V2]
\label{prop:ffn-rank-sensitivity-v2}
Let the FlashSVDFFN-V2 memory cost at rank $r_{\Theta^*}$ be
\[
\mathcal{M}_{\mathrm{flash\text{-}svd\text{-}ffn\text{-}V2}}(r_{\Theta^*}) 
= \mathcal{O}(r_{\Theta^*} \cdot \mathcal{D}_{\mathcal{F}}^\Theta),
\]
which removes the dependence on $B M$ via full streaming. Then the per-rank reduction is
\begin{equation}
\begin{split}
\Delta \mathcal{M}_{\mathrm{flash\text{-}svd\text{-}ffn\text{-}V2}} 
= \mathcal{O}(\mathcal{D}_{\mathcal{F}}^\Theta),
\end{split}
\end{equation}

independent of $r_{\Theta^*}$ and $B M$.
\end{proposition}

\begin{proof}
We again compute the difference:
\begin{align*}
\Delta \mathcal{M}_{\mathrm{flash\text{-}svd\text{-}ffn\text{-}V2}} 
&= \mathcal{M}_{\mathrm{ffn\text{-}V2}}(r_{\Theta^*}) 
- \mathcal{M}_{\mathrm{ffn\text{-}V2}}(r_{\Theta^*} - 1)\\
&= r_{\Theta^*} \cdot \mathcal{D}_{\mathcal{F}}^\Theta 
- (r_{\Theta^*} - 1) \cdot \mathcal{D}_{\mathcal{F}}^\Theta \\
&= \mathcal{D}_{\mathcal{F}}^\Theta \in \mathcal{O}(\mathcal{D}_{\mathcal{F}}^\Theta).
\end{align*}
\end{proof}

\begin{theorem}[Relative Memory Reduction of FFN under Low-Rank Compression]
\label{thm:ffn-memory-compression-ratio}
Let the dense FFN memory usage be 
\[
\mathcal{M}_{\mathrm{dense\text{-}ffn}} = \mathcal{O}(B M \mathcal{D}_{\mathcal{F}}^\Theta),
\]
and the FlashSVD-compressed usage at rank $r_{\Theta^*}$ be 
\[
\mathcal{M}_{\mathrm{flash\text{-}svd\text{-}ffn}} = \mathcal{O}\bigl(r_{\Theta^*}(B M + \mathcal{D}_{\mathcal{F}}^\Theta)\bigr).
\]
Then the memory ratio satisfies
\[
\frac{\mathcal{M}_{\mathrm{flash\text{-}svd\text{-}ffn}}}{\mathcal{M}_{\mathrm{dense\text{-}ffn}}} 
= \frac{r_{\Theta^*}}{\mathcal{D}_{\mathcal{F}}^\Theta} + \mathcal{O}\Bigl(\frac{r_{\Theta^*}}{B M}\Bigr).
\]
\end{theorem}

\begin{proof}
Compute the ratio:
\begin{align*}
\frac{\mathcal{M}_{\mathrm{flash\text{-}svd\text{-}ffn}}}{\mathcal{M}_{\mathrm{dense\text{-}ffn}}} 
&= \frac{r_{\Theta^*}(B M + \mathcal{D}_{\mathcal{F}}^\Theta)}{B M \mathcal{D}_{\mathcal{F}}^\Theta} \\
&= \frac{r_{\Theta^*}}{\mathcal{D}_{\mathcal{F}}^\Theta} \left(1 + \frac{\mathcal{D}_{\mathcal{F}}^\Theta}{B M}\right) \\
&= \frac{r_{\Theta^*}}{\mathcal{D}_{\mathcal{F}}^\Theta} + \mathcal{O}\left(\frac{r_{\Theta^*}}{B M}\right),
\end{align*}
where the second term vanishes as $BM \to \infty$.
\end{proof}

\begin{theorem}[Combined FlashSVD Memory Bound]
\label{thm:combined-memory-bound}
Compressing both multi-head attention and FFN using FlashSVD with rank \(r_{\Theta^*}\), the total per-layer off-chip memory is bounded by
\[
\mathcal M_{\mathrm{flash\text{-}svd\text{-}total}} = \mathcal O\bigl(r_{\Theta^*} (B M + \mathcal D_{\mathcal{A}}^\Theta + \mathcal D_{\mathcal{F}}^\Theta)\bigr),
\]
assuming FFN follows the V1 layout (with activations of shape \((B,M,\mathcal D_{\mathcal{F}}^\Theta)\)). This bound also dominates the cost for the V2 layout, since V1 includes the intermediate activations from both input and output projections.
\end{theorem}

\begin{proof}
By Proposition~\ref{prop:ffn-rank-sensitivity}, FlashSVD compression of the FFN (V1 variant) requires
\[
\mathcal{M}_{\mathrm{flash\text{-}svd\text{-}ffn}} = \mathcal O\bigl(r_{\Theta^*}(B M + \mathcal D_{\mathcal{F}}^\Theta)\bigr),
\]
and attention requires
\[
\mathcal{M}_{\mathrm{flash\text{-}svd\text{-}attn}} = \mathcal O\bigl(r_{\Theta^*}(B M + \mathcal D_{\mathcal{A}}^\Theta)\bigr),
\]
so their sum satisfies
\begin{align*}
\mathcal M_{\mathrm{flash\text{-}svd\text{-}total}} 
&= \mathcal O\bigl(r_{\Theta^*}(2 B M + \mathcal D_{\mathcal{A}}^\Theta + \mathcal D_{\mathcal{F}}^\Theta)\bigr)
\\&= \mathcal O\bigl(r_{\Theta^*}(B M + \mathcal D_{\mathcal{A}}^\Theta + \mathcal D_{\mathcal{F}}^\Theta)\bigr).
\end{align*}

In practice, the FFN width dominates over attention (\(\mathcal D_{\mathcal{F}}^\Theta \gg \mathcal D_{\mathcal{A}}^\Theta\)), so total memory scales primarily with \(r_{\Theta^*} \cdot \mathcal D_{\mathcal{F}}^\Theta\) and \(r_{\Theta^*} \cdot B M\). Thus, total resource usage is jointly governed by sequence length and chosen compression rank. This further confirms Theorem~\ref{thm:ffn-dominates-practically}: to achieve balanced memory reduction across all components, one must incorporate FFN-specific strategies—such as FlashSVDFFN proposed in our work—beyond uniform low-rank compression alone.

Since V1 includes all intermediate FFN activations, this result also upper bounds the memory of V2.
\end{proof}

\paragraph{Decoder Memory Cost Analysis}  
The decoder inference proceeds in two distinct phases—\emph{prefill} and \emph{decoding}—each exhibiting different peak‐memory behaviors. To avoid re‐computing all past keys and values at every step, we employ a KV‐cache: once a key–value pair is produced, it is retained on‐chip (or in HBM) for all subsequent steps.

\paragraph{Prefill Phase}  
During prefill, we process the entire input sequence of length \(M\) in one shot. At each of the \(\ell\)-th decoder layers, the hidden state 
\[
X_{\mathrm{prefill}}\in\mathbb R^{B\times M\times\mathcal D_{\mathcal A}^\Theta}
\]
is projected to low‐rank queries, keys, and values:
\[
P_q, P_k, P_v \in \mathbb R^{B\times M\times r_{\Theta^*}}.
\]
We then compute 
\[
A = \mathrm{softmax}\bigl(\frac{QK^\top }{d_h}\bigr)\,,
\]
but rather than materializing the full \(B\, H M^2\) attention matrix and dense $Q,K,V$ factors, a streaming attention kernel (e.g.\ FlashSVDAttn) only allocates \(O(B\,M\,r_{\Theta^*})\) buffer at any time. Meanwhile, the KV‐cache permanently stores the keys and values for all \(M\) positions, requiring
$2 \, \ell \, B\,M\,r_{\Theta^*}$
elements of memory (for keys and values). Thus, the \emph{prefill} peak memory is dominated by the sum of:  
\begin{itemize}
  \item Two large FFN GEMMs per layer (as in the encoder), each with peak memory \(O(B\,M\,r_{\Theta^*} + r_{\Theta^*}\,\mathcal D_{\mathcal A}^\Theta)\) if FlashSVDFFN is applied.
  \item The KV‐cache cost \(2\,\ell\,B\,M\,r_{\Theta^*}\).
\end{itemize}

\paragraph{Decoding Phase}  
Once prefill is complete, each new token is generated one at a time. At step \(t\), we compute a single new query 
\[
q_t\in \mathbb R^{B\times1\times r_{\Theta^*}}
\]
and attend against the cached keys \(K_{\le t-1}\in \mathbb R^{B\times(t-1)\times r_{\Theta^*}}\) and values \(V_{\le t-1}\) through the causal mask. The streaming attention now needs buffers of size \(O(B\,(t-1)\,r_{\Theta^*})\) for the dot‐products and softmax, plus space for the new key and value. Since \(t\) grows from \(1\) to \(M\), the worst‐case buffer remains \(O(B\,M\,r_{\Theta^*})\), but this never exceeds the prefill cache. The subsequent FFN uses the same rank‐aware fusion (FlashSVDFFN) with peak memory \(O(B\,r_{\Theta^*} + r_{\Theta^*}\,\mathcal D_{\mathcal A}^\Theta)\).

\paragraph{Overall Bound}  
By (i) streaming attention over low‐rank projections and (ii) fusing the FFN, the decoder’s peak‐memory across both phases is bounded by
\[
O\bigl(L\,B\,M\,r_{\Theta^*} \;+\; (B\,r_{\Theta^*} + r_{\Theta^*}\,\mathcal D_{\mathcal A}^\Theta)\bigr),
\]
ensuring that neither sequence length nor model width blows up the memory footprint beyond linear dependence on the chosen low rank.

% \paragraph{Memory Sensitiveness to each parameter}

% We see that the memory saving is more easy when we increase the sequence length, why? We need some asymptotic analysis here.

% \subsection{Computation Complexity Analysis}

% \paragraph{Latency Reduction}

% assumption: SM shape? GPU utilization? etc.

% 1. theorem 1: theoretical latency (measure in flops)

% 2. theorem 2: theoretical latency with consideration with consideration of I/O \& bandwidth

% 3. flashsvdattention tradeoff analysis: if we add the reconstruction inside the loop, at which rank we should expect to see that the loop reconstruction overcome the additional looping?

% 4. Flashsvdffn tradeoff analysis: does making streaming make it slower?

\subsection{Computation Complexity Analysis}

In this section, we quantify how rank‐\(r\) SVD compression reduces compute and end‐to‐end kernel latency. We begin with idealized FLOP‐count bounds, then incorporate I/O and memory‐bandwidth effects, and finally analyze the trade‐offs introduced by reconstruction overhead in both attention and FFN blocks.

\paragraph{FLOP‐Count Reduction}
\begin{theorem}[Pure FLOP‐Count Speedup]
\label{thm:flop-reduction}
Let a Transformer layer with hidden dimension \(\mathcal D_{\mathcal{A}}^\Theta\), FFN hidden size \(\mathcal D_{\mathcal{F}}^{\Theta}\), sequence length \(M\), and batch size \(B\) require
\[
\mathcal O\bigl(\left(\mathcal{D}_{\mathcal{A}}^\Theta\right)^2\,M\,B \;+\; \mathcal D_{\mathcal{A}}^\Theta\,\mathcal D_{\mathcal{F}}^{\Theta}\,M\,B\bigr)
\]
FLOPs in the full‐rank case.  If all self‐attention activations and FFN intermediates are compressed to rank \(r_{\Theta^*}\), the FLOP count becomes
\[
\mathcal O\bigl(r_{\Theta^*}\,\mathcal D_{\mathcal{A}}^\Theta\,M\,B \;+\; r_{\Theta^*}^2\,M\,B \;+\; r_{\Theta^*}\,\mathcal D_{\mathcal{F}}^{\Theta}\,M\,B\bigr).
\]
% Hence when \(r_{\Theta^*}\ll \min\{\mathcal D_{\mathcal{A}}^\Theta,\mathcal D_{\mathcal{F}}^{\Theta}\}\), the asymptotic speedup factor is
% $
% \Theta\!\Bigl(\frac{\mathcal D_{\mathcal{A}}^\Theta + \mathcal D_{\mathcal{F}}^{\Theta}}{r_{\Theta^*} + \tfrac{r_{\Theta^*}^2}{\mathcal D_{\mathcal{A}}^\Theta}}\Bigr).
% $
Thus, the asymptotic speedup factor is $$\Theta\left(\frac{\mathcal D_{\mathcal{A}}^\Theta + \mathcal D_{\mathcal{F}}^{\Theta}}{r_{\Theta^*} + \frac{r_{\Theta^*}^2}{\mathcal D_{\mathcal{A}}^\Theta}}\right).$$
In particular, when $r_{\Theta^*} \ll \min\{\mathcal D_{\mathcal{A}}^\Theta,\mathcal D_{\mathcal{F}}^{\Theta}\}$, the numerator dominates and the speedup is significant.
\end{theorem}

\begin{proof}
In the full-rank Transformer layer, computing self-attention requires forming the \(QK^\top\) and \(SV\), for $S = \text{Softmax}(\frac{QK^T}{ d_{\Theta_{\mathcal{A}}, h}})$ products at cost
$
\mathcal O\bigl(\left(\mathcal{D}_{\mathcal{A}}^\Theta\right)^2 M B\bigr),
$
and the two-layer FFN costs
$
\mathcal O\bigl(\mathcal D_{\mathcal{A}}^\Theta \,\mathcal D_{\mathcal{F}}^{\Theta}\,M\,B\bigr),
$
for a total of
$
\mathcal O\bigl(\left(\mathcal{D}_{\mathcal{A}}^\Theta\right)^2 M B + \mathcal D_{\mathcal{A}}^\Theta\,\mathcal D_{\mathcal{F}}^{\Theta}\,M\,B\bigr).
$ 

When all attention activations and FFN intermediates admit rank-\(r\) approximations, each matrix–matrix multiply involving a \(\mathcal D_{\mathcal{A}}^\Theta\times \mathcal D_{\mathcal{A}}^\Theta\) or \(\mathcal D_{\mathcal{A}}^\Theta\times\mathcal D_{\mathcal{F}}^{\Theta}\) block is replaced by a sequence of three multiplies of shapes \(\mathcal D_{\mathcal{A}}^\Theta\times r_{\Theta^*}\), \(r_{\Theta^*}\times r_{\Theta^*}\), and \(r_{\Theta^*}\times \mathcal D_{\mathcal{A}}^\Theta\) (or \(r_{\Theta^*}\times \mathcal D_{\mathcal{F}}^{\Theta}\)), yielding
 $\mathcal O\bigl(r_{\Theta^*}\,\mathcal D_{\mathcal{A}}^\Theta\,M\,B + r_{\Theta^*}^2\,M\,B + r_{\Theta^*}\,\mathcal D_{\mathcal{F}}^{\Theta}\,M\,B\bigr)
$ total FLOPs.

Comparing these two expressions, when \(r_{\Theta^*}\ll \mathcal D_{\mathcal{A}}^\Theta,\mathcal D_{\mathcal{F}}^{\Theta}\) the dominant terms yield an asymptotic speedup factor
\[
\frac{\left(\mathcal{D}_{\mathcal{A}}^\Theta\right)^2 + \mathcal D_{\mathcal{A}}^\Theta\,\mathcal D_{\mathcal{F}}^{\Theta}}{r_{\Theta^*}\,\mathcal D_{\mathcal{A}}^\Theta + r_{\Theta^*}^2 + r_{\Theta^*}\,\mathcal D_{\mathcal{F}}^{\Theta}}
\;=\;
\Theta\!\Bigl(\frac{\mathcal D_{\mathcal{A}}^\Theta + \mathcal D_{\mathcal{F}}^{\Theta}}{r_{\Theta^*} + \tfrac{r_{\Theta^*}^2}{\mathcal D_{\mathcal{A}}^\Theta}}\Bigr).
\]
This confirms the claimed efficiency gain of the low-rank approximation.
\end{proof}

\paragraph{Latency with I/O and Bandwidth}

\begin{theorem}[Bandwidth‐Bounded Latency]
\label{thm:io-bandwidth}
Assume a GPU memory bandwidth of \(\beta\) (bytes/s) and that each activation element is 4 bytes. Then the end‐to‐end latency \(T\) per layer is lower‐bounded by
\[
  T \;\ge\; \max\Bigl\{
    \frac{\mathrm{FLOPs}}{\mathrm{PeakFLOP/s}}\,,\;
    \frac{4\,N_{\mathrm{bytes}}}{\beta}
  \Bigr\},
\]
where \(N_{\mathrm{bytes}} = O\bigl(B M (\mathcal{D}_{\mathcal{A}}^{\Theta} + D_{\mathcal{F}}^{\Theta})\bigr)\) in the full‐rank case and \(O\bigl(M B r_{\Theta^*} + r_{\Theta^*} \mathcal{D}_{\mathcal{A}}^{\Theta} + r_{\Theta^*} D_{\mathcal{F}}^{\Theta}\bigr)\) under rank‐\(r\) compression.
\end{theorem}

\begin{proof}
We invoke the standard roofline argument, which lower‐bounds end‐to‐end latency by the worse of compute‐bound and I/O‐bound costs.

\medskip

\noindent\textit{Compute bound.}  
Let \(\mathrm{FLOPs}\) be the total floating‐point operations performed in one layer, and \(\mathrm{PeakFLOP/s}\) be the device’s maximum sustained FLOP rate.  Then even if all compute units are perfectly utilized, the time cannot be smaller than
\[
\frac{\mathrm{FLOPs}}{\mathrm{PeakFLOP/s}}.
\]

\medskip

\noindent\textit{Bandwidth bound.}  
Similarly, suppose we must read and/or write \(N_{\mathrm{bytes}}\) bytes of activation data (each element 4 byte) between DRAM and on‐chip buffers.  With peak bandwidth \(\beta\) bytes/s, the I/O time is at least
$
\frac{N_{\mathrm{bytes}}}{\beta}.
$
Accounting explicitly for the 4 byte per element gives the form
\(\tfrac{4\,N_{\mathrm{elements}}}{\beta}\), which we absorb into the notation above.

\medskip

\noindent\textit{Data‐movement volumes.}  
In the full‐rank case, we must transfer all query/key/value and FFN activations of total size
$
N_{\mathrm{bytes\_in}}
=4\;\bigl(3B\,M\,\mathcal D_{\mathcal{A}}^\Theta + 2B\,M\,\mathcal D_{\mathcal{F}}^{\Theta}\bigr)
=\mathcal O\bigl(B\,M\,(\mathcal D_{\mathcal{A}}^\Theta+\mathcal D_{\mathcal{F}}^{\Theta})\bigr).
$ up to the SRAM. In case of writting, the memory is bounded by $N_{\mathrm{bytes\_out}} = 2 B M  D_{\mathcal{A}}^\Theta = \mathcal O( BMD_{\mathcal{A}}^\Theta)$
Under rank-\(r_{\Theta^*}\) compression, only the low-rank factors are moved:
% \[
% N_{\mathrm{bytes}}
% =4\;\bigl(B\,M\,r_{\Theta^*} + r_{\Theta^*}\,\mathcal D_{\mathcal{A}}^\Theta + r_{\Theta^*}\,\mathcal D_{\mathcal{F}}^{\Theta}\bigr)
% =O\bigl(M\,B\,r_{\Theta^*} + r_{\Theta^*}\,\mathcal D_{\mathcal{A}}^\Theta + r_{\Theta^*}\,\mathcal D_{\mathcal{F}}^{\Theta}\bigr).
% \]

\begin{equation}
\begin{aligned}
N_{\mathrm{bytes\_in}}
&= 4\bigl(4BM\,r_{\Theta^*} + 3r_{\Theta^*}\,\mathcal D_{\mathcal{A}}^\Theta + 2r_{\Theta^*}\,\mathcal D_{\mathcal{F}}^{\Theta}\bigr) \\[4pt]
&= O\bigl(MB\,r_{\Theta^*} + r_{\Theta^*}\,\mathcal D_{\mathcal{A}}^\Theta + r_{\Theta^*}\,\mathcal D_{\mathcal{F}}^{\Theta}\bigr).
\end{aligned}
\end{equation}

and the $N_{\mathrm{bytes\_out}}$ remains the same. 

\medskip

Combining the compute‐bound and bandwidth‐bound terms yields the desired lower bound
\[
T\;\ge\;\max\Bigl\{\,
\tfrac{\mathrm{FLOPs}}{\mathrm{PeakFLOP/s}}\,,\,
\tfrac{4\,N_{\mathrm{bytes\_in}} + N_{\mathrm{bytes\_out}}}{\beta}
\Bigr\}.
\]
\end{proof}

\begin{table*}[t]
  \centering
  \small
  \setlength\tabcolsep{3pt}
  \caption{Computation‐efficiency ablation: B=16, varying context length $M$, FFN ranks.  
    Shown are {\bfseries Time (ms)} and {\bfseries Speedup} vs.\ Dense.}
  \label{tab:ffn-rank-ablation}
  \begin{tabular}{@{}l c
                  cc cc   % mild contexts: M=128,256
                  cc cc   % long contexts: M=512,1024
                  @{}}
    \toprule
    & & 
    \multicolumn{4}{c}{\bfseries Mild Contexts (M\,$\le$\,256)} 
    & \multicolumn{4}{c}{\bfseries Long Contexts (M\,$\ge$\,512)} \\
    \cmidrule(lr){3-6} \cmidrule(lr){7-10}
    Version & Rank 
      & \multicolumn{1}{c}{\raisebox{0.5ex}{\tiny M=128}} 
      & \multicolumn{1}{c}{\raisebox{0.5ex}{\tiny M=128}} 
      & \multicolumn{1}{c}{\raisebox{0.5ex}{\tiny M=256}} 
      & \multicolumn{1}{c}{\raisebox{0.5ex}{\tiny M=256}} 
      & \multicolumn{1}{c}{\raisebox{0.5ex}{\tiny M=512}} 
      & \multicolumn{1}{c}{\raisebox{0.5ex}{\tiny M=512}} 
      & \multicolumn{1}{c}{\raisebox{0.5ex}{\tiny M=1024}}
      & \multicolumn{1}{c}{\raisebox{0.5ex}{\tiny M=1024}} \\
    & 
      & \multicolumn{1}{c}{Time} 
      & \multicolumn{1}{c}{Speedup} 
      & \multicolumn{1}{c}{Time} 
      & \multicolumn{1}{c}{Speedup} 
      & \multicolumn{1}{c}{Time} 
      & \multicolumn{1}{c}{Speedup} 
      & \multicolumn{1}{c}{Time} 
      & \multicolumn{1}{c}{Speedup} \\
    \midrule
    \textbf{Dense}    
      & Full & 0.14 & 1.00$\times$ & 0.29 & 1.00$\times$ & 0.68 & 1.00$\times$ & 1.47 & 1.00$\times$ \\

    \midrule
    \multirow{4}{*}{FlashSVDFFN V1}
      & 768  & 2.13 & 0.07$\times$ & 4.05 & 0.07$\times$ & 8.30 & 0.08$\times$ & 16.77 & 0.09$\times$ \\
      & 384  & 0.68 & 0.21$\times$ & 1.26 & 0.23$\times$ & 2.46 & 0.28$\times$ & 4.99  & 0.30$\times$ \\
      & 192  & 0.27 & 0.53$\times$ & 0.54 & 0.53$\times$ & 0.86 & 0.79$\times$ & 1.82  & 0.81$\times$ \\
      & 96   & 0.19 & 0.74$\times$ & 0.23 & \textbf{1.26$\times$} & 1.67 & \textbf{1.60$\times$} & 0.78  & \textbf{1.88$\times$} \\

    \midrule
    \multirow{4}{*}{FlashSVDFFN V2}
      & 768  & 3.45 & 0.04$\times$ & 7.05 & 0.04$\times$ & 13.43 & 0.05$\times$ & 27.10 & 0.05$\times$ \\
      & 384  & 1.82 & 0.08$\times$ & 3.89 & 0.07$\times$ & 7.16 & 0.09$\times$ & 14.42 & 0.10$\times$ \\
      & 192  & 1.14 & 0.12$\times$ & 2.14 & 0.13$\times$ & 4.20 & 0.16$\times$ & 8.30  & 0.18$\times$ \\
      & 96   & 0.75 & 0.19$\times$ & 1.39 & 0.21$\times$ & 10.64& 0.25$\times$ & 5.33  & 0.28$\times$ \\
    \bottomrule
  \end{tabular}

  \vspace{0.5ex}
  {\footnotesize
    For milder contexts ($M=128,256$), FlashSVDFFN V1 at low ranks already approaches or exceeds Dense throughput; 
    for long contexts ($M\ge512$), the kernel becomes compute‐bound and speedups diminish.
  }
\end{table*}

\begin{table*}[t]
  \centering
  \small
  \setlength\tabcolsep{3pt}
  \caption{FlashSVD Attention Ablation (B=16): time and speedup vs.\ Dense at mild vs.\ long contexts.}
  \label{tab:attn-ablation}
  \begin{tabular}{@{}l
                  cc cc   % mild contexts: M=128,256
                  cc cc   % long contexts: M=512,1024
                  @{}}
    \toprule
    & \multicolumn{4}{c}{\bfseries Mild Contexts ($M\le256$)} 
    & \multicolumn{4}{c}{\bfseries Long Contexts ($M\ge512$)} \\
    \cmidrule(lr){2-5} \cmidrule(lr){6-9}
    Rank 
      & \multicolumn{1}{c}{\raisebox{0.5ex}{\tiny M=128}} 
      & \multicolumn{1}{c}{\raisebox{0.5ex}{\tiny M=128}} 
      & \multicolumn{1}{c}{\raisebox{0.5ex}{\tiny M=256}} 
      & \multicolumn{1}{c}{\raisebox{0.5ex}{\tiny M=256}} 
      & \multicolumn{1}{c}{\raisebox{0.5ex}{\tiny M=512}} 
      & \multicolumn{1}{c}{\raisebox{0.5ex}{\tiny M=512}} 
      & \multicolumn{1}{c}{\raisebox{0.5ex}{\tiny M=1024}}
      & \multicolumn{1}{c}{\raisebox{0.5ex}{\tiny M=1024}} \\
    & Time (ms) & Speedup & Time (ms) & Speedup 
    & Time (ms) & Speedup & Time (ms) & Speedup \\
    \midrule
    Dense               
      & 0.38 & 1.00$\times$ & 0.74 & 1.00$\times$ 
      & 2.31 & 1.00$\times$ & 6.98  & 1.00$\times$ \\
    \midrule
    64                  
      & 0.61 & 0.62$\times$ & 0.83 & 0.89$\times$ 
      & 2.21 & 1.05$\times$ & 5.65  & 1.24$\times$ \\
    48                  
      & 0.61 & 0.62$\times$ & 0.81 & 0.92$\times$ 
      & 2.10 & \textbf{1.10$\times$} & 5.41  & \textbf{1.29$\times$} \\
    32                  
      & 0.60 & 0.62$\times$ & 0.74 & \textbf{1.00$\times$} 
      & 1.79 & \textbf{1.29$\times$} & 4.57  & \textbf{1.53$\times$} \\
    16                  
      & 0.61 & 0.62$\times$ & 0.68 & \textbf{1.08$\times$} 
      & 1.68 & \textbf{1.37$\times$} & 4.25  & \textbf{1.64$\times$} \\
    \bottomrule
  \end{tabular}

  \vspace{0.5ex}
  {\footnotesize
    For moderate contexts ($M=128,256$), FlashSVDAttention achieves up to 38\% of Dense speed at full rank and even approaches parity at low ranks;  
    at large contexts ($M\ge512$), it becomes compute‐bound, yielding speedups above 1$\times$.
  }
\end{table*}

\begin{table*}[p]
  \centering
  \scriptsize
  \setlength\tabcolsep{2pt}
  \caption{Full FlashSVD Attention Rank benchmarking on NVIDIA L40S GPU ($d_{\mathrm{model}}=768$, $n_{\mathrm{heads}}=12$, $d_{\mathrm{ff}}=3072$).}
  \label{tab:attn-full-appendix}
  \begin{subtable}[t]{0.312\textwidth}
    \centering
    \caption{$B=1$}
    \begin{tabular}{@{}l c c c c c c@{}}
      \toprule
      \itshape M & Rank & Time & Mem & Err & Params & Speedup \\ 
      &      & (ms) & (MB)&      &  (K)   & vs.\ Dense \\
      \midrule
      \multirow{5}{*}{128}
        & Dense & 0.57 & 26.3 & 0.0000 & 2362 & 1.00$\times$ \\
        & 64    & 0.88 & 41.5 & 0.0003 & 2510 & 0.65$\times$ \\
        & 48    & 0.86 & 39.8 & 0.0204 & 2031 & 0.67$\times$ \\
        & 32    & 0.85 & 39.3 & 0.0296 & 1551 & 0.68$\times$ \\
        & 16    & 0.84 & 38.5 & 0.0381 & 1072 & 0.68$\times$ \\
      \midrule
      \multirow{5}{*}{256}
        & Dense & 0.55 & 41.7 & 0.0000 & 2362 & 1.00$\times$ \\
        & 64    & 0.85 & 45.0 & 0.0003 & 2510 & 0.65$\times$ \\
        & 48    & 0.87 & 43.1 & 0.0155 & 2031 & 0.64$\times$ \\
        & 32    & 0.85 & 42.2 & 0.0240 & 1551 & 0.65$\times$ \\
        & 16    & 0.87 & 40.5 & 0.0283 & 1072 & 0.64$\times$ \\
      \midrule
      \multirow{5}{*}{512}
        & Dense & 0.57 & 50.6 & 0.0000 & 2362 & 1.00$\times$ \\
        & 64    & 0.85 & 50.4 & 0.0002 & 2510 & 0.67$\times$ \\
        & 48    & 0.62 & 48.0 & 0.0104 & 2031 & 0.92$\times$ \\
        & 32    & 0.57 & 46.5 & 0.0150 & 1551 & 1.00$\times$ \\
        & 16    & 0.57 & 45.0 & 0.0185 & 1072 & 1.00$\times$ \\
      \midrule
      \multirow{5}{*}{1024}
        & Dense & 0.37 & 92.0  & 0.0000 & 2362 & 1.00$\times$ \\
        & 64    & 0.55 & 65.6  & 0.0002 & 2510 & 0.66$\times$ \\
        & 48    & 0.57 & 62.7  & 0.0088 & 2031 & 0.64$\times$ \\
        & 32    & 0.55 & 59.6  & 0.0124 & 1551 & 0.66$\times$ \\
        & 16    & 0.57 & 58.5  & 0.0148 & 1072 & 0.65$\times$ \\
      \bottomrule
    \end{tabular}
  \end{subtable}\hfill
  \begin{subtable}[t]{0.312\textwidth}
    \centering
    \caption{$B=16$}
    \begin{tabular}{@{}l c c c c c c@{}}
      \toprule
      \itshape M & Rank & Time & Mem & Err & Params & Speedup \\ 
      &      & (ms) & (MB)&      &  (K)   & vs.\ Dense \\
      \midrule
      \multirow{5}{*}{128}
        & Dense & 0.38 & 77.7  & 0.0000 & 2362 & 1.00$\times$ \\
        & 64    & 0.61 & 88.8  & 0.0003 & 2510 & 0.62$\times$ \\
        & 48    & 0.61 & 86.2  & 0.0210 & 2031 & 0.62$\times$ \\
        & 32    & 0.60 & 83.9  & 0.0305 & 1551 & 0.62$\times$ \\
        & 16    & 0.61 & 79.6  & 0.0377 & 1072 & 0.62$\times$ \\
      \midrule
      \multirow{5}{*}{256}
        & Dense & 0.74 & 123.2 & 0.0000 & 2362 & 1.00$\times$ \\
        & 64    & 0.83 & 139.2 & 0.0002 & 2510 & 0.89$\times$ \\
        & 48    & 0.81 & 135.9 & 0.0150 & 2031 & 0.92$\times$ \\
        & 32    & 0.74 & 130.5 & 0.0218 & 1551 & 1.00$\times$ \\
        & 16    & 0.68 & 124.7 & 0.0270 & 1072 & 1.08$\times$ \\
      \midrule
      \multirow{5}{*}{512}
        & Dense & 2.31 & 298.5 & 0.0000 & 2362 & 1.00$\times$ \\
        & 64    & 2.21 & 236.6 & 0.0002 & 2510 & 1.05$\times$ \\
        & 48    & 2.10 & 231.5 & 0.0113 & 2031 & 1.10$\times$ \\
        & 32    & 1.79 & 221.4 & 0.0162 & 1551 & 1.29$\times$ \\
        & 16    & 1.68 & 212.1 & 0.0197 & 1072 & 1.37$\times$ \\
      \midrule
      \multirow{5}{*}{1024}
        & Dense & 6.98 & 946.3 & 0.0000 & 2362 & 1.00$\times$ \\
        & 64    & 5.65 & 434.2 & 0.0002 & 2510 & 1.24$\times$ \\
        & 48    & 5.41 & 426.0 & 0.0087 & 2031 & 1.29$\times$ \\
        & 32    & 4.57 & 407.5 & 0.0123 & 1551 & 1.53$\times$ \\
        & 16    & 4.25 & 389.1 & 0.0149 & 1072 & 1.64$\times$ \\
      \bottomrule
    \end{tabular}
  \end{subtable}\hfill
  \begin{subtable}[t]{0.312\textwidth}
    \centering
    \caption{$B=64$}
    \begin{tabular}{@{}l c c c c c c@{}}
      \toprule
      \itshape M & Rank & Time & Mem & Err & Params & Speedup \\ 
      &      & (ms) & (MB)&      &  (K)   & vs.\ Dense \\
      \midrule
      \multirow{5}{*}{128}
        & Dense & 1.43  & 240.5 & 0.0000 & 2362 & 1.00$\times$ \\
        & 64    & 1.61  & 254.8 & 0.0003 & 2510 & 0.89$\times$ \\
        & 48    & 1.57  & 230.6 & 0.0208 & 2031 & 0.91$\times$ \\
        & 32    & 1.41  & 221.7 & 0.0300 & 1551 & 1.02$\times$ \\
        & 16    & 1.29  & 211.5 & 0.0371 & 1072 & 1.11$\times$ \\
      \midrule
      \multirow{5}{*}{256}
        & Dense & 3.96  & 378.6 & 0.0000 & 2362 & 1.00$\times$ \\
        & 64    & 3.56  & 434.2 & 0.0002 & 2510 & 1.11$\times$ \\
        & 48    & 3.21  & 425.4 & 0.0153 & 2031 & 1.23$\times$ \\
        & 32    & 3.00  & 407.4 & 0.0220 & 1551 & 1.32$\times$ \\
        & 16    & 2.79  & 387.7 & 0.0271 & 1072 & 1.42$\times$ \\
      \midrule
      \multirow{5}{*}{512}
        & Dense & 10.40 &1090.7 & 0.0000 & 2362 & 1.00$\times$ \\
        & 64    &  8.66 & 829.2 & 0.0002 & 2510 & 1.20$\times$ \\
        & 48    &  8.31 & 816.3 & 0.0112 & 2031 & 1.25$\times$ \\
        & 32    &  7.16 & 779.4 & 0.0161 & 1551 & 1.45$\times$ \\
        & 16    &  6.72 & 742.9 & 0.0196 & 1072 & 1.55$\times$ \\
      \midrule
      \multirow{5}{*}{1024}
        & Dense & 29.19 &3681.9 & 0.0000 & 2362 & 1.00$\times$ \\
        & 64    & 23.45 &1622.1 & 0.0002 & 2510 & 1.24$\times$ \\
        & 48    & 22.46 &1595.5 & 0.0086 & 2031 & 1.30$\times$ \\
        & 32    & 19.35 &1523.4 & 0.0122 & 1551 & 1.51$\times$ \\
        & 16    & 18.61 &1449.6 & 0.0148 & 1072 & 1.57$\times$ \\
      \bottomrule
    \end{tabular}
  \end{subtable}
\end{table*}

\begin{table*}[p]
  \centering
  \scriptsize
  \setlength\tabcolsep{2pt}
  \caption{Full FlashSVDFFN Rank benchmarking on NVIDIA L40S GPU ($d_{\mathrm{model}}=768$, $d_{\mathrm{ff}}=3072$).}
  \label{tab:ffn-full-appendix}
  \begin{subtable}[t]{0.312\textwidth}
    \centering
    \caption{$B=1$}
    \begin{tabular}{@{}l c c c c c c@{}}
      \toprule
      \itshape M & Version & Rank & Time & Mem & Err & Speedup \\ 
      &         &      & (ms) & (MB)&      & vs.\ Full \\ 
      \midrule
      \multirow{5}{*}{128}
        & Dense      & Full    & 0.10 & 19.4 & 0.0000 & 1.00$\times$ \\
        & Flash-v1   & 768     & 0.35 & 30.3 & 0.0008 & 0.29$\times$ \\
        & Flash-v2   & 768     & 1.12 & 30.4 & 0.0008 & 0.09$\times$ \\
        & Flash-v1   & 384     & 0.27 & 24.1 & 0.7148 & 0.38$\times$ \\
        & Flash-v2   & 384     & 0.62 & 23.8 & 0.7148 & 0.17$\times$ \\
      \midrule
      \multirow{5}{*}{256}
        & Dense      & Full    & 0.11 & 23.2 & 0.0000 & 1.00$\times$ \\
        & Flash-v1   & 768     & 0.39 & 32.1 & 0.0008 & 0.27$\times$ \\
        & Flash-v2   & 768     & 1.11 & 30.6 & 0.0008 & 0.09$\times$ \\
        & Flash-v1   & 384     & 0.29 & 26.0 & 0.7115 & 0.36$\times$ \\
        & Flash-v2   & 384     & 0.62 & 24.8 & 0.7115 & 0.17$\times$ \\
      \midrule
      \multirow{5}{*}{512}
        & Dense      & Full    & 0.10 & 27.3 & 0.0000 & 1.00$\times$ \\
        & Flash-v1   & 768     & 0.61 & 34.9 & 0.0008 & 0.17$\times$ \\
        & Flash-v2   & 768     & 1.16 & 32.9 & 0.0008 & 0.09$\times$ \\
        & Flash-v1   & 384     & 0.27 & 28.0 & 0.7133 & 0.37$\times$ \\
        & Flash-v2   & 384     & 0.63 & 26.9 & 0.7133 & 0.16$\times$ \\
      \midrule
      \multirow{5}{*}{1024}
        & Dense      & Full    & 0.10 & 35.5 & 0.0000 & 1.00$\times$ \\
        & Flash-v1   & 768     & 1.12 & 40.8 & 0.0008 & 0.09$\times$ \\
        & Flash-v2   & 768     & 2.34 & 38.6 & 0.0008 & 0.04$\times$ \\
        & Flash-v1   & 384     & 0.38 & 34.5 & 0.7098 & 0.26$\times$ \\
        & Flash-v2   & 384     & 1.31 & 32.6 & 0.7098 & 0.07$\times$ \\
      \bottomrule
    \end{tabular}
  \end{subtable}\hfill
  \begin{subtable}[t]{0.312\textwidth}
    \centering
    \caption{$B=16$}
    \begin{tabular}{@{}l c c c c c c@{}}
      \toprule
      \itshape M & Version & Rank & Time & Mem & Err & Speedup \\ 
      &         &      & (ms) & (MB)&      & vs.\ Full \\ 
      \midrule
      \multirow{5}{*}{128}
        & Dense      & Full & 0.14 & 61.0 & 0.0000 & 1.00$\times$ \\
        & Flash-v1   & 768  & 2.13 & 56.4 & 0.0008 & 0.07$\times$ \\
        & Flash-v2   & 768  & 3.45 & 48.6 & 0.0008 & 0.04$\times$ \\
        & Flash-v1   & 384  & 0.68 & 44.8 & 0.7101 & 0.21$\times$ \\
        & Flash-v2   & 384  & 1.82 & 41.0 & 0.7101 & 0.08$\times$ \\
      \midrule
      \multirow{5}{*}{256}
        & Dense      & Full & 0.29 & 85.5 & 0.0000 & 1.00$\times$ \\
        & Flash-v1   & 768  & 4.05 & 75.3 & 0.0008 & 0.07$\times$ \\
        & Flash-v2   & 768  & 7.05 & 65.5 & 0.0008 & 0.04$\times$ \\
        & Flash-v1   & 384  & 1.26 & 65.3 & 0.7136 & 0.23$\times$ \\
        & Flash-v2   & 384  & 3.89 & 57.5 & 0.7136 & 0.07$\times$ \\
      \midrule
      \multirow{5}{*}{512}
        & Dense      & Full & 0.68 &150.5 & 0.0000 & 1.00$\times$ \\
        & Flash-v1   & 768  & 8.30 &118.8 & 0.0008 & 0.08$\times$ \\
        & Flash-v2   & 768  &13.43 &101.6 & 0.0008 & 0.05$\times$ \\
        & Flash-v1   & 384  & 2.46 &106.9 & 0.7105 & 0.28$\times$ \\
        & Flash-v2   & 384  & 7.16 & 88.8 & 0.7105 & 0.09$\times$ \\
      \midrule
      \multirow{5}{*}{1024}
        & Dense      & Full & 1.47 &282.5 & 0.0000 & 1.00$\times$ \\
        & Flash-v1   & 768  &16.77 &209.4 & 0.0008 & 0.09$\times$ \\
        & Flash-v2   & 768  &27.10 &173.1 & 0.0008 & 0.05$\times$ \\
        & Flash-v1   & 384  & 4.99 &191.4 & 0.7134 & 0.30$\times$ \\
        & Flash-v2   & 384  &14.42 &154.8 & 0.7134 & 0.10$\times$ \\
      \bottomrule
    \end{tabular}
  \end{subtable}\hfill
  \begin{subtable}[t]{0.312\textwidth}
    \centering
    \caption{$B=64$}
    \begin{tabular}{@{}l c c c c c c@{}}
      \toprule
      \itshape M & Version & Rank & Time & Mem & Err & Speedup \\ 
      &         &      & (ms) & (MB)&      & vs.\ Full \\ 
      \midrule
      \multirow{5}{*}{128}
        & Dense      & Full & 0.69 &186.5 & 0.0000 & 1.00$\times$ \\
        & Flash-v1   & 768  & 8.28 &137.4 & 0.0008 & 0.08$\times$ \\
        & Flash-v2   & 768  &13.33 &101.1 & 0.0008 & 0.05$\times$ \\
        & Flash-v1   & 384  & 2.56 &107.4 & 0.7114 & 0.27$\times$ \\
        & Flash-v2   & 384  & 7.26 & 88.9 & 0.7114 & 0.09$\times$ \\
      \midrule
      \multirow{5}{*}{256}
        & Dense      & Full & 1.47 &282.5 & 0.0000 & 1.00$\times$ \\
        & Flash-v1   & 768  &16.71 &208.5 & 0.0008 & 0.09$\times$ \\
        & Flash-v2   & 768  &27.17 &173.9 & 0.0008 & 0.05$\times$ \\
        & Flash-v1   & 384  & 5.05 &191.6 & 0.7102 & 0.29$\times$ \\
        & Flash-v2   & 384  &14.45 &154.8 & 0.7103 & 0.10$\times$ \\
      \midrule
      \multirow{5}{*}{512}
        & Dense      & Full &10.40 &1090.7& 0.0000 & 1.00$\times$ \\
        & Flash-v1   & 768  & 8.66 & 829.2& 0.0002 & 1.20$\times$ \\
        & Flash-v2   & 768  &54.82 & 317.0& 0.0008 & 0.05$\times$ \\
        & Flash-v1   & 384  & 9.91 & 358.8& 0.7090 & 0.27$\times$ \\
        & Flash-v2   & 384  &29.21 & 286.8& 0.7090 & 0.09$\times$ \\
      \midrule
      \multirow{5}{*}{1024}
        & Dense      & Full &29.19 &3681.9& 0.0000 & 1.00$\times$ \\
        & Flash-v1   & 768  &74.28 & 749.3& 0.0008 & 0.08$\times$ \\
        & Flash-v2   & 768  &108.50& 604.4& 0.0008 & 0.05$\times$ \\
        & Flash-v1   & 384  &21.40 & 694.8& 0.7133 & 0.27$\times$ \\
        & Flash-v2   & 384  &58.20 & 550.8& 0.7133 & 0.10$\times$ \\
      \bottomrule
    \end{tabular}
  \end{subtable}
\end{table*}

\begin{table}[t]
  \centering
  \small
  \setlength\tabcolsep{4pt}
  \caption{Finetuned FlashSVD shows Strong Potential in Edge Deployment (batch\_size=32)}
  \label{tab:finetune-full}
  \begin{tabular}{@{}c c c c c@{}}
    \toprule
    \multicolumn{5}{c}{\bfseries SST-2}\\
    \midrule
    Rank & Acc (\%) & Trans. & Peak & Lat\\
    \midrule
    dense  & 92.52 & 143.3  & 560.9 & 66.9 \\
    \midrule
    32-192  & 88.30 & 107.7  & 330.8  & 74.4 \\
    32-192 (SVD)  & 88.19 & 203.7  & 426.9  & 88.9 \\ 
    16-96  & 86.81 & 105.8 & 263.9  & 70.3 \\
    16-96 (SVD) & 86.81 & 201.9 & 359.9  & 78.7 \\
    \midrule
    \multicolumn{5}{c}{\bfseries Unfinetuned SST-2}\\
    \midrule
    Rank & Acc (\%) & Trans. & Peak & Lat\\
    \midrule
    32-192  & 58.82 & 110.1  & 328.6 & 92.7 \\
    32-192 (SVD)  & 58.82 & 206.1  & 424.6   & 131.4 \\
    16-96 & 50.89 & 109.1 & 265.2  & 94.7 \\
    16-96 (SVD)  & 50.89 & 205.1 & 361.2  & 124.9 \\
    \midrule
    \multicolumn{5}{c}{\bfseries STS-B (pearson accuracy)}\\
    \midrule
    dense  & 81.41 & 281.3  & 699.0 & 98.8 \\
    \midrule
    Rank & Acc (\%) & Trans. & Peak & Lat\\
    \midrule
    32-192  & 71.49 & 206.3  & 429.4  & 112.5 \\ % pearson=0.6533
    32-192 (SVD)  & 71.49 & 398.3  & 621.5  & 118.0  \\
    40-240  & 80.52 & 207.8  & 460.8  & 114.8\\ % pearson=0.7407
    40-240 (SVD)  & 80.52 & 399.8 & 652.8  & 128.9 \\
    \midrule
    \multicolumn{5}{c}{\bfseries Unfinetuned STS-B}\\
    \midrule
    Rank & Acc (\%) & Trans. & Peak & Lat\\
    \midrule
    32-192  & 0.76 & 215.2  & 433.7  & 126.0 \\ 
    32-192 (SVD) & 0.78 & 407.1  & 625.7  & 137.4  \\ 
    40-240  & 13.51 & 216.3  & 466.1  & 136.7 \\ 
    40-240 (SVD)  & 13.50 & 408.3 & 658.0  & 160.1 \\ 
    \bottomrule
  \end{tabular}
\end{table}

\end{document}